\documentclass{article}
\PassOptionsToPackage{svgnames,dvipsnames}{xcolor} %
\usepackage[svgnames,dvipsnames]{xcolor} %
\PassOptionsToPackage{final,kerning=true,spacing=true,factor=1100,stretch=15,shrink=15}{microtype}
\usepackage[final,kerning=true,spacing=true,factor=1100,stretch=15,shrink=15]{microtype}
\usepackage{amsmath}

\usepackage{amsthm}

\usepackage{float}
\newfloat{Template}{htbp}{loa}
\usepackage[labelfont=bf,textfont=bf]{caption} %
\usepackage{booktabs} %

\usepackage{xspace}

\usepackage{graphicx}
\usepackage[labelfont=bf,textfont=bf]{caption} %
\usepackage{subcaption}
\usepackage{alphalph}

\usepackage{bm}

\usepackage{enumitem}
\setlist[itemize]{label=$\bullet$}

\usepackage{url}

\usepackage{amssymb}
\usepackage{stmaryrd}
\usepackage{scalefnt}
\usepackage{mathtools}
\usepackage{etoolbox}

\usepackage{multirow}
\usepackage[shortcuts]{extdash} %

\usepackage{hyperref}

\newcommand{\eat}[1]{}
\newcommand{\kw}[1]{{\ensuremath {\mathsf{#1}}}\xspace}
\newcommand{\at}[1]{\protect\ensuremath{\mathsf{#1}}\xspace}

\newcommand{\tw}[1]{\text{\fontfamily{lmtt}\selectfont {#1}}\xspace} %

\newcommand{\stitle}[1]{\vspace{0.1ex}\noindent{\bf #1}}
\newcommand{\etitle}[1]{\vspace{0.0ex}\noindent{\underline{\smash{\em #1}}}}
\newcommand{\eetitle}[1]{\vspace{0.0ex}\noindent{{\em #1}}}

\newcommand{\ei}{\end{itemize}\vspace{1ex}}

\newcommand{\ee}{\end{enumerate}\vspace{1ex}}

\newcommand{\eop}{\hspace*{\fill}\mbox{$\Box$}}     %
\newcounter{example}

\newcounter{definition}%
\renewcommand{\thedefinition}{\arabic{definition}}

\newcounter{theorem}%
\renewcommand{\thetheorem}{\arabic{theorem}}
\newcounter{prop}%
\renewcommand{\theprop}{\arabic{theorem}}
\newcounter{lemma}%
\renewcommand{\thelemma}{\arabic{theorem}}
\newcounter{cor}%
\renewcommand{\thecor}{\arabic{theorem}}
\newenvironment{theorem}{\begin{em}
        \refstepcounter{theorem}
        {\vspace{1ex} \noindent\bf  Theorem  \thetheorem:}}{
      \end{em}\eop\vspace{1ex}} %
\newenvironment{lemma}{\begin{em}
        \refstepcounter{theorem}
        {\vspace{1.5ex}\noindent\bf Lemma \thelemma:}}{
      \end{em}\eop\vspace{1.5ex}} %
\newenvironment{cor}{\begin{em}
        \refstepcounter{theorem}
        {\vspace{1.5ex}\noindent\bf Corollary \thecor:}}{
        \end{em}\eop\vspace{1.5ex}} %
\newcounter{arule}
\renewcommand{\thearule}{\arabic{arule}}

\newcounter{claim}
\renewcommand{\theclaim}{\arabic{claim}}

\renewenvironment{proof}{
        \vspace{1ex}
        {\noindent\bf Proof:}}{\eop\vspace{1ex}}

\newcommand{\ie}{\emph{i.e.,}\xspace}
\newcommand{\eg}{\emph{e.g.,}\xspace}
\newcommand{\wrt}{\emph{w.r.t.}\xspace}
\newcommand{\aka}{\emph{a.k.a.}\xspace}

\newcommand{\ak}{\allowbreak}

\newcommand{\Ac}{\mathcal{A}}
\newcommand{\Bc}{\mathcal{B}}
\newcommand{\Cc}{\mathcal{C}}

\newcommand{\Dc}{\mathcal{D}}
\newcommand{\Sc}{\mathcal{S}}
\newcommand{\Pc}{\mathcal{P}}
\newcommand{\Uc}{\mathcal{U}}

\newcommand{\Yc}{\mathcal{Y}}

\makeatletter
\expandafter\def\csname ver@algorithm.sty\endcsname{9999/12/31}
\expandafter\def\csname ver@algorithmic.sty\endcsname{9999/12/31}
\makeatother

\usepackage[linesnumbered, ruled, vlined]{algorithm2e}
\SetAlgoSkip{} %

\usepackage{algpseudocode}

\SetKwRepeat{Do}{do}{while}%
\SetAlFnt{\footnotesize}
\SetAlCapFnt{\footnotesize}
\SetAlCapNameFnt{\footnotesize}
\SetAlCapHSkip{0pt}
\newcommand{\llIf}[2]{{\let\par\relax\lIf{#1}{#2}}}
\newcommand{\llFor}[2]{{\let\par\relax\lFor{#1}{#2}}}
\newcommand{\llForEach}[2]{{\let\par\relax\lForEach{#1}{#2}}}
\newcommand{\llElse}[1]{{\let\par\relax\lElse{#1}}}

\SetCommentSty{mycommfont}
\SetNoFillComment
\SetArgSty{textnormal}

\SetKwProg{proc}{Procedure}{}{}

\makeatletter
\newcommand{\nosemic}{\renewcommand{\@endalgocfline}{\relax}}%
\newcommand{\dosemic}{\renewcommand{\@endalgocfline}{\algocf@endline}}%
\let\oldnl\nl%
\newcommand{\nonl}{\renewcommand{\nl}{\let\nl\oldnl}}%
\makeatother

\SetKwFor{RepTimes}{repeat}{times}{end}

\newcounter{ccc}

\makeatletter
\newcommand{\xMapsto}[2][]{\ext@arrow 0599{\Mapstofill@}{#1}{#2}}
\def\Mapstofill@{\arrowfill@{\Mapstochar\Relbar}\Relbar\Rightarrow}
\makeatother

\newcommand{\mdcp}{\kw{SBEP}}
\newcommand\sbep\mdcp

\newcommand{\mscp}{\kw{DBEP}}
\newcommand\dbep\mscp

\newcommand{\dic}{\kw{dSBC}}
\newcommand\dsbc\dic
\newcommand{\sic}{\kw{sDBC}}
\newcommand\sdbc\sic

\makeatletter

\newcommand*\wthelper[2]{%
        \hbox{\dimen@\accentfontxheight#1%
                \accentfontxheight#11.3\dimen@
                $\m@th#1\widetilde{#2}$%
                \accentfontxheight#1\dimen@
        }%
}

\newcommand*\accentfontxheight[1]{%
        \fontdimen5\ifx#1\displaystyle
                \textfont
        \else\ifx#1\textstyle
                \textfont
        \else\ifx#1\scriptstyle
                \scriptfont
        \else
                \scriptscriptfont
        \fi\fi\fi3
}
\makeatother

\newcommand{\R}{\mathbb{R}}

\newcommand{\Oc}{\mathcal{O}}
\newcommand{\Ec}{\mathbf{E}}
\newcommand{\E}{\mathbb{E}}
\newcommand{\Mc}{\mathcal{M}}

\newcommand{\Lc}{\mathcal{L}}

\newcommand\dist{\at{dist}}

\newcommand{\ksim}{{\normalfont \tw{sim}}}

\newcommand{\arguana}{{\text{\small ArguAna}}}
\newcommand{\nfcorpus}{{\text{\small NFCorpus}}}
\newcommand{\scidocs}{{\text{\small SciDocs}}}
\newcommand{\scifact}{{\text{\small SciFact}}}
\newcommand{\fiqa}{{\text{\small FiQA}}}
\newcommand{\fever}{{\text{\small FEVER}}}
\newcommand{\stackexchange}{{\textit{StackExchangeClustering}}\xspace}
\newcommand{\reddit}{{\textit{RedditClustering}}\xspace}

\newcommand{\linear}{\at{Linear}}
\newcommand{\cca}{\at{CCA}}
\newcommand{\mlp}{\at{MLP}}
\newcommand{\rcsls}{\at{RCSLS}}

\newcommand{\pr}{\at{Proc}}
\newcommand{\ao}{\at{AO}}
\newcommand{\ugw}{\at{UGW}}
\newcommand{\mvh}{\at{GEH}}

\newcommand{\geh}{\mvh}

\newcommand{\gte}{\tw{GTE}}
\newcommand{\openai}{\tw{OpenAI}}
\newcommand{\mistral}{\tw{Mistral}}
\newcommand{\qwen}{\tw{Qwen}}
\newcommand{\kalm}{\tw{KaLM}}
\newcommand{\glove}{\tw{GloVe}}
\newcommand{\fasttext}{\tw{fastText}}

\colorlet{DarkRed}{red!70!black}
\definecolor{uoedarkred}{RGB}{165, 0, 52}
\definecolor{ta3chocolate}{rgb}{0.56078, 0.34902, 0.0078431}	%

\newcommand{\revise}[1]{{\color{blue}{#1}}}
\renewcommand{\revise}[1]{{\color{black}{#1}}}

\usepackage[accepted]{icml2026}

\usepackage[textsize=tiny]{todonotes}

\icmltitlerunning{Vector Linking via Cross-Model Local Isometric Consistency}

\begin{document}

\twocolumn[
\icmltitle{Vector Linking via Cross-Model Local Isometric Consistency}

  \begin{icmlauthorlist}
    \icmlauthor{Ziying Chen}{ed}
    \icmlauthor{Yang Cao}{ed}
    \icmlauthor{He Sun}{ed,siat}
    \icmlauthor{Beining Yang}{ed}
    \icmlauthor{Tianjian Yang}{ed}
  \end{icmlauthorlist}

  \icmlaffiliation{ed}{School of Informatics, University of Edinburgh, Edinburgh, United Kingdom}
  \icmlaffiliation{siat}{Shenzhen Institutes of Advanced Technology, Chinese Academy of Sciences, Shenzhen, China}

  \icmlcorrespondingauthor{Yang Cao}{yang.cao@ed.ac.uk}

  \vskip 0.3in
]

\printAffiliationsAndNotice{}  %

\begin{abstract}
We study {\em Vector Linking}: given two embedding clouds produced by different
black\-/box encoders over partially overlapping datasets, recover cross\-/model
object correspondences using only vectors. Empirically and theoretically, we
show that independently trained contrastive encoders exhibit local geometric consistency: short\-/range distances are approximately preserved up to a
scale factor, while long-range distances are not due to model-specific
distortion.  Building on this, we propose an iterative, reference-based
geometric embedding hashing that recovers vector links from a tiny seed set of
paired anchors. It represents each vector by distances to sampled paired
anchors, proposes candidate links via hash-space matching, and aggregates
evidence across views in a Beta--Bernoulli posterior to bootstrap high\-/confidence
links as new anchors. Experiments across multiple benchmarks and embedding model
pairs demonstrate accurate and robust linking under varying overlap, seed
budgets, and out\-/of\-/domain anchors, with applications to vector
database integration and cross\-/model clustering.
Code is available at \url{https://github.com/DBgroup-Edinburgh/VecLinking}.
\end{abstract}

\section{Introduction}
\label{sec-intro}

Information systems increasingly rely on embedding-based retrieval: large
collections of objects are mapped to vectors and indexed for similarity search.
In practice, however, embedding models evolve quickly, and different systems
often adopt different fine-tuned encoders. As a result, practitioners are left
with multiple vector indices whose representations are not directly comparable,
even when those indices contain many of the same objects. This interoperability
gap hinders unified retrieval, cross-index deduplication, joint clustering, and
vector database integration.

\stitle{Vector Linking.}
We study \emph{vector linking}: recovering which vectors in two embedding clouds
correspond to the same underlying object when the clouds are produced by
different black-box contrastive encoders and overlap only partially. 
Formally, let $\Oc_1$ and $\Oc_2$
be two datasets of objects with \emph{unknown} overlap
$\Omega = \Oc_1 \cap \Oc_2$. Let $f_1$ and $f_2$ be two encoders,
and let $\Ec_1 = f_1(\Oc_1)$ and $\Ec_2 = f_2(\Oc_2)$ be the resulting
embedding sets. We assume access only to $\Ec_{1}$, $\Ec_{2}$,
and a small seed set of paired anchors $S \subseteq M^*$, where $M^* =
\{(f_1(x), f_2(x)) : x \in \Omega\}\subseteq E_{1}\times E_{2}$.
The goal is to recover as many pairs in $M^*$ as possible without access to
the raw objects, model parameters, gradients, or retraining.

This setting differs from standard embedding alignment in two important ways.
First, the overlap is \emph{partial and unknown}: there is no global bijection
between the two embedding sets, and the non-overlapping regions do not simply
behave like outliers. Instead, they can substantially alter the global geometry
seen by each encoder. Second, we target a strict \emph{post-hoc black-box}
regime. Many compatibility and alignment methods assume access to training data,
encoder internals, or training-time intervention; here, only static vectors are
available. Together, these two properties make a single global transformation
unreliable. 

\stitle{Local isometric consistency}. 
Our starting point is a simple but robust finding. When we compare
pairwise distances between shared objects across independently trained
contrastive encoders, short distances remain strongly correlated while
long-range distances decorrelate quickly. Equivalently, small neighborhoods are
far more stable across models than the global arrangement of the embedding
clouds. Theoretically, we show that this pattern is not merely an empirical
coincidence: by analyzing a localized alignment-uniformity surrogate for
contrastive learning, under standard assumptions, we show that independently
trained contrastive encoders can induce locally isometric metrics up to scale.

\stitle{Geometric embedding hashing}.
Motivated by this observation, we propose \emph{Geometric Embedding Hashing}
(GEH). The basic unit of GEH is a \emph{distance-to-anchor signature}: given a
small set of paired anchors, each vector is represented by its distances to
those anchors within its own embedding space. If two vectors correspond to the
same object, and if the chosen anchors lie in their local neighborhoods, then
these relative distance patterns should remain similar up to scale even when the
global shapes of the two embedding clouds differ markedly. GEH therefore
compares normalized, scale-free signatures rather than raw distances.

A single anchor set is not locally informative for every point, so GEH does not
rely on one global hash. Instead, it repeatedly samples many small anchor
subsets, or \emph{views}, matches points independently in each induced hash
space, and treats resulting matches as noisy votes. A Beta-Bernoulli
posterior aggregates evidence across views, and high-confidence matches are
promoted as new anchors for the next round. This multi-view bootstrapping lets
GEH grow a tiny seed set into a large correspondence set while gracefully filtering
spurious collisions caused by model-specific distortion and partial overlap. 
\looseness = -1

We evaluate GEH on multiple BEIR benchmarks and five encoder pairs spanning
both API-based and open-weight models. Across varying overlap ratios, seed
budgets, and out-of-domain seed settings, GEH consistently outperforms eight
linear, nonlinear, and optimal-transport baselines using only 15 to 30 seed
pairs. For instance, with only 15 paired seeds, GEH achieves over 90\% recall on
\fiqa~\cite{fiqa} for \mistral and \openai. We further show that the recovered links
improve downstream tasks including vector database integration and cross-model clustering. 

The results suggest that vector linking is a practical primitive for embedding
interoperability, and that local geometric consistency across contrastive
encoders holds key to tackle it.\looseness = -1

\stitle{Contributions \& organization.}
We contribute as follows.
\begin{itemize}[leftmargin=3ex, itemsep=0.0ex, topsep=0.0ex, partopsep=0.0ex]
\item  We propose {\em vector linking}, the problem of recovering
correspondences between two black-box embedding clouds under partial, unknown
overlap.
\item We establish, both empirically and theoretically, a cross\-/model \emph{local
distance consistency} property, forming the foundation of encoder-invariant
hashing (Section~\ref{sec-local}).
\item We develop a multi-view geometric hashing algorithm with
  posterior-guided bootstrapping that accurately recovers vector links without
  accessing raw objects or model internals, using only tiny seeds
  (Sections~\ref{sec-hash}-\ref{sec-linking}).
\item We demonstrate accurate and robust linking across multiple benchmarks and embedding model pairs (Section~\ref{sec-exp}).
\item We further demonstrate its benefits to vector database integration and
  cross-model clustering (Section~\ref{sec-applications}). 
\end{itemize}

\stitle{Related Work.}
Geometric point set registration under partial overlap has been studied via
hypothesis testing (e.g., RANSAC \cite{10.1145/358669.358692}, TEASER++
\cite{yang2020teaserfastcertifiablepoint}), iterative refinement (e.g., ICP
\cite{121791}, Go-ICP \cite{Yang_2016}), and invariant signatures (geometric
hashing \cite{589995}).  These tools are primarily designed for 3D rigid space
and cannot handle high\-/dimensional heteroscedastic model\-/induced distortion
and unknown overlap that vector linking targets.

Embedding alignment methods for \eg bilingual lexicon induction,
learn a global mapping between spaces (linear/Procrustes, OT/GW)
\cite{mikolov2013Exploiting,xing-etal-2015-normalized,smith2017offline,LampleCRDJ18,Artetxe_2018,Alvarez-MelisJ18,pmlr-v89-grave19a},
often relying on an approximate global isomorphism.
Such global consistency has also been
exploited for domain adaptation
\cite{shen2021backwardcompatiblerepresentationlearning,HuBC0SL22,10.5555/2283516.2283652,wang2018visualdomainadaptationmanifold,ganin2016domainadversarialtrainingneuralnetworks,hoffman2017cycadacycleconsistentadversarialdomain},
which further demand training\-/time access not
available in the black-box vector linking setting.

Unlike embedding alignment that seeks a coupling that makes the {\em entire} spaces
globally comparable, vector linking instead seeks a {\em partial} one-to-one
correspondence relation on the unknown shared support, while leaving vectors
outside the overlap unmatched.
This creates an objective mismatch for global
alignment, as there is no global bijection to recover. Further, 
non-overlapping regions are structured and potentially large, so they do not
behave like removable outliers. Alignment can thus improve global fit on
unmatched regions while worsening correspondences on the overlap. \looseness = -1

Vector linking bootstraps downstream interoperability tasks such as cross-model
vector database integration~\cite{LA2M} and joint
clustering~\cite{enevoldsen2025mmtebmassivemultilingualtext}, which assumes that 
that reliable cross-model anchor pairs are already known. Vector
linking addresses this assumption by recovering correspondences from
black-box vector clouds.

\stitle{Conflict of Interest Disclosure}.
The authors declare no financial conflicts of interest. All authors are
affiliated solely with academic institutions and the embedding models evaluated in this work are independent third\-/party systems.

\section{Foundation of Embedding Hashing}
\label{sec-local}

This section provides
the geometric foundation behind the idea of encoder\-/invariant geometric
hashing. We first establish an empirical
short\-/to\-/long range transition in cross\-/model distance consistency
(Section~\ref{subsec-evidence}). We then provide a localized geometric
explanation for why contrastive encoders tend to preserve local geometry
(Section~\ref{subsec-contrastive}).

\subsection{Emergence of Local Distance Consistency}
\label{subsec-evidence}

We begin by quantifying how pairwise euclidean distances compare across embedding spaces.
Let $\Ec_1$ and $\Ec_2$ be embeddings of the same raw dataset $D$ produced
by two different encoders (\eg \mistral vs. \openai).
We sample pairs $(u,v)$ of objects from $D$, compute $d_{\Ec_1}(u,v)$ and
$d_{\Ec_2}(u,v)$, and bin pairs by $d_{\Ec_1}(u,v)$. We report, per bin, the
Pearson correlation between $d_{\Ec_1}(u,v)$ and $d_{\Ec_2}(u,v)$ in
Fig.~\ref{fig-correlation} across multiple BEIR benchmarks (See \ref{app:datasets} for more).
\looseness = -1

\begin{figure}[t!]
  \begin{center}
  \centerline{\includegraphics[width=0.85\columnwidth]{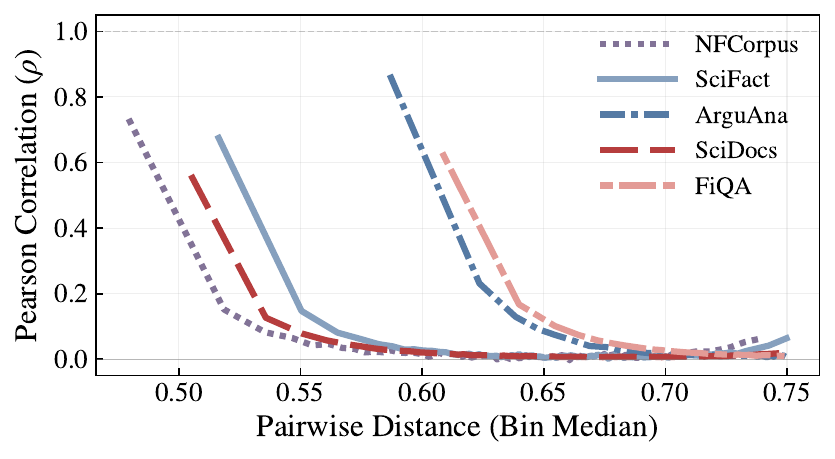}}
    \caption{Consistency (linear correlation) {\sc Vs.} vector distances:
      {\normalfont\footnotesize The x-axis
        shows the pairwise distance in the reference space (Mistral), while
        the y-axis reports the Pearson correlation ($\rho$) of these distances with
        their counterparts in the target space (OpenAI).}}
    \label{fig-correlation}
  \end{center}
\end{figure}

\stitle{Local consistency.}
For short distances, \eg $d_{\Ec_1}(u,\ak v)\ak \lesssim 0.57$ for \arguana, the
correlation is substantially positive with $\rho$ above 0.8.
Further, as shown in Appendix \ref{app:topk}, we also find that
top-$k$ retrieval exhibits strong consistency with small $k$ (\eg
$k\ak<\ak 10$) across embeddings compared to large $k$, which indicates that
nearby pairs under $\Ec_1$ tend to remain nearby under $\Ec_2$.
This seems to suggest that a linear correlation
in the short\-/distance regime has substantially more statistical significance compared
to long distances. 

\stitle{Global Decorrelation ($\rho \approx 0$)}. As distances increase, the
correlation decays rapidly, consistently nearing zero. This collapse indicates
that the scaling factor $\alpha$ is not globally constant.  While the models
agree on the ``shape'' of local neighborhoods, they diverge significantly on the
global arrangement and distribution of data objects. This renders long-range
distances inconsistent across embedding spaces.

This suggests that short distances are consistent across
encoders, and thus a distance-to-anchor vector that consists of short distances
can be a viable choice for encoder-invariant geometric hashing. \revise{
  We have also observed that such correlation is significant
weaker for non-contrastive encoders (see Appendix \ref{app:correlation} for
  more).} To further
confirm this, we still want to check if this is just an empirical coincidence or
something fundamental to the embedding models.

\subsection{A Geometric Justification of Local Isometries}
\label{subsec-contrastive}

We give a geometric explanation for the short-distance consistency of 
Fig.~\ref{fig-correlation}. We show that the phenomenon is inherent to
contrastive encoders rather than an empirical coincidence.

\stitle{Geometric modeling.}
We model data as a random variable $X$ supported on a smooth $d$-dimensional
manifold $\Mc \subset \R^{N}$ with density $p(x)$ \wrt the intrinsic
Riemannian volume measure on $\Mc$.
We denote geodesic distance by $d_{\Mc}(x,y)$, which serves as an intrinsic
notion of semantic dissimilarity between the data objects modeled by $x$ and
$y$.

An embedding model (\aka encoder) is a map $ f:\Mc\ak \to\ak \R^K$
with normalized outputs $f(\Mc) \subset S^{K-1}$.
Let the Jacobian of the encoder $f$ at $x$ be $J_{f}(x) \in
\R^{K\times d}$, where $d$ is the intrinsic dimension of
$\Mc$; it maps tangent vectors from the data manifold to the embedding space.
We denote the metric tensor induced by the encoder $f$
as $G_{f}(x) := J_{f}(x)^{\top}J_{f}(x) \in \R^{d\times d}$,
which characterizes how local distances are distorted by the map.
(A1) We assume that $f$ is twice differentiable and injective, and
that $G_{f}(x)$ is positive definite for all $x\in \Mc$, \ie
$J_{f}(x)$ has full rank $d$.
\looseness = -1

\stitle{Short-range neighborhoods.}
For each $ x\in\Mc $, let $\delta_{\Mc}(x)>0 $ be such that whenever
$d_{\Mc}(x,y)<\delta_{\Mc}(x) $ the shortest geodesic from $x$ to $y$
on $\Mc$ is unique. For such $y$, define the geodesic displacement $v(x,y)\in T_x\Mc$ as the tangent
vector at $x$ pointing toward $y$ along this unique shortest geodesic, normalized so that
$\|v(x,y)\| = d_{\Mc}(x,y)$.

\stitle{Contrastive learning}.
We consider encoders trained via contrastive learning with
InfoNCE-type contrastive loss objectives.
The training signal comes from:
(i) positive pairs, which are two semantic-preserving ``views'' of the same data object,
and (ii) negative pairs, which pair unrelated data objects.
Given $x\in\Mc$, a ``positive view'' $x^+$ is sampled by a stochastic
augmentation. We assume the following.
(A2: local positives) $d_{\Mc}(x,\ak x^+)\ak <\ak \delta_{\Mc}(x)$ almost surely.
(A3: local isotropy) Following \cite{DaoGRSSR19,wang2020understanding}, we assume that the
distribution of $x^{+}$ is centered and isotropic on the local
tangent space. Specifically, fix $x\ak \in\ak \Mc$ and let the geodesic
displacement vector $v\ak =\ak v(x,\ak x^+)\in T_x\Mc$.
We assume $\E[v\ak \mid\ak x]\ak =\ak 0$ and $\E[vv^\top\mid x]\ak =\ak c\,I_d$.
(We discuss relaxations to $c=c(x)$ in Appendix~\ref{app:assumptions}.)
\looseness = -1

\stitle{Global contrastive surrogate}.
We adopt the standard alignment-uniformity perspective on contrastive
learning~\cite{wang2020understanding,ZimmermannSSBB21}:
positives should be mapped close (alignment), while the overall representation
distribution should be spread out to avoid collapse (uniformity).
Let $X$ be the random data point on $\Mc$ and let $X^{+}$ denote its
positive view. Write $Z := f(X)$ for the induced random representation. 

Following~\cite{oord2018cpc,poole2019variational}, for InfoNCE-like
contrastive losses, alignment loss minimizes
$\E\big[\|f(x)\ak -\ak f(x^{+})\|^{2}\big]$, and uniformity can be
modeled by maximizing the entropy of the representation, which we
interpret intrinsically on $f(\Mc)$ as the standard differential
entropy~\cite{cover2006elements}. Specifically, let $q$ denote the
density of $Z$ on $f(\Mc)$ \wrt the induced $d$-dimensional surface
volume on $f(\Mc)$, and define entropy $H(Z) := -\E\big[\log q(Z)\big]$. 
Then we have the implementation-independent surrogate of contrastive
loss:
\[\Lc_{\lambda}(f) := \E\big[\|f(X)-f(X^{+})\|^{2}\big] - \lambda
  H(Z),\]
where $\lambda >0$ is a model-dependent coefficient that balances
alignment and uniformity. 
This surrogate is not identical to InfoNCE, but captures its geometric pressure
toward (i) local alignment of positives and (ii) spread of representations.

\stitle{A localized geometric view}.
As we focus on local geometric properties,  we develop a localized
view of the global contrastive loss surrogate.
Note that $\Lc_{\lambda}(f)$ is an expectation over $X$, it can be
written as an average of per-point contributions. Hence, we can
write $\mathcal{L}_\lambda(f)$ equivalently as 
\[\E\Big[\underbrace{\E\big[\|f(X)-f(X^+)\|^2\mid
    X\big]}_{\varphi_{\text{align}}:\text{local alignment at }X} + \underbrace{\lambda\,\log
    q(f(X))}_{\varphi_{\text{uni}}:\text{local uniformity at }X} \Big].\]

Motivated by this, we define a localized loss at $x\in \Mc$, denoted
by $\Lc_\lambda(x;f)$, such that
$\Lc_{\lambda}(f)=\E[\Lc_{\lambda}(X;f)]$; hence
\[\mathcal{L}_\lambda(x;f)\ak :=\ak \varphi_{\text{align}}(X=x)\ak +\ak \varphi_{\text{uni}}(X=x).\]
Under the encoder assumption (A1),
by area formula and change-of-variables \cite{lee2013smooth} we have
$q(f(x)) = p(x)/\sqrt{\det(G_f(x))}$.
Therefore, up to an $x$-dependent constant not involving $f$, we have
$\Lc_{\lambda}(x;\ak f)\ak =\ak
\varphi_{\text{align}}(X\ak=\ak x)\ak -\ak \frac{\lambda}{2}\log\det(G_{f}(x))\ak +\ak 
\text{const}(x)$.
Further by the short-range data augmentation (A2, A3),
$\varphi_{\text{align}}(X=x) \approx c\cdot \mathrm{tr}(G_f(x))$
(ignore higher-order terms; see Corollary~\ref{cor:align-trace} in Appendix~\ref{app:sec2-proofs}).
Thus the leading-order localized geometric objective is
\vspace{-1ex}
\[\widetilde{\mathcal{L}}_\lambda(x;f) := c\cdot\mathrm{tr}(G_f(x))-\frac{\lambda}{2}\log\det(G_f(x)).\vspace{-1ex}\]

\stitle{Locally optimal encoders}. We say that an encoder $f$ is {\em
  locally optimal at $x\in \Mc$} (\wrt $\lambda$) if its $G_{f}(x)$
minimizes the leading-order localized geometric objective, \ie
$G_f(x)\in\arg\min_{G\succ 0}
\left\{c\cdot\mathrm{tr}(G)-\frac{\lambda}{2}\log\det(G)\right\}$.

Consider a manifold $\Mc$. Then we show the following (see
  Appendix~\ref{app:sec2-proofs} for a full proof).

\begin{theorem}\label{thm-local-iso} 
Let $f_{1}$ and $f_{2}$ be two encoders locally optimal at
point $x\in \Mc$ with parameters $\lambda_{1}$ and $\lambda_{2}$,
respectively. For any $y\in \Mc$ with $d_{\Mc}(x, y) <
\delta_{\Mc}(x)$: 
\[\|f_{1}(x) - f_{1}(y)\| = \kappa \cdot \|f_{2}(x) -
  f_{2}(y)\| + \Oc(d_{\Mc}(x, y)^{2}),\]
where $\kappa = \sqrt{\lambda_{1}/\lambda_{2}}$.
\end{theorem}

Theorem~\ref{thm-local-iso} is inherently local: it holds only for $y$ within
neighborhood radius $\delta_{\Mc}(x)$, consistent with the empirical decorrelation at
long distances in Fig.~\ref{fig-correlation}. Moreover, the same
analysis extends to a relaxation of (A3) where the local augmentation scale is
point-dependent, \ie $\E[vv^\top \mid x]=c(x)I_d$, yielding a region-dependent
$\kappa$ (see Appendix~\ref{app:assumptions}).

\section{Encoder-Invariant Geometric Hashing}
\label{sec-hash}

Theorem~\ref{thm-local-iso} establishes a local cross\-/model geometric
consistency: under local optimality, two contrastive encoders preserve
short-range distances up to a scale factor. We now translate this local
property into a concrete {\em hashing framework} for vector linking. The goal is
to construct, for each vector, a signature that is (approximately)
encoder\-/invariant over the overlap $\Omega$, so that matching objects collide
(or become nearest neighbors) in a shared hash space.

\stitle{Distance-to-anchor hash}.
A {\em geometric view} of $\Omega$ across $\Ec_{1}$ and
$\Ec_{2}$ is a set $\Ac$ of paired vectors $\{(a_1,\ak a_1'),\ak \dots,\ak (a_k,\ak
a_k')\}\ak \subseteq\ak \Ec_{1}\times \Ec_{2}$, where each pair $(a_j,\ak a_j')$ 
encodes the same overlap object in $\Omega$ and is referred to as a paired
anchor. Fix a distance function $\dist(\cdot, \cdot)$. Given a view $\Ac$, we define the {\em
  distance-to-anchor hash} of $u\in \Ec_{1}$ and $v\in \Ec_{2}$ \wrt $\Ac$ as:
$\mathbf{r}_{\Ac}(u)\ak :=\ak  \big(\dist(u,\ak a_1),\ak \dots,\ak \dist(u,\ak a_k)\big)\ak \in\ak \R^k$ and 
$\mathbf{r}'_{\Ac}(v)\ak :=\ak  \big(\dist(v,\ak a'_1),\ak \dots,\ak \dist(v,\ak a'_k)\big)\ak \in\ak \R^k$,
respectively. 

As Theorem~\ref{thm-local-iso} predicts an unknown scale factor between
encoders, we compare hashes $\mathbf{r}_{\Ac}(u)$ and $\mathbf{r}'_{\Ac}(v)$
using a {\em scale-free} similarity. A simple choice is cosine similarity over
normalization: $\ksim_{\Ac}(u,v) := \langle \widehat{\mathbf{r}}_{\Ac}(u),
\widehat{\mathbf{r}}'_{\Ac}(v)\rangle$, where 
$\widehat{\mathbf{r}}_{\Ac}(u) :=
\frac{\mathbf{r}_{\Ac}(u)}{\|\mathbf{r}_{\Ac}(u)\|_2}$ and 
$\widehat{\mathbf{r}}'_{\Ac}(v) :=
\frac{\mathbf{r}'_{\Ac}(v)}{\|\mathbf{r}'_{\Ac}(v)\|_2}$. 

\eat{%
\stitle{Locality $\Rightarrow$ encoder-invariant hashing}. 
Let $x\ak \in\ak \Omega$ be an overlap object with representations
$u\ak =\ak f_1(x)\ak \in\ak \Ec_1$ and $v\ak =\ak f_2(x)\ak \in\ak \Ec_2$.
Let the anchors in $\Ac$ correspond to overlap objects $x_1,\dots,x_k\ak \in\ak
\Omega$, \ie $a_j\ak =\ak f_1(x_j)$ and $a'_j\ak =\ak f_2(x_j)$.
If each anchor is in the {\em short-range} neighborhood of $x$
(\ie $ d_{\Mc}(x,\ak x_j)\ak <\ak \delta_{\Mc}(x)$; recall Section~\ref{subsec-contrastive}), then applying
Theorem~\ref{thm-local-iso} with $y=x_j$ yields, component-wise,
$\dist(u,a_j) = \kappa\,\dist(v,a'_j) + O\!\big(d_{\Mc}(x,x_j)^2\big)$ for
each $j\ak =\ak 1,\ak\dots,\ak k$. Equivalently, in vector form,
$\mathbf{r}_{\Ac}(u) = \kappa\,\mathbf{r}'_{\Ac}(v) + \boldsymbol{\epsilon}_{\Ac}(x)$,
where $\|\boldsymbol{\epsilon}_{\Ac}(x)\| = O\,(\max_j d_{\Mc}(x,x_j)^2)$.

Hence, when anchors are locally relevant, the two hash vectors are approximately
collinear, and thus their scale-free similarity is close to $1$. This implies
that distance\-/to\-/anchor hashing is {\em encoder-invariant} in the
short-range regime: true matches remain stable under different locally chosen
views, despite unknown global relationship between $\Ec_1$ and $\Ec_2$.
}%

\stitle{Locality $\Rightarrow$ encoder-invariant hashing}. 
Let $x \in \Omega$ have representations $u = f_1(x) \in \Ec_1$ and
$v = f_2(x) \in \Ec_2$. Consider a view
$\Ac=\{(a_1,a'_1),\dots,(a_k,a'_k)\}$ whose anchors correspond to overlap
objects $x_1,\dots,x_k \in \Omega$ (\ie $a_j=f_1(x_j)$ and
$a'_j=f_2(x_j)$). If all anchors are short\-/range for $x$, \ie 
$d_{\Mc}(x,x_j) < \delta_{\Mc}(x)$ for all $j$, then applying
Theorem~\ref{thm-local-iso} with $y=x_j$ yields the componentwise relation
$\dist(u,a_j)=\kappa\,\dist(v,a'_j)+O(d_{\Mc}(x,x_j)^2)$. Hence
\[
\mathbf{r}_{\Ac}(u) = \kappa\,\mathbf{r}'_{\Ac}(v) \;+\; \boldsymbol{\epsilon}_{\Ac}(x),
\|\boldsymbol{\epsilon}_{\Ac}(x)\| = O\!\big(\max_j d_{\Mc}(x,x_j)^2\big),\]
In the ideal case $\boldsymbol{\epsilon}_{\Ac}(x)=0$, the two hashes are exactly
related by a positive scalar, so after $\ell_2$ normalization they coincide and
our scale\-/free similarity satisfies $\ksim_{\Ac}(u,\ak v)=1$. When anchors are
sufficiently close, the second\-/order remainder is small, hence the normalized
hashes remain close and $\ksim_{\Ac}(u,\ak v)$ stays near $1$. Therefore,
distance\-/to\-/anchor hashing is approximately encoder\-/invariant in the short\-/range
regime, providing a geometric basis for vector linking via hash\-/space nearest
neighbors.

The locality requirement above is essential. As shown empirically in
Section~\ref{subsec-evidence}, long-range distances decorrelate across
encoders. Therefore, if a view contains anchors that are far from the candidate
vectors, those hash coordinates become dominated by model-specific distortion
and can overwhelm the signal from locally consistent distances. This implies
that global hashing with one fixed anchor set is unreliable for points that do
not lie near that anchor set, which is precisely the typical case under partial,
unknown overlap. %

\eat{%
\stitle{Localizing hashes}. While distance-to-anchor hashes are
encoder-invariant when they are constructed from anchors in the short-range
neighborhood, it is, however, nontrivial to decide what counts as short-range
distances, as the threshold $\delta_{\Mc}$ is determined by the data manifold
(Section~\ref{subsec-contrastive}).

Rather than explicitly estimating the unknown locality threshold $\delta_{\Mc}$,
we treat locality {\em statistically} by sampling many small, diverse geometric
views. Concretely, we repeatedly draw anchor subsets $\Ac^{(1)}, \dots,
\Ac^{(T)}$ from the available pool of paired anchors (we will discuss how we
extract the paired anchors in Section~\ref{sec-linking}).
Each view $\Ac^{(t)}$ defines its own hash space via the distance\-/to\-/anchor
signatures, and produces a set of candidate links by nearest\-/neighbor matching
under the scale\-/free similarity $\ksim_{\Ac^{(t)}}(\cdot,\cdot)$.
\looseness = -1

Each view contributes one noisy ``vote'' for every candidate link it proposes.
Theorem~\ref{thm-local-iso} implies that a true link $(u,v)$ is expected to be
proposed consistently across subsets of views whose anchors are short\-/range
for that link, while spurious links caused by long\-/range, model\-/specific
distortion are view\-/dependent and rarely accumulate.  Thus
aggregating votes across many local views yields a distortion\-/tolerant signal
for vector linking under partial, unknown overlap.\looseness = -1

An additional advantage of this multi-view hashing is that it does not require the
scale factor in Theorem\ref{thm-local-iso} to be globally constant (or
explicitly estimated). Since we compare per\-/view hashes using a scale\-/free
similarity, any positive multiplicative scaling of all anchor distances within a view cancels out, so the same
voting\-/and\-/aggregation mechanism remains effective even when the proportionality
factor varies across the data manifold, \eg under the relaxed isotropy model of
(A3), \ie $\E[vv^\top\mid x]=c(x)I_d$ (Appendix\ref{app:assumptions}).
}%

\stitle{Localizing hashes via multi-view voting}.
While distance-to-anchor hashes are encoder-invariant when they are constructed
from anchors in the short-range neighborhood, it is, however, nontrivial to
decide what counts as short-range distances, as the threshold $\delta_{\Mc}$ is
unknown and data\-/dependent.\looseness = -1

Rather than explicitly estimating the unknown locality threshold $\delta_{\Mc}$,
we treat locality {\em statistically} by sampling many small, diverse geometric views.
$\Ac^{(1)}, \dots, \Ac^{(T)}$ from the current pool of paired anchors (via bootstrapping, as we will see in Section~\ref{sec-linking}). Each view induces its own hash space and proposes
candidate links by nearest\-/neighbor matching under $\ksim_{\Ac^{(t)}}(\cdot,\cdot)$; the proposed
pairs are treated as votes. True links are supported across many views that happen to include
locally relevant anchors, whereas spurious collisions caused by long-range, model-specific
distortion are view-dependent and rarely accumulate.
Since matching within each view uses a scale-free similarity, this aggregation
remains effective even when the proportionality factor varies over the manifold,
\eg under the relaxed isotropy model of (A3), \ie $\E[vv^\top\mid x]=c(x)I_d$
(Appendix~\ref{app:assumptions}).

\begin{figure}[t!]
  \begin{center}
    \centerline{\includegraphics[width=0.85\columnwidth]{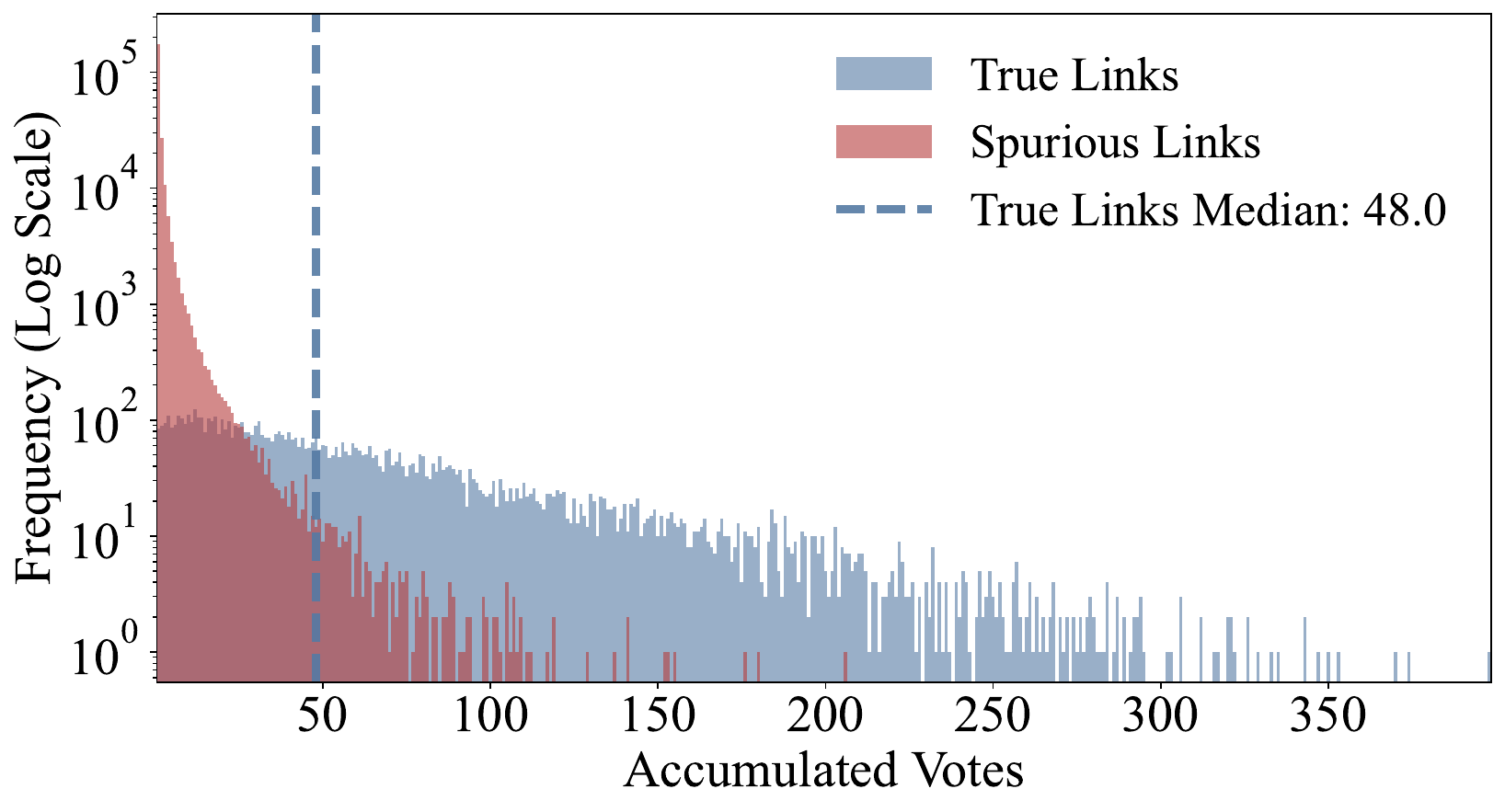}}
    \caption{Signal (true links) {\sc Vs.} Noise (spurious links):
      {\normalfont\footnotesize x-axis is accumulated votes for candidate
        links on \arguana (\gte vs. \openai); y-axis reports pair frequency (logscaled).}}
    \label{fig-vote_hist}
  \end{center}
\vspace{-2ex}  
\end{figure}

We tested this voting mechanism in a one\-/shot diagnostic
on \arguana~encoded by \gte~\cite{li2023towards} and \openai~\cite{openai}. We sampled a fixed pool of 500 ground\-/truth paired anchors
from the overlap and drew 500 random views, each containing 30 anchors. In each
view we computed hashes and collected voted links via hash collisions, then
tallied for each candidate pair the number of views in which it is
proposed. Fig.~\ref{fig-vote_hist} shows the resulting vote histogram on a log
scale. We observe a sharp separation: spurious links (red)
follow a steep exponential decay, with the vast majority receiving negligible
support. In contrast, true links (blue) exhibit a robust distribution with a
median of 48 votes, demonstrating that stable local geometry consistency allows
true links to survive across diverse views.
\looseness = -1

\section{Bootstrapping Multi-View Hashing}
\label{sec-linking}

Section~\ref{sec-hash} shows that distance-to-anchor hashing is reliable only
when a view contains anchors that are locally relevant, and we need many such
views to statistically form the localized hash via voting. We alleviate the high
demand of anchors by an  iterative bootstrapping framework that
starts from a tiny seed set of paired anchors and grows it using multi\-/view
hash collisions with posterior\-/guided promotion.\looseness = -1

\stitle{Framework}. The framework, denoted by \mvh
(\underline{G}eometric \underline{E}mbedding \underline{H}ashing)
and shown in Fig.~\ref{fig-framework}, takes as input embedding clouds $\Ec_{1}$
and $\Ec_{2}$, and a tiny seed set of paired anchors $\Sc\subseteq \Ec_{1}\times
\Ec_{2}$ known via \eg domain knowledge, and outputs a set of inferred links
$\Lc_{T} \subseteq \Ec_{1}\times \Ec_{2}$.

Starting with $\Lc_{0}:= \Sc$, \mvh generates $\Lc_{T}$ through iterations. At
iteration $t$, it derives $\Lc_{t}$ by using inferred links identified at  
iteration $t-1$ in $\Lc_{t-1}$ as hash anchors, 
in three steps:
\looseness = -1

\begin{itemize}[leftmargin=3ex, itemsep=0.0ex, topsep=0.0ex, partopsep=0.0ex]
\item {\em (View generation)} It samples $m_t$ geometric views $\Ac_{t,1},\ak
\dots,\ak \Ac_{t,m_t}$ from $\Lc_{t-1}$, such that each view $\Ac_{t,k}\ak
\subseteq \ak \Lc_{t-1}$ is a subset of paired anchors of fixed size $s_{t}$.

\item {\em (Per-view link proposals)}
For each view $\Ac_{t,k}$, it computes view\-/specific distance\-/to\-/anchor
hashes and proposes a set of candidate links
$\Pc_{t,k}\subseteq \Ec_1\times \Ec_2$ by nearest\-/neighbor matching in
the hash space.

\item {\em (Posterior-guided bootstrapping)}
It aggregates all proposals into a confidence score for each candidate link
and promotes high\-/confidence links as new anchors in $\Lc_t$.
\looseness = -1

\end{itemize}

\begin{figure}[t]
  \begin{center}
    \vspace{-1.5ex}
    \centerline{\includegraphics[width=\columnwidth]{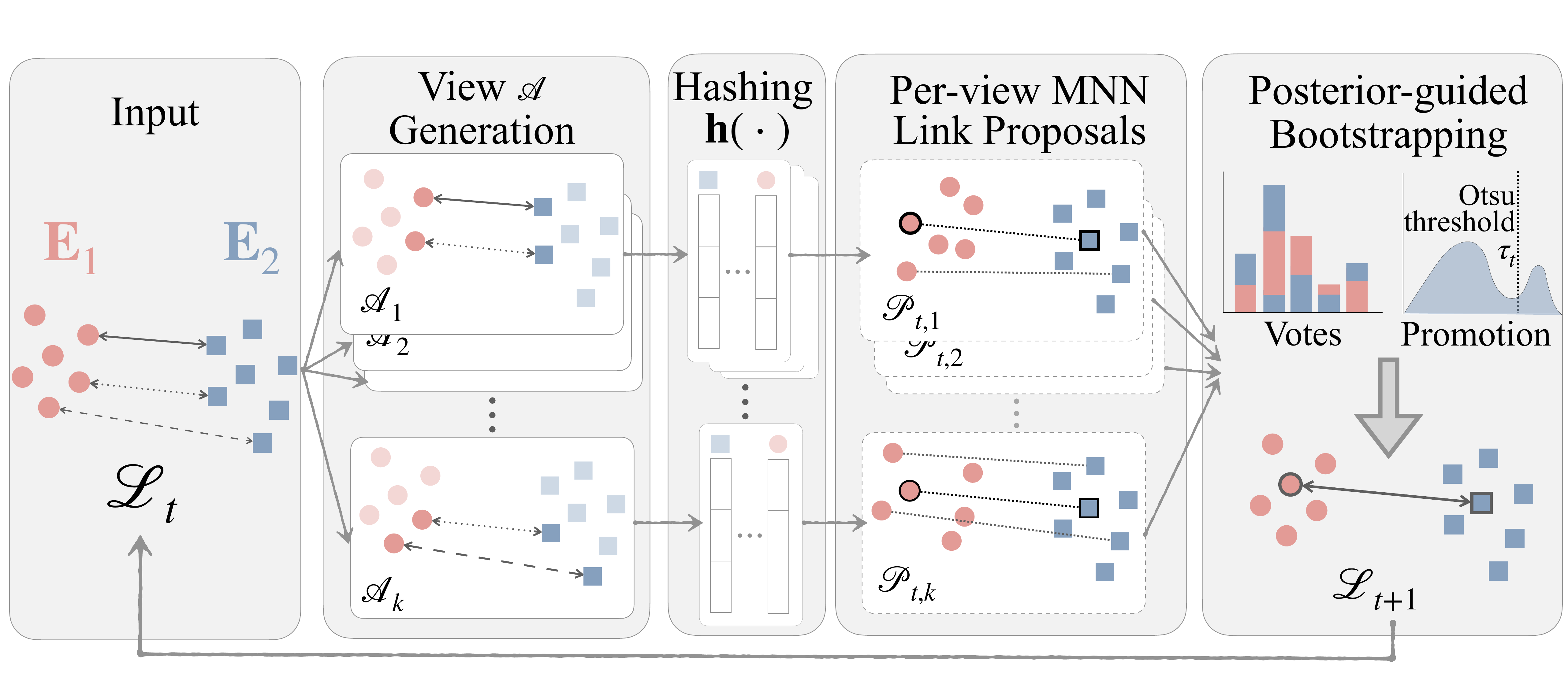}}
    \caption{The geometric embedding hashing (\mvh) framework}
    \label{fig-framework}
    \vspace{-1.0ex}
  \end{center}
\end{figure}

We next instantiate these three steps in full.

\stitle{View generation}.
At iteration $t$, we draw $m_t$ views, $\Ac_{t,1},\ak \dots,\ak \Ac_{t,m_t}\ak \subseteq\ak 
\Lc_{t-1}$, each of size $s_{t}$.  To stabilize early iterations, we
include the seed set $\Sc$ in every view.
This ensures views share a reliable core signal even when $\Lc_{t-1}$ is small.
\looseness = -1

\etitle{View sampling}.
The quality of a view depends on anchor diversity: clustered anchors yield
redundant hash coordinates and poor stability. Hence, we choose
each view by greedy Furthest Point Sampling (FPS)~\cite{GONZALEZ1985293}
on one side of the paired anchors.
Let $\Bc_{t-1}:=\Lc_{t-1}\setminus
\mathcal{S}$, \ie anchors bootstrapped at previous iteration $t-1$.  
To form view $\Ac_{t,k}$, FPS starts from a random element of $\Bc_{t-1}$
and then iteratively adds the paired anchor whose (chosen-side) vector maximizes
its minimum distance to the current subset. This encourages views with widely
separated anchors and improves stability (see
Appendix~\ref{app:well-conditioned-views} for detailed analysis).
\looseness = -1

\etitle{View scheduling}.
As bootstrapping (iteration) progresses, the anchor pool $\Lc_{t-1}$ grows. We increase
view diversity by sampling more views from $\Lc_{t-1}$ while reducing anchors
per view. Let $g_t := \frac{|\Lc_{t-1}|}{|\mathcal{S}|}$.
We define a scaling factor $\mathrm{sf}_t := 1 + c\log g_t$
for some $c\ge 0$; we set the number of views at iteration
$r$ to $m_t := \lceil m_0\,\mathrm{sf}_t \rceil$ and the size of
each view to $s_t := \left\lceil \rho_0\,|\Lc_{t-1}|/\mathrm{sf}_t \right\rceil$, with $\rho_0\in(0,1]$.

Intuitively, this increases the number of views as anchors grow, while making each
view smaller so that it preferentially reflects local geometry. Further, this stabilizes the per-anchor coverage, \ie each anchor appears in roughly a constant
number of views per iteration (see Appendix~\ref{app:view-scheduling}).

\stitle{Per-view link proposal}. 
Given a view $\Ac_{t,k}\ak =\ak \{(a_1,\ak a'_1),\ak \dots,\ak (a_{s_t},\ak
a'_{s_t})\}$, we construct a view-specific hash for each point and propose links
by similarity search in hash space.

\etitle{Kernelized hashes}.
Raw distance\-/to\-/anchor vectors can be dominated by far anchors where
cross\-/model distances are least consistent. We therefore apply a monotone
kernel that downweights large distances.
For $u \in \Ec_1$ (analogously for $v \in \Ec_2$), we use kernelized hash
$\big(\mathbf{h}_{t,k}(u)\big)_j := \exp\!\left(-\frac{\dist(u,a_j)}{\sigma_{t,k}}\right)$, for $j=1,\ak
\dots,\ak s_t$, where $\dist(\cdot,\cdot)$ is the cosine distance between
$\ell_2$-normalized embeddings and $\sigma_{t,k}>0$ is a per\-/view bandwidth by median heuristic: the median of the nonzero pairwise $\ell_2$ distances between hashes within the view.  
This preserves the rank ordering induced by distances while emphasizing the
short-range regime where Theorem~\ref{thm-local-iso} applies.
\looseness = -1

\etitle{Mutual nearest neighbor search in hash space}.
We compare hash vectors by cosine similarity in $\mathbb{R}^{s_t}$. To mitigate
hubness in nearest\-/neighbor search, we use CSLS~\cite{LampleCRDJ18} with
parameter $k_{\mathrm{CSLS}}$.
Let $\mathrm{csls}_{t,k}(u,v)$ denote the resulting view\-/specific similarity score.
We then propose only mutual nearest neighbors (MNN): a pair $(u,v)$ is
proposed if $v$ maximizes $\mathrm{csls}_{t,k}(u,\cdot)$ and $u$ maximizes
$\mathrm{csls}_{t,k}(\cdot,v)$. The set of all proposed pairs is denoted
$\mathcal{P}_{t,k}$.
\revise{
At multi-million scale, computing hashes and scoring $\kw{csls}_\Ac$ for all points 
can be costly, so we restrict hash construction and MNN to a local
$k_{\mathrm{NN}}$ neighborhood around each view's anchors (Appendix~\ref{app:per-view-impl}). 
This restriction is consistent with Section~\ref{subsec-evidence}: we save cost on candidate pairs
that are outside a view's local neighborhood which are
less reliable and contribute little useful voting signal.}

\stitle{Posterior-guided bootstrapping}. 
This step expands $\Lc_{t-1}$ from tiny seeds ($\Lc_{0} = \Sc$) by promoting
only links that are consistently supported across many views. This is exactly
the signal/noise separation observed in Section~\ref{sec-hash}.  

\etitle{Vote counts}.
Storing statistics for all $|\Ec_1|\cdot|\Ec_2|$ pairs is infeasible. We track
only pairs that are proposed at least once. Define the candidate universe up to
iteration $t$ by $\Cc_t := \bigcup_{r=1}^{t} \bigcup_{k=1}^{m_r} \Pc_{r,k}$. 
Each view contributes one binary vote for each $(u,v)\in\mathcal{C}_t$, denoted
by  $Y_{t,k}(u,v) := \mathbb{I}\big[(u,v)\in \mathcal{P}_{t,k}\big]$.
Define the cumulative number of positive votes as
$\nu_{(u,v),t} := \sum_{r=1}^{t}\ \sum_{k=1}^{m_r} Y_{r,k}(u,v)$, 
and the total number of views as $N_{\le t} := \sum_{r=1}^{t} m_r$.

\etitle{Beta--Bernoulli posterior as \revise{link confidence score}}.
We model success probability of $(u,\ak v)\ak\in\ak \Cc$ as $\theta_{(u,v)}$ with
an uninformative $\mathrm{Beta}(1,1)$ prior. By Beta--Bernoulli conjugacy,
\looseness = -1
\[
\theta_{(u,v)} \mid \Yc_{(u,v),0:t}
\sim
\mathrm{Beta}\bigl(1+\nu_{(u,v),t},\;1+N_{\le t}-\nu_{(u,v),t}\bigr),
\]
where $\Yc_{(u,v),0:t}:=\{Y_{r,k}(u,v)\}_{r=0,\dots,t;\,k=1,\dots,m_r}$.
\revise{We use the posterior mean
$\hat{\theta}_{(u,v),t} = (1+\nu_{(u,v),t})/(2+N_{\le t})$
as the link confidence score.}

\etitle{Adaptive link promotion}.
To identify links without tuning a fixed threshold, we compute an
iteration\-/specific threshold $\tau_t$ using Otsu's rule~\cite{4310076} applied
to the histogram of $\{\hat{\theta}_{(u,v),t}\}_{(u,v)\in\mathcal{C}_t}$.  We
promote all pairs with $\hat{\theta}_{(u,v),t}\ak \ge\ak \tau_t$.
Since each vector should match at most one target vector, we enforce one\-/to\-/one
matching by greedily selecting non-conflicting promoted pairs in decreasing
$\hat{\theta}_{(u,v),t}$.
\looseness = -1

As Otsu maximizes between\-/class variance, $\tau_t$ adapts to the
typically bimodal separation between consistently supported links (high
posterior) and transient collisions (low posterior), concentrating $\Lc_t$ 
on stable, consensus-supported links, while filtering out distortion-driven
spurious links.
\looseness = -1

\stitle{Termination \& complexity}. 
\mvh stops when $\Lc_t$ stabilizes, \eg when no new or very few pairs are
promoted for consecutive iterations. Per view,
hash construction costs $O((|\Ec_1|+|\Ec_2|)\,s_t)$, and nearest-neighbor search
is performed in $s_t$ dimensions; total cost scales with the number
$\sum_{t=1}^{T} m_t$ of evaluated views. In practice $m_t$ is small (tens) and per-view cost
is moderate.
(See Appendix~\ref{app:stopping} for details.)
\looseness = -1

\revise{
  For extremely large-scale cases, the local-neighborhood restriction bounds the hash construction to
  $O(s_t^{2}\,k_{\mathrm{NN}})$.
  The cost of identifying each view's local set is amortized by a one-time
  $k$-nearest-neighbor index over $E_1\cup E_2$ of build cost
  $O\bigl(|E_1|+|E_2|\bigr)$, reused across all $T$ iterations and
  $\sum_t m_t$ views. (See Appendix~\ref{app:per-view-impl} for more details.)
}
\looseness = -1

\section{Effectiveness}
\label{sec-exp}

We evaluate the effectiveness of \mvh for vector linking.

\vspace{-1ex}
\subsection{Experimental Setup}

\stitle{Benchmarks}.
We used \revise{6} BEIR~\cite{thakur2021beir} text retrieval benchmarks:
\nfcorpus, \scifact, \arguana, \scidocs, \fiqa, and \fever
(see Table~\ref{tab:beir_datasets} in Appendix~\ref{app:datasets}).
Each benchmark provides a corpus $\Dc$ of documents and built-in
query\-/answer pairs for retrieval performance evaluation.

\stitle{Vector linking setup}.
Given a corpus $\Dc$, we constructed two partially overlapping corpora $\Dc_1,
\Dc_2$ as follows. We sampled an overlap set $\Omega \subset \Dc$ and split the
residual $\Dc \setminus \Omega$ uniformly at random into two disjoint sets
$\Uc_1, \Uc_2$ and set $\Dc_1 := \Omega \cup \mathcal{U}_1$ and $\Dc_2 := \Omega
\cup \mathcal{U}_2$. We controlled the overlap level via $\alpha :=
\frac{|\mathcal{D}_1 \cap \mathcal{D}_2|}{|\mathcal{D}_1 \cup \mathcal{D}_2|} =
\frac{|\Omega|}{|\Omega| + |\mathcal{U}_1| + |\mathcal{U}_2|}$. We embedded
$\Dc_1$  and $\Dc_{2}$ with encoders $f_1$ and $f_2$, respectively, to obtain
the two embedding clouds:
$\Ec_1\ak :=\ak f_1(\mathcal{D}_1)$ and $\Ec_2\ak :=\ak f_2(\mathcal{D}_2)$.
We set the ground\-/truth correspondence set $M^* := \{(f_1(x), f_2(x)) : x \in
\Omega\}$. All methods were given only $\Ec_1, \Ec_2$ and were not told
$\Omega$, nor had access to $f_{i}$ or $\Dc_{i}$.
\looseness = -1

By default, we drew a seed set $\Sc\subset M^*$ by uniformly sampling overlap
items. We report results for three seed sizes
${|\Sc|}\ak \in \ak \{15,\ak
20,\ak 30\}$. We also evaluated out-of-domain (OOD) seeds which are drawn from a
different dataset (Section~\ref{subsec-ood}).
In each case, all methods received the same $\Sc$.

\stitle{Models}.  We used 5 pairs of major embedding models:
(a) \mistral (Mistral-embed \cite{mistral}) {\sc vs.} \openai
(Text-embedding-3-small \cite{openai}), 
(b) \gte (GTE-Qwen2-7B-instruct \cite{li2023towards}) {\sc vs.} \mistral,
(c) \gte\ {\sc vs.} \openai, 
(d) \qwen (Qwen3-Embedding-8B \cite{qwen3embedding}) {\sc vs.} \kalm
(KaLM-Embedding-Gemma3-12B \cite{zhao2025kalmembeddingv2superiortrainingtechniques}),  and
(e) \qwen\ {\sc vs.} \openai.

\begin{table}[t!]
  \centering
\scriptsize
\setlength{\tabcolsep}{0.2ex}
\renewcommand{\arraystretch}{1.2}
\caption{Vector linking at overlap $\alpha$=0.3, seeds $|\Sc|$=15:
  {\normalfont each cell shows {\bf precision/recall/F1}(\%); \textbf{bold}
    indicates best per column.}}
    \label{tab:main_results}
    \begin{tabular}{l|c|c|c|c|c}
    \hline
    \multirow{2}{*}{Method} & Qwen-OpenAI & GTE-OpenAI & GTE-Mistral & Mistral-OpenAI & Qwen-Kalm \\
      & NFCorpus & SciFact & ArguAna & SciDocs & FiQA \\
    \hline
    \linear & 2.8/7.2/4.0 & 2.2/5.0/3.1 & 0.4/0.3/0.3 & 2.2/1.7/1.9 & 0.7/0.8/0.8 \\
    \cca & 46.7/10.6/17.3 & 29.3/5.2/8.8 & 25.2/3.9/6.8 & 11.0/1.4/2.4 & 12.5/0.7/1.4 \\
    \mlp & 36.1/3.3/6.0 & 14.2/1.4/2.5 & 11.0/0.7/1.2 & 10.4/0.3/0.5 & 9.2/0.1/0.3 \\
    \at{RCSLS} & 38.9/3.0/5.6 & 28.6/2.1/3.8 & 26.2/0.2/0.4 & 32.6/0.3/0.6 & 15.8/0.2/0.4 \\
    \pr & 52.5/11.8/19.3 & 37.9/5.7/9.9 & 30.8/4.8/8.4 & 15.6/1.6/2.9 & 20.6/1.3/2.4 \\
    \ugw & \revise{14.8/2.4/4.2} & \revise{15.9/2.5/4.4} & \revise{5.9/0.1/0.2} & \revise{3.8/0.1/0.1} & \revise{2.4/0.0/0.0} \\
    \ao & 22.5/22.5/22.5 & 5.6/5.6/5.6 & 0.6/0.6/0.6 & 0.2/0.2/0.2 & 0.1/0.1/0.1 \\
    \textbf{\mvh} & \textbf{82.1}/\textbf{95.6}/\textbf{88.3} & \textbf{83.2}/\textbf{89.1}/\textbf{86.0} & \textbf{77.1}/\textbf{84.5}/\textbf{80.7} & \textbf{82.8}/\textbf{81.6}/\textbf{82.2} & \textbf{79.8}/\textbf{79.9}/\textbf{79.8} \\
    \hline
    \end{tabular}
\end{table}

\stitle{Baselines}. 
As there is no prior method that directly tackles vector linking with partial
overlap, we adapt embedding alignment methods by incorporating
ideas from \mvh: they first align $\Ec_{1}$ to $\Ec_{2}$ by supervising on
$\Sc$, yielding a shared embedding space with aligned $\Ec_{1}$ and $\Ec_{2}$;
they then use
CSLS based MNN search to identify links as \mvh does for link proposal. Specifically,
we trained 5 alignment methods:

\begin{itemize}[leftmargin=3ex, itemsep=0ex, topsep=0.0ex, partopsep=0.0ex]
\item \linear:  regression with MSE~\cite{mikolov2013Exploiting}.
\item \cca: canonical correlation analysis~\cite{lu-etal-2015-deep}.
\item \mlp: two-layer MLP trained with cosine loss on seeds.
\item \rcsls: RCSLS~\cite{joulin2018loss}, a retrieval-based linear mapping optimized via SGD.
\item \pr: orthogonal Procrustes alignment on seeds (closed-form SVD)~\cite{smith2017offline}.
\revise{\item \ugw: unbalanced Gromov-Wasserstein~\cite{ugw}, quadratic OT on intra-space distance matrices, warm-started from a seed-biased coupling}.

\end{itemize}
We also tested \ao(Anchor Optimization~\cite{ao}), which was given overlap size to adapt to the task.

\stitle{Metrics}.
We measured the output of each method, \ie set of predicted links $M$,
by (a) $\mathrm{Precision} \ak :=\ak \frac{|M\cap
M^*|}{|M|}$, (b) $\mathrm{Recall} := \frac{|M\cap M^*|}{|M^*|}$, and (c)
$\mathrm{F1} :=
\frac{2\,\mathrm{Precision}\cdot\mathrm{Recall}}{\mathrm{Precision}+\mathrm{Recall}}$.
In tables, we report all the metrics, and bold indicates the best. 

Further details can be found in Appendix~\ref{app:datasets}.

\subsection{Performance on Vector Linking}

\stitle{Overall}.
We first evaluated the performance of all methods for linking all 5 pairs of
embedding models across all datasets except \fever
(reserved for scalability test). Table~\ref{tab:main_results} summarizes
their recall, precision, and F1 with only 15 seed anchors in $\Sc$ for an
overlap ratio $\alpha = 30\%$ (see
Tables~\ref{tab:single_nfcorpus_mistral_openai}-\ref{tab:single_fiqa_qwen_kalm}
in Appendix~\ref{app:more_model_pairs} for a complete report).

\begin{table}[t]
\centering
\scriptsize
\caption{Scalability on \fever{} 
{\normalfont\footnotesize on \mistral$\leftrightarrow$\openai, single A100 80\,GB GPU,
overlap $\alpha=0.3$, $|\Sc|=30$. \textbf{Bold} marks the best per column;
runtime is end-to-end wall-clock seconds. (\ugw cannot complete on \fever so it is not reported.)}}
\label{tab:fever_scalability}
\begin{tabular}{lccc}
\toprule
Method & Precision (\%) & Recall (\%) & Runtime (s) \\
\midrule
\linear & 4.13 & 0.00 & 4420 \\
\cca    & 5.22 & 0.61 & \textbf{1613} \\
\mlp    & 1.93 & 0.01 & 4414 \\
\rcsls  & 6.40 & 0.11 & 3348 \\
\pr     & 9.03 & 0.73 & 4494 \\
\textbf{\mvh} & \textbf{93.8} & \textbf{68.9} & 3328 \\
\bottomrule
\end{tabular}
\vspace{-0.5ex}
\end{table}

The results are very encouraging: our method \mvh consistently outperforms all other methods across all cases by a substantial
margin. For example, on \fiqa, we need only 15 seeds to
recover the entire overlap between \qwen and \kalm models, achieving 79.9\% in recall, 79.8\% in precision, and 79.8\% in F1-score, respectively, while the second-best method achieves only 1.3\%, 20.6\%, and 2.4\%. The results for
linking across other model pairs and datasets are similar.

\begin{table*}[!htbp]
  \centering
  \scriptsize
\setlength{\tabcolsep}{1.5ex} %
\renewcommand{\arraystretch}{1.1}%
  \caption{Vector linking on \scifact{} (\mistral{}$\leftrightarrow$\openai{}):
  {\normalfont each cell reports \textbf{precision/recall/F1} (\%). Best values per metric are \textbf{bolded}.}}
  \label{tab:varyAS}
  \begin{tabular}{l|ccc|ccc|ccc}
  \hline
  \multirow{2}{*}{Method} & \multicolumn{3}{c|}{Overlap .15} & \multicolumn{3}{c|}{Overlap .20} & \multicolumn{3}{c}{Overlap .30} \\
   & Seeds 15 & Seeds 20 & Seeds 30 & Seeds 15 & Seeds 20 & Seeds 30 & Seeds 15 & Seeds 20 & Seeds 30 \\
  \hline
  \linear & 4.2/4.4/4.3 & 6.8/7.8/7.3 & 11.0/13.3/12.1 & 4.2/4.5/4.3 & 7.2/8.1/7.6 & 10.6/16.5/12.9 & 2.9/5.2/3.7 & 6.1/11.5/8.0 & 10.9/23.6/14.9 \\
  \cca & 16.7/5.5/8.3 & 21.3/11.5/14.9 & 34.4/29.2/31.6 & 16.8/4.9/7.6 & 28.9/10.0/14.9 & 42.0/26.2/32.3 & 26.3/4.9/8.2 & 33.4/10.0/15.4 & 53.7/26.4/35.4 \\
  \mlp & 15.4/1.6/2.8 & 19.4/4.2/6.9 & 26.0/11.2/15.7 & 15.1/1.9/3.3 & 22.7/5.2/8.5 & 35.8/13.2/19.2 & 21.7/1.3/2.5 & 29.0/3.8/6.7 & 40.3/11.8/18.3 \\
  \rcsls & 46.2/2.3/4.3 & 39.9/4.5/8.1 & 39.7/12.0/18.4 & 47.9/1.8/3.5 & 39.4/3.2/5.9 & 44.3/9.7/15.9 & 43.3/1.4/2.7 & 39.6/2.8/5.2 & 53.3/9.9/16.6 \\
  \pr & 17.2/4.7/7.4 & 26.3/12.0/16.5 & 39.8/29.2/33.7 & 23.4/6.8/10.5 & 35.8/11.5/17.4 & 48.3/27.0/34.6 & 29.3/5.3/8.9 & 41.6/10.5/16.8 & 62.3/29.6/40.1 \\
  \revise{\ugw} & \revise{3.8/0.4/0.7} & \revise{3.7/0.4/0.7} & \revise{3.6/0.4/0.7} & \revise{1.1/0.1/0.2} & \revise{1.1/0.1/0.2} & \revise{1.1/0.1/0.2} & \revise{4.7/0.3/0.5} & \revise{4.7/0.3/0.5} & \revise{3.5/0.2/0.4} \\
  \ao & 4.2/4.2/4.2 & 9.4/9.4/9.4 & 25.4/25.4/25.4 & 2.7/2.7/2.7 & 6.5/6.5/6.5 & 32.0/32.0/32.0 & 6.5/6.5/6.5 & 27.2/27.2/27.2 & 43.9/43.9/43.9 \\
  \textbf{\mvh} & \textbf{63.3}/\textbf{84.0}/\textbf{72.2} & \textbf{63.0}/\textbf{81.2}/\textbf{70.9} & \textbf{62.3}/\textbf{82.3}/\textbf{70.9} & \textbf{73.6}/\textbf{87.0}/\textbf{79.7} & \textbf{73.2}/\textbf{85.3}/\textbf{78.8} & \textbf{72.2}/\textbf{85.8}/\textbf{78.4} & \textbf{84.0}/\textbf{87.7}/\textbf{85.8} & \textbf{83.8}/\textbf{86.5}/\textbf{85.1} & \textbf{83.6}/\textbf{87.0}/\textbf{85.3} \\
  \hline
  \end{tabular}
\end{table*}

\stitle{Varying overlap and seeds}.  We further evaluated the impact of overlap ratio $\alpha$ and number of seed anchors $|\Sc|$.  The results over \scifact~dataset for linking \mistral and \openai models are shown in Table~\ref{tab:varyAS}
(see Appendix~\ref{app:more_model_pairs} for more). \looseness = -1

Our method, \mvh, is particularly robust and stable to the seed anchors; for
instance, with 15 seeds, it already performed as well as it did with 30 seeds,
while all other methods required more seeds to improve performance. All methods
performed better with larger overlap, but the large gap between \mvh and others
remained stable and significant.

\revise{
  \stitle{Scalability}. We also evaluated the scalability of \geh on FEVER (5.4 million
corpus) using \mistral $\leftrightarrow$ \openai embeddings, fixing the overlap ratio
to $\alpha=0.3$, seed budget $|\Sc|=30$. All methods are run on a single
NVIDIA A100 (80\,GB). For \mvh, the per-view local-set restriction described
in Section~\ref{sec-linking}  is activated due to the size of \fever{}.

Table~\ref{tab:fever_scalability} reports the precision, recall,
and end\-/to\-/end wall\-/clock runtime on a single A100. \mvh attains $93.8\%$
precision and $68.9\%$ recall in $3328$\,s. \mvh
remains within the same order of magnitude as the fastest baseline (\cca,
$\sim$$2.1\times$ slower) and is faster than all other baselines
despite its iterative design. The gap is large: relative to \cca,
\mvh{} improves precision from $5.22\%$ to $93.8\%$  and
recall from $0.61\%$ to $68.9\%$, at only $\sim$$2.1\times$
the wall-clock cost.

Note that CSLS+MNN link extraction is shared by all alignment baselines, so the
incremental cost of \mvh{} comes only from evaluating multiple views.
Crucially, each view operates in the low-dimensional distance-to-anchor hash
space induced by a small anchor set, rather than re-matching in the original
embedding space. The iterative procedure thus performs many moderate-cost
hash-space retrievals rather than repeated dense searches over full embedding
clouds. 
}

\eat{
\begin{figure}[t!]
  \begin{center}
    \centerline{\includegraphics[width=\columnwidth]{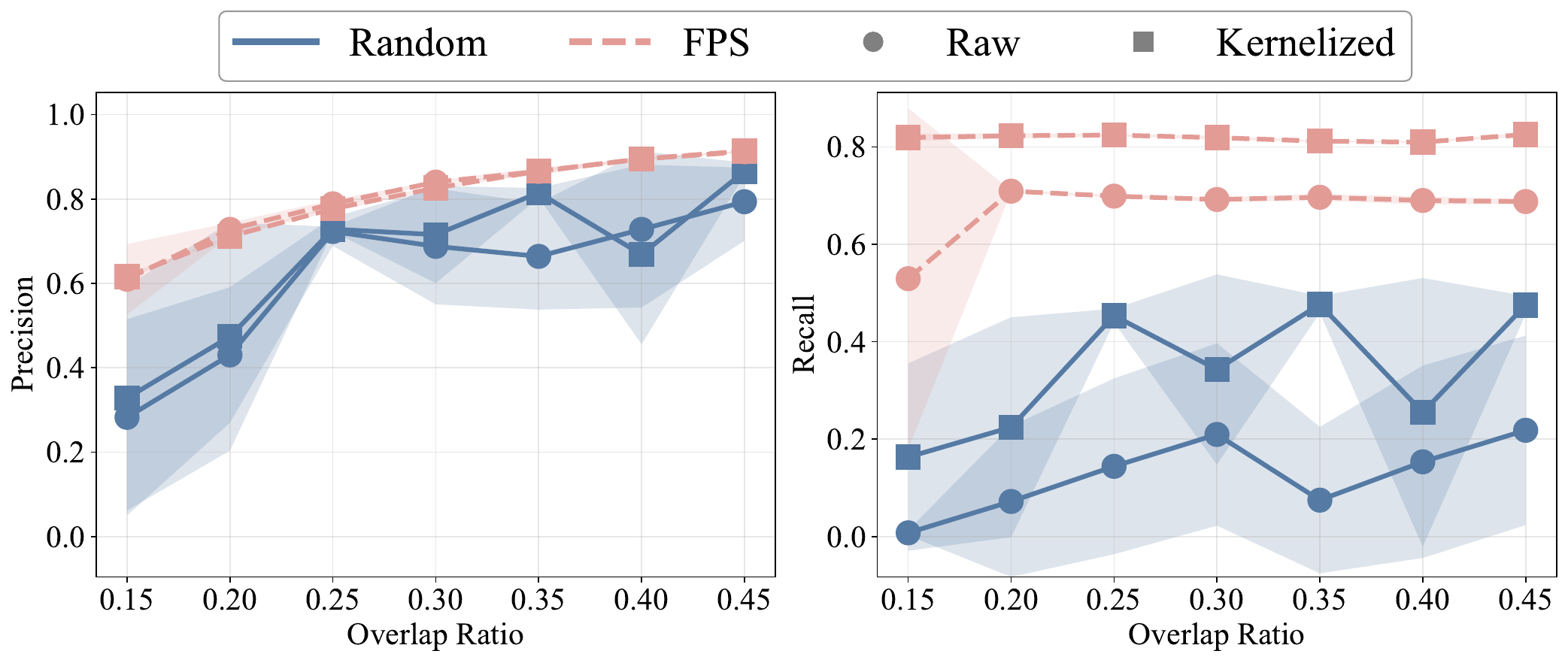}}
    \caption{Impact of view construction and distance encoding:
    {\normalfont\footnotesize on \scidocs\ with \mistral\ {\sc vs.} \openai, we compared the precision (left) and recall (right) of view strategies (FPS, Random), each with Kernelized or Raw distances.Shaded areas show variance ($\pm1$ std).
    }}
    \label{fig:viewrbf}
  \end{center}
  \vspace{-2.0ex}
\end{figure}

\subsection{Ablation Study}
\label{sec:ablation}

We ablate both the view generation strategy and the choice of hash encoding in
\mvh. With 15 seeds as fixed supervision, we vary the overlap ratio from 0.15 to
0.45 and compare two view generation strategies, namely, \emph{Random} (uniformly
sampled) and \emph{FPS} (furthest-point-sampled); each strategy is combined with two
distance encodings: \emph{Raw} (unprocessed distance $\dist$), \emph{Kernelized}
($\exp(-\dist/\sigma)$).  By default \mvh uses FPS with Kernelized
distance (Section~\ref{sec-linking}); \revise{see Appendix~\ref{app:ablation}
for per-component ablation details.}

Fig.~\ref{fig:viewrbf} reports the precision and recall of vector linking on \scidocs~with \mistral and \openai.
Kernelized encoding consistently dominates the Raw variant
for each view strategy, especially in recall, confirming that emphasizing
short\-/range distances improves robustness.
For view constructions, FPS always gives the better performance, particularly in
recall, indicating that dispersed geometric views are preferable in the presence
of cross-model distance distortion due to their better stability
(see analysis in Appendix~\ref{app:well-conditioned-views}).
}

\revise{
\begin{table}[t]
  \centering
  \scriptsize
  \setlength{\tabcolsep}{6pt}
  \caption{
    Ablation of \mvh\ components
  {\normalfont\footnotesize on \scidocs\ with \mistral\ {\sc vs.} \openai, at $\alpha$=0.15 and seeds $|\Sc|$=15: 
  \mvh denotes the complete pipeline: FPS view sampling, Kernelized signature, adaptive view scheduling, multi-view posterior aggregation, and bootstrapping. 
  Each subsequent row removes one component relative to \mvh; the row \mbox{$-$ FPS \& Kernel} removes both per-view components jointly. 
  Each cell reports the seed-level mean $\pm$ standard deviation (\%) of the corresponding metric. 
  A $\dag$ marks cells where the seed-level coefficient of variation ($\sigma/\mu$) exceeds $10\%$.
  }
  }
  \label{tab:scidocs-ablation}
  \begin{tabular}{lccc}
  \toprule
  Variant & Precision (\%) & Recall (\%) & F1 (\%) \\
  \midrule
  \mvh                               &  62.1 $\pm$  1.1     &  81.7 $\pm$  0.7     &  70.5 $\pm$  0.6     \\
  \;\;$-$ Kernel                     &  61.0 $\pm$  8.3\dag &  52.9 $\pm$ 35.0\dag &  51.0 $\pm$ 33.3\dag \\
  \;\;$-$ FPS                        &  32.8 $\pm$ 26.7\dag &  16.3 $\pm$ 19.2\dag &  20.9 $\pm$ 23.3\dag \\
  \;\;$-$ FPS \& Kernel              &  28.3 $\pm$ 23.3\dag &   0.8 $\pm$  0.5\dag &   1.4 $\pm$  1.1\dag \\
  \;\;$-$ Adaptive schedule          &  33.4 $\pm$  0.6     &  61.8 $\pm$  1.6     &  43.4 $\pm$  0.9     \\
  \;\;$-$ Multi-view voting          &  24.0 $\pm$  2.6\dag &  39.4 $\pm$  5.0\dag &  29.8 $\pm$  3.4\dag \\
  \;\;$-$ Bootstrapping              &   1.9 $\pm$  0.6\dag &   2.1 $\pm$  1.2\dag &   1.9 $\pm$  0.8\dag \\
  \bottomrule
  \end{tabular}
\end{table}

\subsection{Ablation Study}
\label{sec:ablation}

\stitle{Staged ablation}.
We ablate the five components of \mvh\ on \scidocs\ (\mistral\ {\sc vs.}\ \openai) with $\alpha\!=\!0.15$ and $|\Sc|\!=\!15$.
For each variant we run $10$ independent trials, each with a fresh random seed
controlling both the draw of $\Sc$ and the internal randomness of view
generation, and report the across-trial mean and standard deviation
($\mu\pm\sigma$).

Each variant is \mvh\ but with the named component swapped for a simpler
default. The replacement defaults are:
\looseness = -1

\begin{itemize}[itemsep=0.0ex,topsep=0.0ex,partopsep=0ex]
\item \textbf{$-$ Kernel}: use the raw distance vector $r_\Ac$ instead of the kernelized hash $h_\Ac$ (Section~\ref{sec-linking}).
  \item \textbf{$-$ FPS}: draw view anchors uniformly at random from $\Lc_{t-1}$ instead of by FPS (Section~\ref{sec-linking}).
  \item \textbf{$-$ FPS \& Kernel}: both per-view defaults applied jointly (raw
    hashes plus random view draws).
  \item \textbf{$-$ View scheduling}: freeze the schedule at iteration zero, $\rho_t\equiv\rho_0$ and $m_t\equiv m_0$ (Section~\ref{sec-linking}).
  \item \textbf{$-$ Multi-view voting}: use single-view MNN proposal, $m_t{=}1$ and $\Ac_{t,1}{=}\Lc_{t-1}$ (Section~\ref{sec-hash}).
  \item \textbf{$-$ Bootstrapping}: run a single iteration on the $\Sc$; no anchor-pool growth (Section~\ref{sec-linking}).
\end{itemize}

Table~\ref{tab:scidocs-ablation} reports the results. Bootstrapping and
multi-view voting account for most of the absolute performance.
Without bootstrapping, only the 15 seeds were available as
anchors and \mvh\ collapsed to $1.9/2.1/1.9$ for Precision(\%)/ Recall(\%)/ F1 (\%); without multi-view voting,
single-view link proposals could not separate true links from
distortion-driven collisions and performance fell to $24.0/39.4/29.8$.

Removing FPS or the kernel signature makes performance highly variable (\eg
recall $\sigma$ rose from $0.7$ to $19$--$35$): these two per-view components
act as variance reducers, FPS by spreading anchors so each view stays
well-conditioned and the kernel by suppressing the long-range distance regime
that decorrelates across encoders (Section~\ref{subsec-evidence}).}

\begin{figure}[t!]
  \begin{center}
    \centerline{\includegraphics[width=\columnwidth]{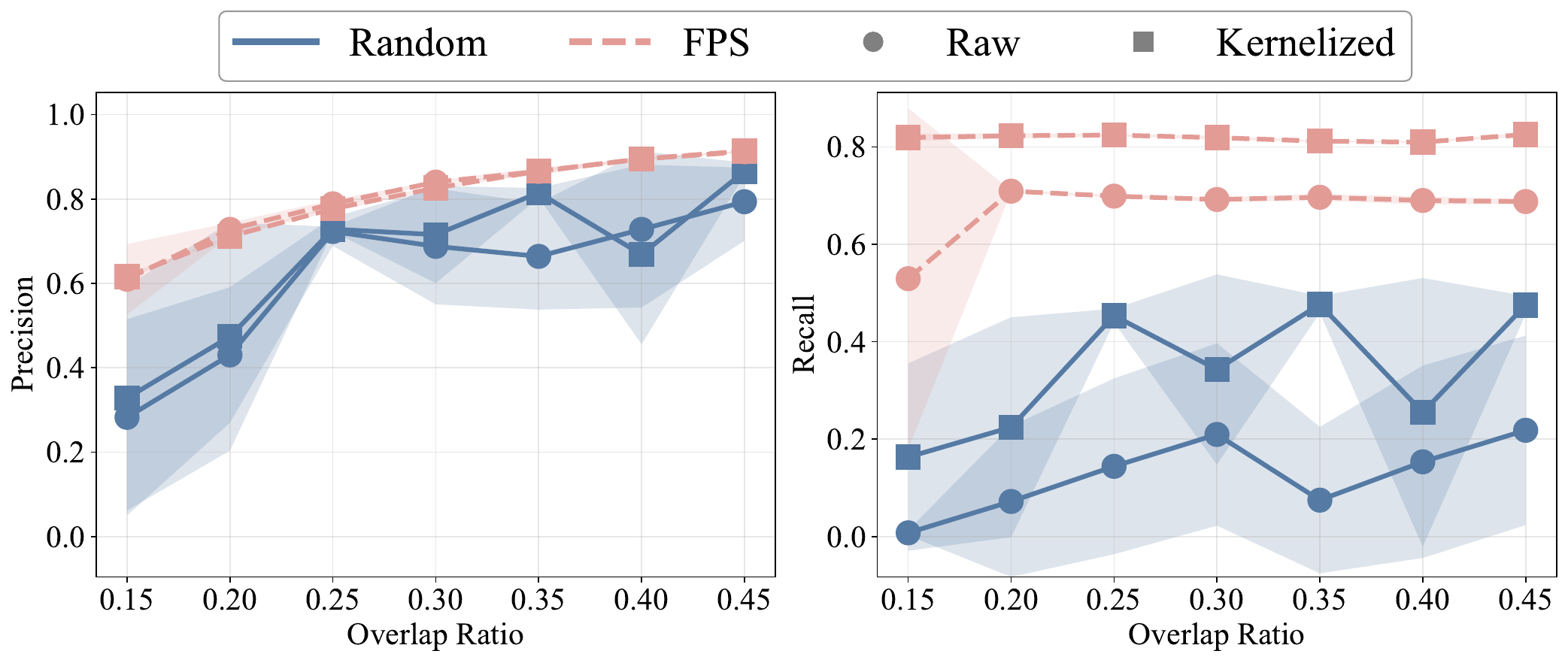}}
    \caption{Impact of view construction and distance encoding:
    {\normalfont\footnotesize on \scidocs\ with \mistral\ {\sc vs.} \openai, we compared the precision (left) and recall (right) of view strategies (FPS, Random), each with Kernelized or Raw distances.Shaded areas show variance ($\pm1$ std).
    }}
    \label{fig:viewrbf}
  \end{center}
  \vspace{-2.0ex}
\end{figure}

\stitle{View generation and hash encoding}.
We further examined the view generation strategy and hash encoding in
\mvh. With 15 fixed seeds, we vary the overlap ratio from 0.15 to
0.45 and compare two view generation strategies, namely, \emph{Random} (uniformly
sampled) and \emph{FPS} (furthest-point-sampled), each combined with two
distance encodings: \emph{Raw} (unprocessed distance $\dist$), \emph{Kernelized}
($\exp(-\dist/\sigma)$).

Fig.~\ref{fig:viewrbf} reports the precision and recall of vector linking on \scidocs~with \mistral and \openai.
Kernelized encoding consistently dominates the Raw variant
for each view strategy, especially in recall, confirming that emphasizing
short\-/range distances improves robustness.
For view constructions, FPS always gives the better performance, particularly in
recall, indicating that dispersed geometric views are preferable in the presence
of cross-model distance distortion due to their better stability
(see analysis in Appendix~\ref{app:well-conditioned-views}).

\subsection{Robustness to Domain Shifts (OOD Analysis)}
\label{subsec-ood}

In practice, exact vectors stored in a private index may be inaccessible; users
can encode a small public corpus with both models and use those embeddings as
references. To simulate this, we replace in-domain anchors with
\emph{out-of-domain (OOD)} anchors drawn from a different dataset.

For each target dataset, we construct \mistral and \openai embeddings, enforce a 30\% overlap, and draw 30 OOD seeds from a separate reference dataset. 
Fig.~\ref{fig:ood-heatmap} reports the precision and recall across all reference-target dataset pairs:
precision remains in 77-87\% range and recall mostly between 80-97\%. Additional results across seed budgets and
overlap ratios are given in Appendix~\ref{app:ood-heatmaps}. This indicates that small OOD corpora are
sufficient to serve as anchors for geometric hashing to link private embedding clouds.
\looseness = -1

\begin{figure}[t!]
\vspace{-1.5ex}  
  \begin{center}
    \centerline{\includegraphics[width=\columnwidth]{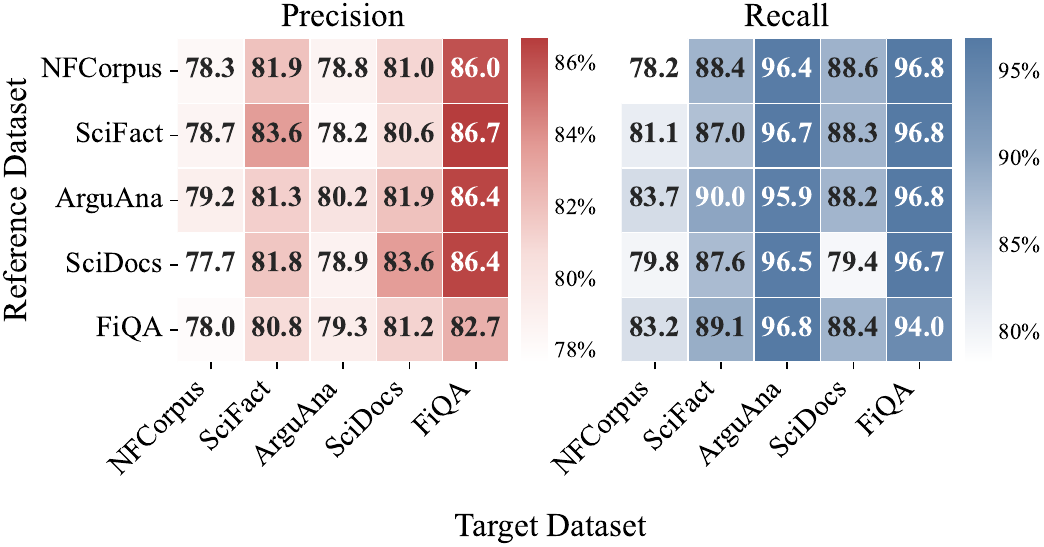}}
    \caption{Out-of-domain anchor transfer:
    {\normalfont\footnotesize
    Precision (left) and recall (right) of our method on five target datasets
    (columns) when supervised seeds are drawn from an out-of-domain reference
    dataset (rows). All runs use 30\% overlap in the target and 30 OOD seeds.}}
    \label{fig:ood-heatmap}
  \end{center}
\vspace{-1.5ex}  
\end{figure}

\section{Applications}
\label{sec-applications}

Finally, we demonstrate applications of vector linking. 

\stitle{Vector database integration}.
We demonstrate its benefit for \emph{vector database integration}, enabling
unified retrieval across vector databases embedded by distinct encoders.

\eetitle{Setup}. We used the integration protocol of \cite{LA2M},
which learns clustered Procrustes over known paired anchors to transform one
vector database and merge it with the target database for unified
querying. Instead of assuming paired anchors are given, we used
\mvh to infer links across databases and then apply the integration
protocol. Following \cite{LA2M}, we evaluated the integrated database 
via the recall@100 and NDCG@100 of benchmark queries.
\looseness = -1

\etitle{Baselines}. We compared against (i) \emph{Random}, random anchor pairing; (ii) \emph{Seed}, mapping learned from seed anchor pairs only; and (iii) \emph{Union} retrieval without cross-space mapping (directly taking the union of two databases). As an optimal reference, we also used a single model to re-embed the full unioned corpus encoded by the two databases, where the retrieval performance is the theoretical upper limit.
\looseness = -1

\etitle{Results}. Using the split in Section~\ref{sec-exp} with answer-free overlap (Appendix~\ref{sec:appendix_la2m}), we evaluate query performance of integrating \mistral and
\openai databases, with overlap ratio $\alpha$ varying from 5\% to 40\% and 30 seeds. 
Figure~\ref{fig:vecdb_integration} reports results over \scifact.
Our method substantially outperforms all baselines on both Recall@100 and NDCG@100, with performance improving as overlap increases, approaching to the theoretical limit
of using \mistral or \openai alone.

\begin{figure}[t!]
\vspace{-1.5ex}  
  \begin{center}
    \centerline{\includegraphics[width=\columnwidth]{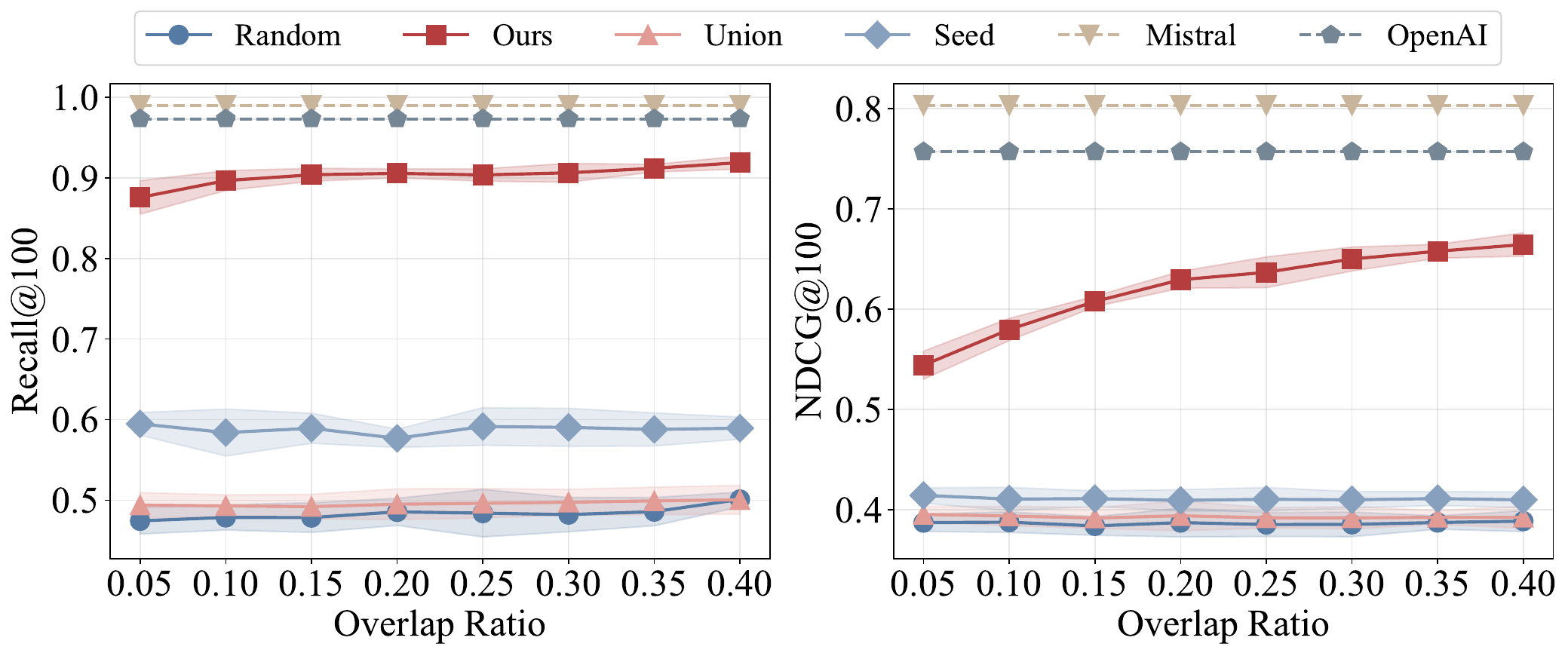}}
    \caption{Integrated vector database retrieval performance:
      \normalfont\footnotesize over \scifact, Recall@100 (left) and NDCG@100
      (right) vs. overlap ratio $\alpha$, where the overlap contains
      \emph{no benchmark answers}. \mistral and \openai are the
      theoretical upper limit of retrieval quality where we embed all objects with one single model.}
\label{fig:vecdb_integration}
\end{center}
\vspace{-1.5ex}
\end{figure}

\stitle{Cross-model clustering}.
We also demonstrate {\em cross\-/model clustering}. Using vector linking, we
stitched the two embedding sets and run clustering detection to recover clusters
spanning across datasets. Our method recovers consistent cross\-/embedding
cluster assignments for $75$--$98$\% of overlapping objects, and achieves
cluster quality within $\approx 1\%$ of the clusters obtained when all objects
are embedded by a single encoder (see Appendix~\ref{app:application} for
details).

\eat{%
\subsection{Cross-Model Clustering}
We introduce cross-model clustering, exploiting our formulation of vector linking to stitch partially overlapping embedding sets into a unified graph that recovers global semantic clusters.
\paragraph{Setup.} We use Qwen\cite{} and KaLM\cite{} to encode partially overlapping subsets of the \stackexchange~ and \reddit~ datasets\cite{geigle2021tweactransformerextendableqa}. We construct independent $k$-NN graphs for each set and fuse them into a joint graph by merging correspondence pairs into unified super-nodes. We then apply Leiden community detection to this joint graph and evaluate with V-measure against ground-truth semantic labels. To quantify how well overlapping points are aligned across models, we also report an Overlap Agreement Rate (OAR): the fraction of overlapping pairs that are assigned to the same cluster in the merged graph.
\paragraph{Results.} 
As shown in table~\ref{tab:clustering_vmeasure}, naive concatenation yields an OAR of $0\%$,  confirming that the two model-specific graphs remain effectively disconnected and overlapping points are assigned to unrelated clusters.  Using only the initial seeds improves connectivity but still produces low V-measures and poor overlap agreement. In contrast, Vector Linking (OURS) achieves V-measures in the range $0.65$-$0.69$, matching or slightly surpassing the strong baselines that cluster the entire corpus using embeddings from a single model, and yielding approximately 90\% overlapped agreements. These results show that our geometric hashing provides a high-fidelity bridge between embedding spaces, aligning the underlying manifold topology well enough to recover semantic clusters. Additional clustering metrics (NMI/ARI) are provided in Appendix~ \ref{sec:res}.
\begin{table}[!htbp]
    \centering
    \footnotesize
    \setlength{\tabcolsep}{2pt}
    \caption{Cross-model clustering performance for Qwen$\leftrightarrow$KaLM embeddings.
    {\normalfont\scriptsize
    Each Concat/Seed/OURS cell reports V-measure / Overlap Agreement Rate (OAR, \%)
    for a given overlap ratio (Ovlp) and seed number (Seed = number of overlapping
    pairs revealed as initial references). Qwen/KaLM columns show single-model
    V-measure (\%) obtained by clustering the full corpus encoded by each model
    separately. Bold indicates the best V-measure among (Concat/Seed/OURS) for each configuration.}}
    \label{tab:clustering_vmeasure}
    \begin{tabular}{lccccc|c|c}
    \hline
    Dataset & Ovlp & Seed & Concat & Seed & OURS & Qwen & KaLM \\
    \hline
    \multirow{4}{*}{Reddit} & \multirow{2}{*}{0.2} & 20 & 52.8/0.0 & 53.4/6.6 & \textbf{62.8/83.2} & \multirow{4}{*}{63.5} & \multirow{4}{*}{65.6} \\
     & & 30 & 52.8/0.0 & 54.0/10.5 & \textbf{63.0/84.6} & & \\
    \cline{2-6}
     & \multirow{2}{*}{0.3} & 20 & 51.7/0.0 & 53.6/4.6 & \textbf{66.5/96.9} & & \\
     & & 30 & 51.7/0.0 & 53.9/7.7 & \textbf{67.1/97.6} & & \\
    \hline
    \multirow{4}{*}{StackEx} & \multirow{2}{*}{0.2} & 20 & 62.2/0.0 & 62.5/7.1 & \textbf{67.0/75.9} & \multirow{4}{*}{68.7} & \multirow{4}{*}{68.8} \\
     & & 30 & 62.2/0.0 & 62.4/9.8 & \textbf{67.4/88.8} & & \\
    \cline{2-6}
     & \multirow{2}{*}{0.3} & 20 & 61.8/0.0 & 61.8/5.9 & \textbf{67.5/78.8} & & \\
     & & 30 & 61.8/0.0 & 62.1/8.3 & \textbf{68.3/91.9} & & \\
    \hline
    \end{tabular}
\end{table}

}%

\section{Conclusion}
\label{sec-conclusion}

We introduced vector linking, recovering cross\-/model correspondences from two
black\-/box embedding clouds under partial, unknown overlap.
Our core observation is that independently trained contrastive encoders exhibit
{\em local} cross\-/model geometric consistency. This motivates
encoder\-/invariant geometric hashing based on distance\-/to\-/anchor signatures, and we instantiate it with a multi\-/view iterative
algorithm that bootstraps a large anchor pool from a tiny seed set that promotes
short\-/range distances. Experiments across multiple benchmarks and model pairs
show robust, high\-/accuracy linking and enable downstream tasks such as vector
database integration and cross\-/model clustering.

Future work includes reducing seed assumptions, extending to multi\-/model
linking, and studying when local consistency holds beyond the current contrastive surrogate.

\clearpage
\section*{Impact Statement}
This paper presents work whose goal is to advance the field of Machine
Learning. There are many potential societal consequences of our work, none
which we feel must be specifically highlighted here.

\section*{Acknowledgements}
We thank the anonymous ICML reviewers and area chair for their constructive feedback.
This work is supported by RAEng RF\textbackslash 201920\textbackslash
19\textbackslash 319 and the Huawei-Edinburgh Joint Lab.

\bibliography{paper}
\bibliographystyle{icml2026}

\clearpage
\onecolumn
\appendix
\renewcommand{\thelemma}{\Alph{theorem}}
\renewcommand{\thecor}{\Alph{theorem}}
\renewcommand{\thetheorem}{\Alph{theorem}}
\setcounter{theorem}{0}

\section{Proofs and Additional Details of Section~\ref{sec-local}}
\label{app:sec2-proofs}

\subsection{Localizing the alignment term $\varphi_{\mathrm{align}}$ in $\Lc_{\lambda}(f)$}

Fix an anchor point $x\in\Mc$. Let $\exp_x:T_x\Mc\to\Mc$ denote the Riemannian
exponential map. %
For any positive view $x^+$ satisfying (A2), let $v:=v(x,x^+)\in T_x\Mc$ be
the geodesic displacement, so $x^+=\exp_x(v)$ and $\|v\|=d_{\Mc}(x,x^+)$.
We write $O(\|v\|^k)$ to denote a scalar remainder term $\epsilon_{k}(v)$ such that
there exist constants $r_0>0$ and $C$ with
$|\epsilon_{k}(v)|\le C\|v\|^k$ for all $\|v\|\le r_0$.
We use $\|\cdot\|$ for the Euclidean norm on vectors and the corresponding
induced operator norms on linear/bilinear maps, and $\langle\cdot,\cdot\rangle$
for the standard inner product on $\R^K$.
For a differentiable map $h$ between finite-dimensional vector spaces, $Dh(u)$
denotes its differential (Jacobian as a linear map) at $u$.

\begin{lemma}\label{lem:align-exp}
Under assumptions (A1)--(A2),
\[\|f(x)-f(x^+)\|^2 \;=\; v^\top G_f(x)\,v \;+\; O(\|v\|^3).\vspace{-2.2em}\]
\end{lemma}

\begin{proof}
Fix $x\in\Mc$. Define $g(u):= f(\exp_x(u)) $ for $u$ in a neighborhood
of $0\in T_x\Mc$. By (A1) and the smoothness of $\exp_x$, the map $ g$ is $C^2 $
in a neighborhood of 0. Hence, by multivariate Taylor's theorem, there exist
constants $r_0>0 $ and $C$ such that for all $ u\in T_x\Mc$ with
$\|u\|\le r_0$, \[g(u)=g(0)+Dg(0)\,u + \rho(u), \qquad \|\rho(u)\|\le C\|u\|^2.\]
By the chain rule, $Dg(0)=Df(x)\circ D(\exp_x)(0)$.
Since $D(\exp_x)(0)=\mathrm{Id}_{T_x\Mc}$ (\eg~\cite{lee2013smooth}), we have $Dg(0)=Df(x)$.
In an orthonormal basis of $T_x\Mc$, $Df(x)$ is represented by $J_f(x)$ and
$G_f(x)=J_f(x)^\top J_f(x)$.
Let $v=v(x,x^+)$ and assume $\|v\|\le r_0$. Since $x^+=\exp_x(v)$, we have
$f(x^+)-f(x)=g(v)-g(0)=J_f(x)\,v+\rho(v)$ with $\|\rho(v)\|\le C\|v\|^2$. Therefore, 
$\displaystyle \|f(x^+)-f(x)\|^2 =\|J_f(x)v+\rho(v)\|^2 = \|J_f(x)v\|^2 + 2\langle J_f(x)v,\rho(v)\rangle + \|\rho(v)\|^2$.
The cross term satisfies
\[
\big|2\langle J_f(x)v,\rho(v)\rangle\big| \le 
2\|J_f(x)v\|\,\|\rho(v)\| \le 2\|J_f(x)\|\,\|v\|\cdot C\|v\|^2 = O(\|v\|^3),
\]
and the last term satisfies $\|\rho(v)\|^2\le C^2\|v\|^4=O(\|v\|^3)$ as
$\|v\|\to 0$. Hence
$\displaystyle  \|f(x^+)-f(x)\|^2 = \|J_f(x)v\|^2 + O(\|v\|^3) = v^\top
J_f(x)^\top J_f(x)v + O(\|v\|^3) = v^\top G_f(x)\,v + O(\|v\|^3)$, which proves the lemma.
\end{proof}

This immediately gives us a rewriting of the alignment term in the localized
surrogate used in Section~\ref{subsec-contrastive}.

\begin{cor}\label{cor:align-trace}
Under (A1)--(A3),
\[
\E[\|f(x)-f(x^+)\|^2\mid x]
=
c\cdot\mathrm{tr}(G_f(x))
+
O\big(\E[\|v\|^3\mid x]\big).
\]
\end{cor}

\begin{proof}
By Lemma~\ref{lem:align-exp}, there exist $r_0>0$, $C<\infty$, and a remainder
$\epsilon(v)$ such that for all $\|v\|\le r_0$,
\[
\|f(x)-f(x^+)\|^2 = v^\top G_f(x)v + \epsilon(v),
\qquad
|\epsilon(v)| \le C\|v\|^3.
\]
Taking conditional expectation given $x$ and using linearity,
\[
\E[\|f(x)-f(x^+)\|^2\mid x]
=
\E[v^\top G_f(x)v\mid x] + \E[\epsilon(v)\mid x].
\]
Moreover,
\[
|\E[\epsilon(v)\mid x]|
\le
\E[|\epsilon(v)|\mid x]
\le
C\,\E[\|v\|^3\mid x],
\]
so $\E[\epsilon(v)\mid x]=O(\E[\|v\|^3\mid x])$.

For the quadratic form, note $v^\top G_f(x)v=\mathrm{tr}(G_f(x)vv^\top)$, hence
\[
\E[v^\top G_f(x)v\mid x]
=
\mathrm{tr}\!\left(G_f(x)\,\E[vv^\top\mid x]\right).
\]
Under (A3), $\E[vv^\top\mid x]=cI_d$, so
$\E[v^\top G_f(x)v\mid x]=c\cdot\mathrm{tr}(G_f(x))$. Combining the above yields
the claim.
\end{proof}

We will also use a variant of Lemma~\ref{lem:align-exp}  stated as follows:

\begin{cor}\label{cor:local-dist}
Under (A1), for $y$ in a sufficiently small normal neighborhood of $x$ and
$v=\exp_x^{-1}(y)$, we have
\[\|f(y)-f(x)\| = \sqrt{v^\top G_f(x)v} + O(\|v\|^2).\vspace{-5ex}\]
\end{cor}

\begin{proof}
The proof of Lemma~\ref{lem:align-exp} is purely local and uses only that the
second point lies in a normal neighborhood of $x$. Hence the same argument
applies with $x^+$ replaced by $y$: 
\[\|f(y)-f(x)\|^2 = v^\top G_f(x)v + O(\|v\|^3).\]
Write $a(v):=v^\top G_f(x)v$ and let $r(v)=O(\|v\|^3)$ denote the scalar
remainder so that $\|f(y)-f(x)\|^2=a(v)+r(v)$.
Since $G_f(x)\succ 0$ by (A1), there exists $m>0$ such that
$a(v)\ge m\|v\|^2$ for all $v$.
For sufficiently small $\|v\|$, $a(v)+r(v)\ge \tfrac{m}{2}\|v\|^2$. Hence, for
some constants $C, c'>0$, 
\[
\big|\sqrt{a(v)+r(v)}-\sqrt{a(v)}\big|=  \frac{|r(v)|}{\sqrt{a(v)+r(v)}+\sqrt{a(v)}}\le \frac{C\|v\|^3}{c'\|v\|} = O(\|v\|^2).
\]
This completes the proof of Corollary~\ref{cor:local-dist}. 
\end{proof}

\subsection{Proof of Theorem~\ref{thm-local-iso}}

We give a full proof of Theorem~\ref{thm-local-iso}. The intuition of the proof
is that, for a nearby point $y$ around $x$, the encoder admits a first-order Taylor
approximation along the unique short geodesic from $x$ to $y$. The leading
term is governed by the Jacobian $J_f(x)$, hence by the induced metric
$G_f(x)=J_f(x)^\top J_f(x)$. Local encoder optimality forces $G_f(x)$ to be a scalar
multiple of the identity, which makes local distances proportional across encoders up to a scalar.

Consider any $x\in\Mc$ and let $y\in\Mc$ satisfy $d_{\Mc}(x,y)<\delta_{\Mc}(x)$.
Let $\exp_x:T_x\Mc\to\Mc$ denote the Riemannian exponential map. Since we are
within the normal neighborhood of $x$, there exists a unique
$v\in T_x\Mc$ such that $y=\exp_x(v)$ and $\|v\|=d_{\Mc}(x,y)$
(\cite{lee2013smooth}).
Throughout, all $\Oc(\cdot)$ terms are as $y\to x$ (equivalently $\|v\|\to 0$).

\etitle{(1) local metric minimization}.
For $i\in\{1,2\}$, local optimality at $x$ means that $G_{f_i}(x)$ minimizes
\[\Phi_i(G):=c\,\mathrm{tr}(G)-\frac{\lambda_i}{2}\log\det(G) \qquad \text{over } G\succ 0.\]
Note that $-\log\det(G)$ is strictly convex, so $\Phi_i$ is strictly convex
and thus has a unique minimizer. For any symmetric direction $H$, since
$\frac{d}{dt}\mathrm{tr}(G+tH)\big|_{t=0}=\mathrm{tr}(H)$ and
$\frac{d}{dt}\log\det(G+tH)\big|_{t=0}=\mathrm{tr}(G^{-1}H)$ (Jacobi's formula;
\eg~\cite{matrixcookbook}), we have 
\[\frac{d}{dt}\Phi_i(G+tH)\Big|_{t=0} = \mathrm{tr}\!\left(\Big[cI_d-\frac{\lambda_i}{2}G^{-1}\Big]H\right).\]
At the minimizer this derivative is $0$ for all symmetric $H$.
Since $\mathrm{tr}(AH)=0$ for all symmetric $H$ implies $A=0$ (take $H=A$),
we obtain $cI_d-\frac{\lambda_i}{2}G^{-1}=0$. Therefore,
\[G_{f_i}(x)=\frac{\lambda_i}{2c}I_d, \qquad i\in\{1,2\}.\]

\etitle{(2) Local distance expansion via the induced metric}. 
Note that from Corollary~\ref{cor:local-dist},  for  $i\in \{1, 2\}$
\[
\|f_i(y)-f_i(x)\|  =  \|J_{f_i}(x)v\|+\Oc(\|v\|^2)  = \sqrt{v^\top G_{f_i}(x)v}+\Oc(\|v\|^2).
\]
Substituting $G_{f_i}(x)=\frac{\lambda_i}{2c}I_d$ (from step (1)) and  $\|v\|=d_{\Mc}(x,y)$, we have
\[\|f_i(y)-f_i(x)\| = \sqrt{\frac{\lambda_i}{2c}}\,d_{\Mc}(x,y)+\Oc(d_{\Mc}(x,y)^2).\]

\etitle{Step 3: Compare encoders.}
Let $\kappa:=\sqrt{\lambda_1/\lambda_2}$. Then 
\[\kappa\,\|f_2(y)-f_2(x)\| = \sqrt{\frac{\lambda_1}{2c}}\,d_{\Mc}(x,y)+\Oc(d_{\Mc}(x,y)^2).\]
Thus, by comparing with the $i=1$ expansion, we have
\[\|f_1(y)-f_1(x)\| = \kappa\,\|f_2(y)-f_2(x)\|+\Oc(d_{\Mc}(x,y)^2),\]
which is equivalent to the stated form
$\|f_1(x)-f_1(y)\|=\kappa\,\|f_2(x)-f_2(y)\|+\Oc(d_{\Mc}(x,y)^2)$.

\subsection{Relaxing (A3 $\to$ A3$'$): Point-Dependent Local Augmentation Assumption}
\label{app:assumptions}

Assumption (A3) in Section~\ref{sec-local} simplifies the positive-pair
distribution by $\E[vv^\top\mid x]=cI_d$ with a constant for all $x\in\Mc$.
In practice, the magnitude of a semantic-preserving augmentation may also 
depend on the anchor point (\eg some examples admit larger perturbations than
others). A natural relaxation is to allow the isotropic scale to vary with $x$.
\looseness = -1

\stitle{A relaxed local-isotropy model (A3$'$)}. 
We replace (A3) of Section~\ref{subsec-contrastive} by the following
point-dependent variant, denoted by (A3$'$): Fix $ x\in\Mc $ and let $ v=v(x,x^+)
\in T_x\Mc $ denote the geodesic displacement to a positive view. Assume
$\E[v\mid x]=0$ and  $\E[vv^\top\mid x]=c(x)\,I_d$ for some function $c(\cdot):\Mc\to(0,\infty)$.

\etitle{Effect on localized alignment term} $\varphi_{\mathrm{align}}$. 
By Lemma~\ref{lem:align-exp} and taking conditional expectation, we have
\[\E[\|f(x)\ak -\ak f(x^+)\|^2 \mid\ak  x]\ak 
=\ak  \E[v^\top G_f(x)v\mid x]\ak  +\ak  O\big(\E[\|v\|^3\mid x]\big).\]
Using $v^\top G_f(x)v\ak =\ak \mathrm{tr}(G_f(x)vv^\top)$ yields
$\displaystyle \E[\|f(x)-f(x^+)\|^2\mid x]
=\mathrm{tr}\!\big(G_f(x)\,\E[vv^\top\mid x]\big)
+ O\big(\E[\|v\|^3\mid x]\big)$.

Under (A3$'$), $\E[vv^\top\mid x]=c(x)I_d$, hence
\[\E[\|f(x)-f(x^+)\|^2\mid x] = c(x)\,\mathrm{tr}(G_f(x)) + O\big(\E[\|v\|^3\mid x]\big).\]
Therefore the leading-order localized objective becomes
\[\widetilde{\Lc}_\lambda(x;f) = c(x)\,\mathrm{tr}(G_f(x))
  -\frac{\lambda}{2}\log\det(G_f(x)). \]

\etitle{Local optimum}. 
Since the leading-order local objective depends on $f$ only through $G_{f}(x)$,
we minimize over $G\succ 0$ to characterize the optimal local metric; we then
call $f$ locally optimal at $x$ if its induced metric matches this minimizer. 
Minimizing over $G$ yields the first-order condition
\[c(x)\,I_d - \frac{\lambda}{2}G^{-1} = 0,\]
so the unique minimizer is 
$\displaystyle G_f^\star(x) = \frac{\lambda}{2c(x)}\,I_d$.

Thus the encoder remains locally a scaled isometry on $T_x\Mc$, but the scale factor depends on $x$ through $c(x)$.

\stitle{Implication on Theorem~\ref{thm-local-iso}}. 
If two encoders $f_1,f_2$ satisfy the same analysis but potentially with
different augmentation scales $c_1(x),c_2(x)$ and parameters $\lambda_1,
\lambda_2$, then the local distance expansions become (for $i\in\{1,2\}$): 
\[\|f_i(x)-f_i(y)\| = \sqrt{\frac{\lambda_i}{2c_i(x)}}\,d_{\Mc}(x,y) +
  O\big(d_{\Mc}(x,y)^2\big)\]
for $y$ in the normal neighborhood of $x$. Eliminating $d_{\Mc}(x,y)$
gives a point-dependent scaling relation
\[\|f_1(x)-f_1(y)\| = \kappa(x)\,\|f_2(x)-f_2(y)\|  +
  O\big(d_{\Mc}(x,y)^2\big)\]
with $\displaystyle
\kappa(x)=\sqrt{\frac{\lambda_1\,c_2(x)}{\lambda_2\,c_1(x)}}$. 
In particular, if the two encoders share the same local positive-pair
distribution in the sense that $c_1(x)=c_2(x)$ for all $x$, then $\kappa(x)$ reduces to the constant $\sqrt{\lambda_1/\lambda_2}$ as stated in
Theorem~\ref{thm-local-iso}.

\etitle{Remark}.
Allowing $c=c(x)$ provides a simple mechanism for why a {\em single global}
scale does not fit all points: even when encoders are locally conformal, they
may ``expand'' or ``contract'' neighborhoods by different amounts at different
anchors. This complements the empirical observation that short-range distances
are substantially more consistent than long-range distances, while also
explaining residual variability within the short-range regime.

\subsection{Correlation Analysis}
\label{app:correlation}
We report the Pearson correlation coefficient $\rho$ between pairwise Euclidean distances measured in the reference embedding space and the corresponding distances between the same item pairs in the target space.
\revise{Besides the contrastive encoder pairs,
we also include classic word embedding models \glove~\cite{pennington2014glove} and \fasttext~\cite{bojanowski2016enriching}
trained without a contrastive objective as non-contrastive baselines.
The same two models are reused as the non-contrastive baseline
in the top-$k$ Jaccard analysis of \S\ref{app:topk}.}

\revise{We probe the local\-/consistency / global\-/decorrelation pattern of
Fig.~\ref{fig-correlation} along four axes:
\textit{(i)} encoder pair: six contrastive encoder pairs
(Fig.~\ref{fig:corr}, panels~a--f);
\textit{(ii)} target dimensionality: with \mistral{} as the reference,
varying the target embedding dimensionality of \openai{} on \scifact{}
(Fig.~\ref{fig:corr}, panel~g);
\textit{(iii)} task domain: \mistral{}$\rightarrow$\openai{} on two
clustering benchmarks (Fig.~\ref{fig:corr}, panel~h);
\textit{(iv)} encoder family: contrastive vs.\ non-contrastive on four
retrieval datasets (Fig.~\ref{fig:corr}, panels~i--l), where each panel
overlays \glove{}$\leftrightarrow$\fasttext{} with
\mistral{}$\leftrightarrow$\openai.
For (i)--(iii), $\rho$ is high at short range and decays as the reference
distance grows, indicating that the geometric signal underlying GEH is a
property of the contrastive encoder family. In contrast, the non-contrastive
pair in (iv) already starts at markedly lower $\rho$ at short range, and its
long-range tail does not always decay. This supports our restriction
of GEH's analysis to contrastive encoders (cf.\ \S\ref{subsec-evidence}).}

\begin{figure*}
  \centering

  \begin{subfigure}[t]{0.32\textwidth}
    \centering
    \includegraphics[width=\linewidth]{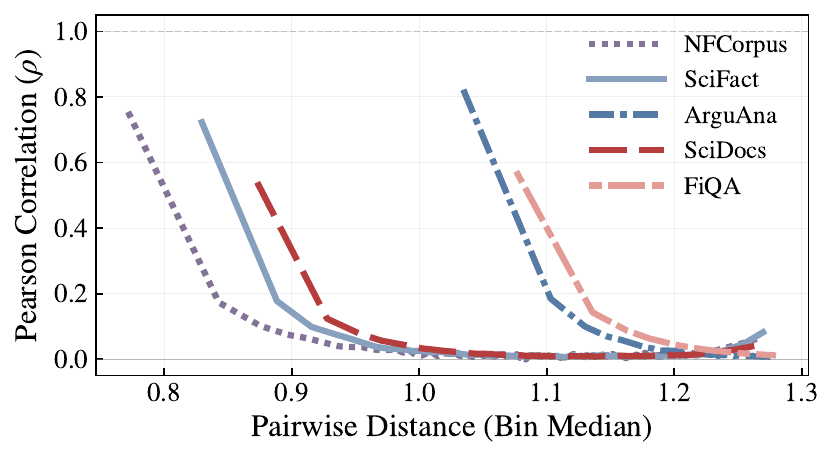}
    \caption{\gte $\rightarrow$ \mistral}
    \label{fig:corr-gte-mistral}
  \end{subfigure}\hfill
  \begin{subfigure}[t]{0.32\textwidth}
    \centering
    \includegraphics[width=\linewidth]{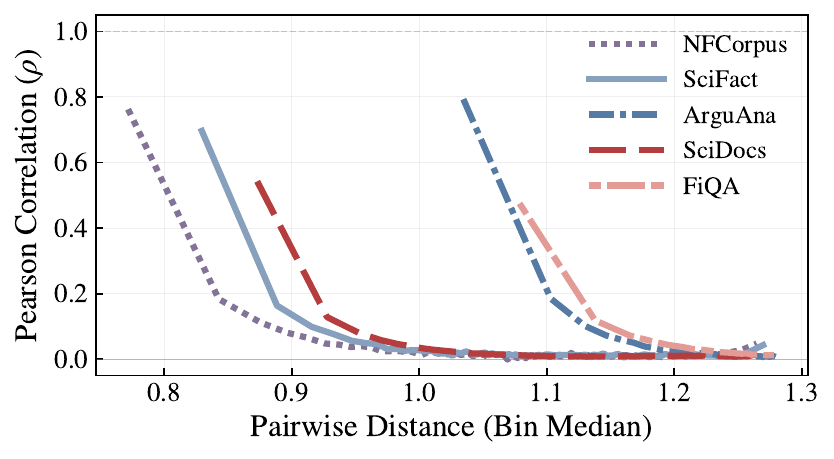}
    \caption{\gte $\rightarrow$ \openai}
    \label{fig:corr-gte-openai}
  \end{subfigure}\hfill
  \begin{subfigure}[t]{0.32\textwidth}
    \centering
    \includegraphics[width=\linewidth]{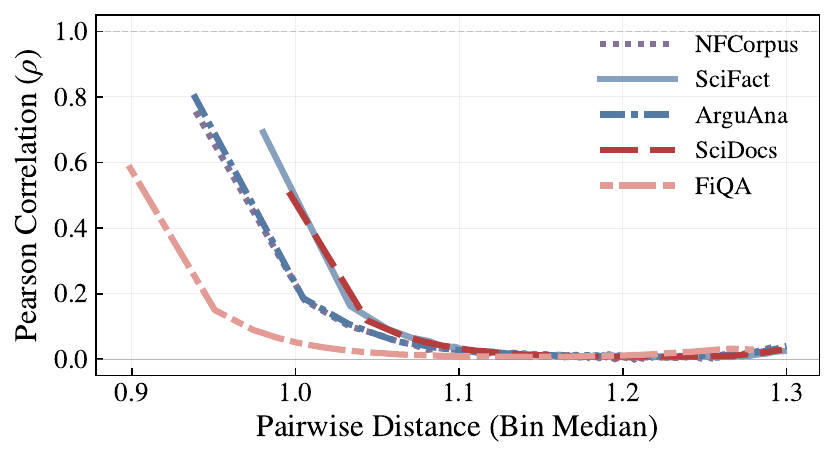}
    \caption{\kalm $\rightarrow$ \mistral}
    \label{fig:corr-kalm-mistral}
  \end{subfigure}

  \begin{subfigure}[t]{0.32\textwidth}
    \centering
    \includegraphics[width=\linewidth]{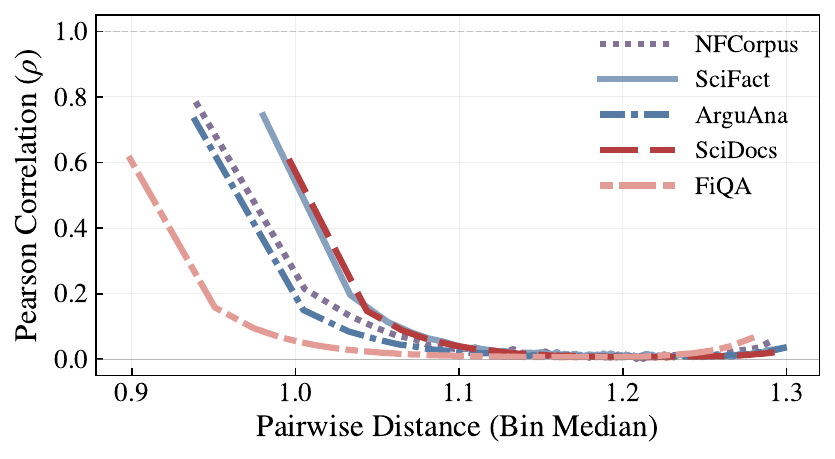}
    \caption{\kalm $\rightarrow$ \qwen}
    \label{fig:corr-qwen-kalm}
  \end{subfigure}\hfill
  \begin{subfigure}[t]{0.32\textwidth}
    \centering
    \includegraphics[width=\linewidth]{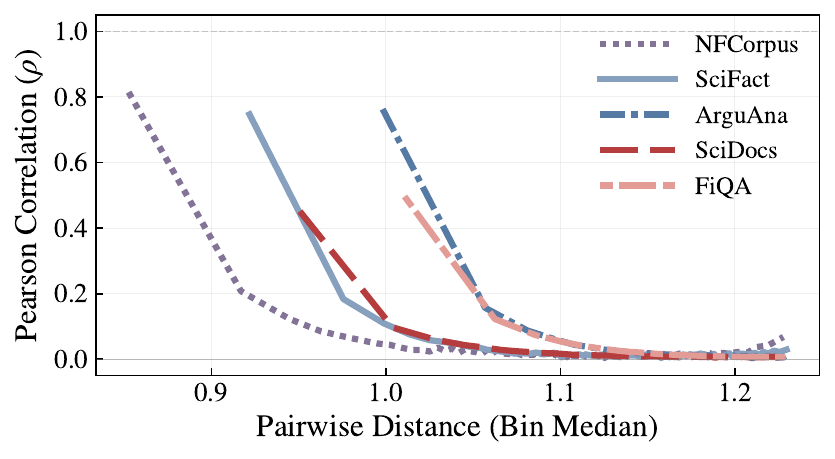}
    \caption{\qwen $\rightarrow$ \openai}
    \label{fig:corr-qwen-openai}
  \end{subfigure}\hfill
  \begin{subfigure}[t]{0.32\textwidth}
    \centering
    \includegraphics[width=\linewidth]{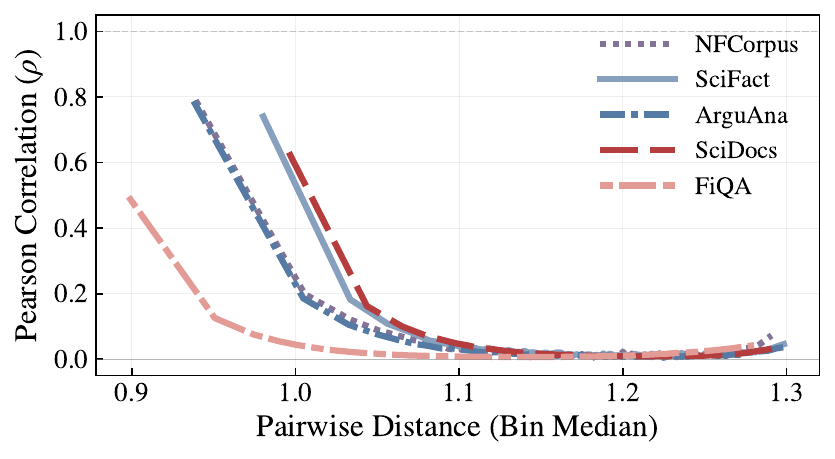}
    \caption{\kalm $\rightarrow$ \openai}
    \label{fig:corr-kalm-openai}
  \end{subfigure}

  \begin{subfigure}[t]{0.32\textwidth}
    \centering
    \includegraphics[width=\linewidth]{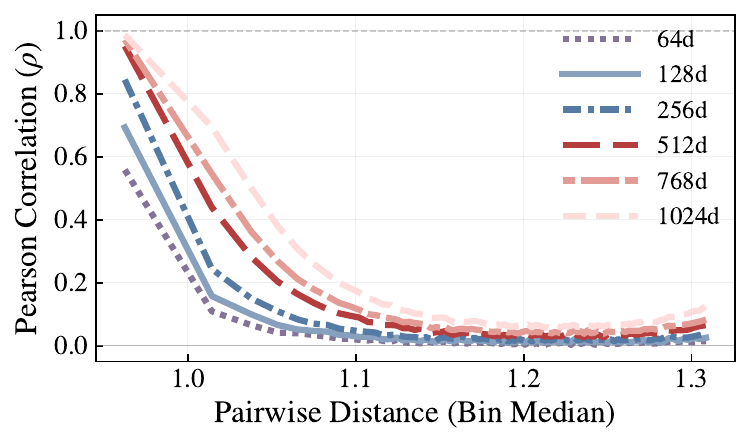}
    \caption{\revise{\openai{} dim sweep, \scifact}}
    \label{fig:corr-dims}
  \end{subfigure}\hfill
  \begin{subfigure}[t]{0.32\textwidth}
    \centering
    \includegraphics[width=\linewidth]{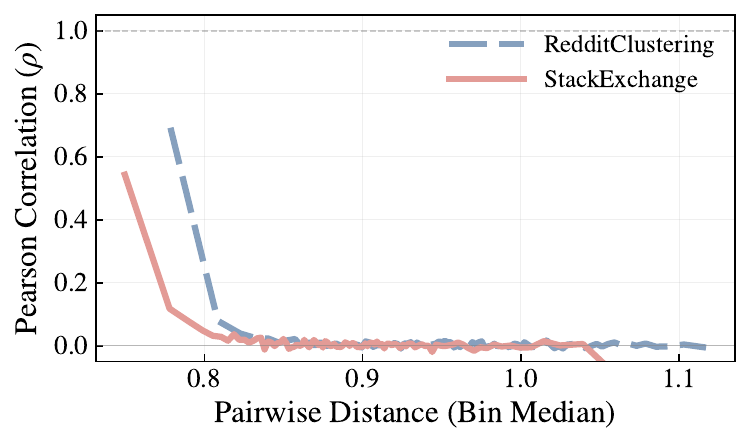}
    \caption{\revise{Clustering tasks}}
    \label{fig:corr-clustering}
  \end{subfigure}\hfill
  \begin{subfigure}[t]{0.32\textwidth}
    \centering
    \includegraphics[width=\linewidth]{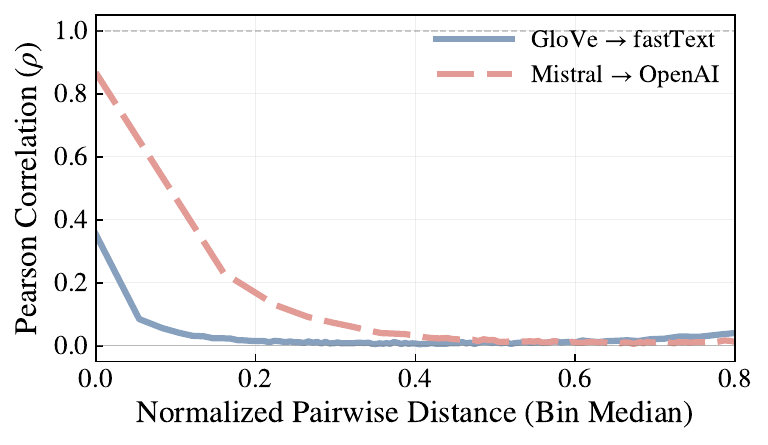}
    \caption{\revise{\arguana}}
    \label{fig:corr-nc-arguana}
  \end{subfigure}

  \begin{subfigure}[t]{0.32\textwidth}
    \centering
    \includegraphics[width=\linewidth]{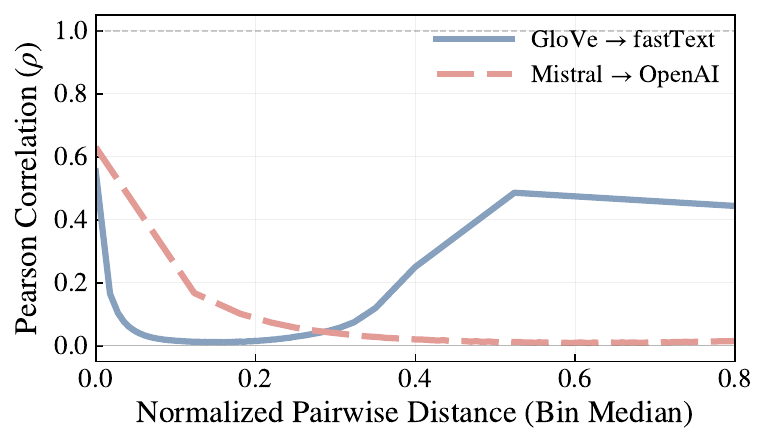}
    \caption{\revise{\fiqa}}
    \label{fig:corr-nc-fiqa}
  \end{subfigure}\hfill
  \begin{subfigure}[t]{0.32\textwidth}
    \centering
    \includegraphics[width=\linewidth]{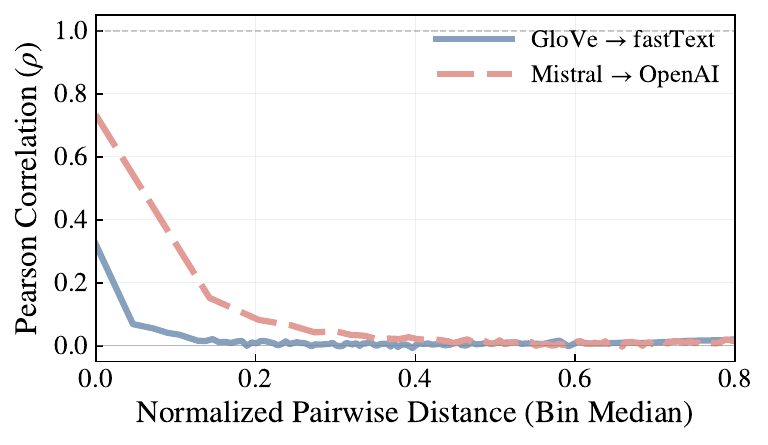}
    \caption{\revise{\nfcorpus}}
    \label{fig:corr-nc-nfcorpus}
  \end{subfigure}\hfill
  \begin{subfigure}[t]{0.32\textwidth}
    \centering
    \includegraphics[width=\linewidth]{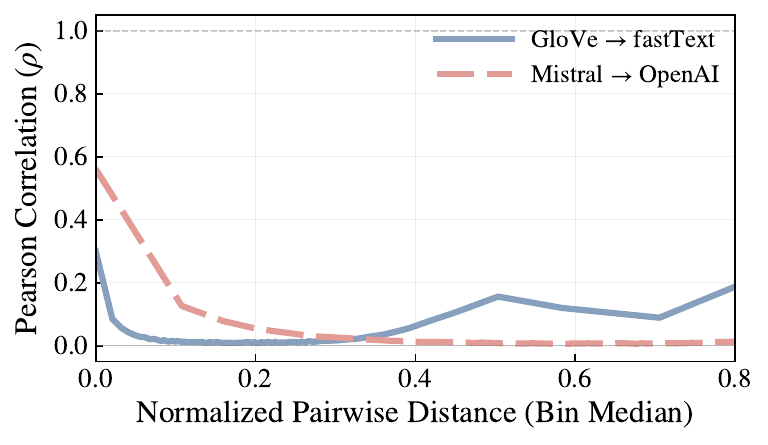}
    \caption{\revise{\scidocs}}
    \label{fig:corr-nc-scidocs}
  \end{subfigure}

  \caption{Distance consistency across embedding spaces.
  {\normalfont\footnotesize
  Each subplot shows Pearson correlation $\rho$ between pairwise distances in the reference space and their counterparts in the target space, binned by the reference distance.
  (a--f) Six contrastive encoder pairs.
  \revise{(g) \mistral{}$\rightarrow$\openai{} on \scifact{} for a sweep of \openai{} dimensionalities.
  (h) \mistral{}$\rightarrow$\openai{} on two clustering benchmarks.
  (i--l) Non-contrastive comparison: each panel overlays \glove{}$\leftrightarrow$\fasttext{} with \mistral{}$\leftrightarrow$\openai{} on the same dataset.}}}
  \label{fig:corr}
\end{figure*}

\subsection{Retrieval Result analysis }
\label{app:topk}
To empirically validate the local geometric consistency of embedding spaces, we analyze the consistency of top-$k$ retrieval results across different embedding models. Given two embedding sets $\Ec_1$ and $\Ec_2$ encoding the same corpus, we perform top-$k$ nearest neighbor retrieval for each query point in both spaces and measure their agreement using the Jaccard index: $J_k = |\mathcal{N}_k^{(1)} \cap \mathcal{N}_k^{(2)}| / |\mathcal{N}_k^{(1)} \cup \mathcal{N}_k^{(2)}|$, where $\mathcal{N}_k^{(i)}$ denotes the set of $k$ nearest neighbors in embedding space $i$. We evaluate this metric across multiple embedding model pairs.

\revise{As shown in Figure~\ref{fig:jaccard}, the behavior splits sharply by encoder family.
For pairs of contrastive encoders (panels a--c), the Jaccard index starts at $J\!\approx\!0.7$--$0.8$ at $k{=}1$ and decays monotonically to a plateau around $0.37$--$0.45$ at $k{=}50$, empirically confirming Theorem~\ref{thm-local-iso}.
In contrast, when at least one side is a non-contrastive encoder (panels d--f), the Jaccard index never enters this short-range high-consistency regime: \glove $\leftrightarrow$ \fasttext (panel d) stays near $J\!\approx\!0.1$--$0.17$ over the entire range of $k$, and a \mistral / \openai query side that yields $J\!\approx\!0.8$ against a contrastive target collapses to $J\!\approx\!0.2$ against a non-contrastive target (panels e--f).}

\begin{figure*}
  \centering

  \begin{subfigure}[t]{0.32\textwidth}
    \centering
    \includegraphics[width=\linewidth]{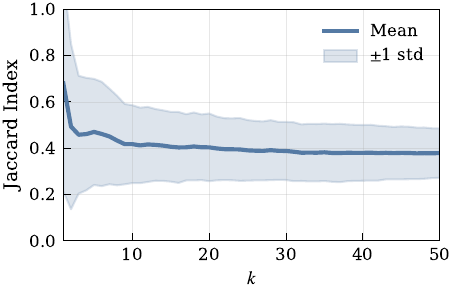}
    \caption{\kalm $\leftrightarrow$ \gte}
    \label{fig:jaccard-kalm-gte}
  \end{subfigure}\hfill
  \begin{subfigure}[t]{0.32\textwidth}
    \centering
    \includegraphics[width=\linewidth]{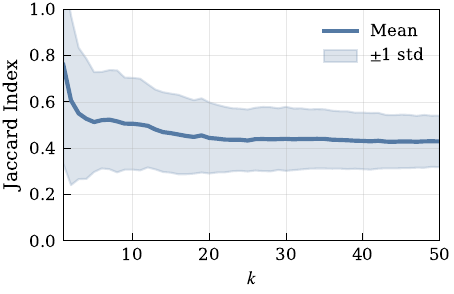}
    \caption{\qwen $\leftrightarrow$  \gte}
    \label{fig:jaccard-qwen-gte}
  \end{subfigure}\hfill
  \begin{subfigure}[t]{0.32\textwidth}
    \centering
    \includegraphics[width=\linewidth]{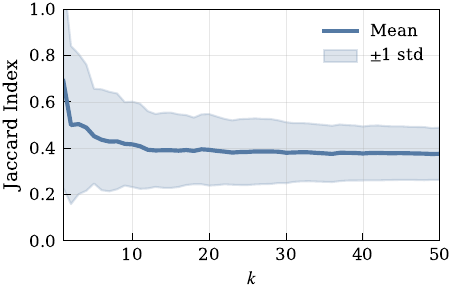}
    \caption{\qwen $\leftrightarrow$ \kalm}
    \label{fig:jaccard-qwen-kalm}
  \end{subfigure}

  \begin{subfigure}[t]{0.32\textwidth}
    \centering
    \includegraphics[width=\linewidth]{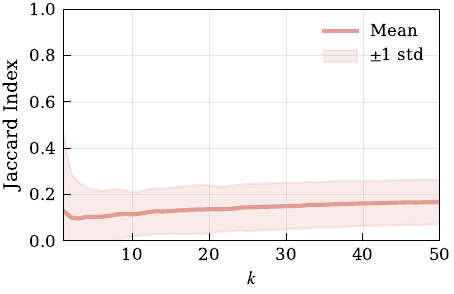}
    \caption{\revise{\glove $\leftrightarrow$  \fasttext}}
    \label{fig:jaccard-glove-fasttext}
  \end{subfigure}\hfill
  \begin{subfigure}[t]{0.32\textwidth}
    \centering
    \includegraphics[width=\linewidth]{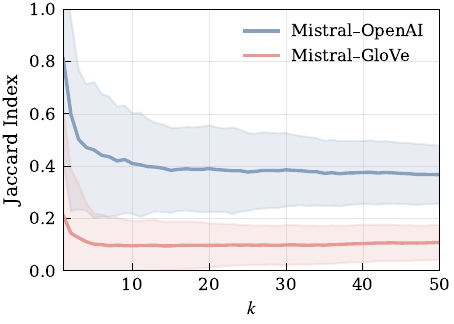}
    \caption{\revise{\mistral $\leftrightarrow$ \openai and \glove}}
    \label{fig:jaccard-mistral-glove-openai}
  \end{subfigure}\hfill
  \begin{subfigure}[t]{0.32\textwidth}
    \centering
    \includegraphics[width=\linewidth]{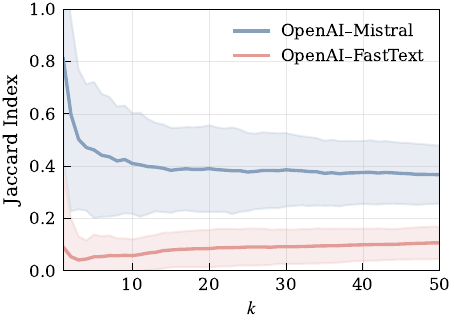}
    \caption{\revise{\openai $\leftrightarrow$ \mistral and \fasttext}}
    \label{fig:jaccard-openai-mistral-fasttext}
  \end{subfigure}

  \caption{Cross-embedding retrieval consistency analysis:{\normalfont\footnotesize \scifact~,  Each panel reports the mean $\pm$ 1 std of the per-query Jaccard index between top-$k$ retrieval results from two embedding spaces over 100 random queries.}}
  \label{fig:jaccard}
\end{figure*}

\section{Details of Section~\ref{sec-linking}}
\label{app:sec4-details}

\subsection{Why FPS for View Sampling (Stability Analysis)}
\label{app:well-conditioned-views}

We justify FPS view sampling for the kernelized distance-to-anchor hash used in the main text.  In a nutshell, we show
that a hash is stable if the anchor-induced distance coordinates
provide diverse directional information, and clustered anchors yield redundant
coordinates and amplify cross-model distance distortions.

\stitle{Local stability via the tangent jacobian} 
Since we employ cosine distance, we operate on the unit hypersphere $\mathbb{S}^{D-1} \subset \mathbb{R}^D$. 
The hashing function $\mathbf{h}_{\mathcal{A}}: \mathbb{S}^{D-1} \to \mathbb{R}^{m_t}$ can be defined by the kernel with cosine distance:
\[
h_{\mathcal{A}}(w)_j = \exp\left(-\frac{1 - \langle w, a'_j \rangle}{\sigma}\right), \quad j=1,\dots,m_t.
\]
We differentiate $h_{\mathcal{A}}(w)_j$ with respect to $w$ and restrict the domain to the tangent space of the sphere, $T_v\mathbb{S}^{D-1} = \{z \in \mathbb{R}^D : \langle z, v \rangle = 0\}$.
The $j$-th row of the Jacobian $J_{\mathcal{A}}(v) \in \mathbb{R}^{m_t \times D}$ is given by the projection of the gradient onto $T_v$:
\[
(J_{\mathcal{A}}(v))_{j,:} = \frac{1}{\sigma} h_j(v) \cdot \underbrace{(a'_j - \langle a'_j, v \rangle v)^\top}_{\mathbf{p}_j^\top}.
\]
Here, $\mathbf{p}_j$ represents the component of the anchor $a'_j$ orthogonal to the query $v$. 

We assume the data lies on a submanifold $\mathcal{M} \subset \mathbb{S}^{D-1}$ of intrinsic dimension $d \ll D$. For the hash to be stable (locally injective) on $\mathcal{M}$, the mapping must distinguish perturbations in any tangent direction.Let $\mathbf{\Pi}_v \in \mathbb{R}^{D \times d}$ be an orthonormal basis for the tangent space $T_v\mathcal{M}$. The Restricted Jacobian $\mathbf{J}_{\mathcal{M}}(v) \in \mathbb{R}^{m_t \times d}$ is defined as:$$\mathbf{J}_{\mathcal{M}}(v) = J_{\mathcal{A}}(v) \cdot \mathbf{\Pi}_v.$$A necessary condition for stability is that $\mathbf{J}_{\mathcal{M}}(v)$ has full column rank, which requires $m_t \ge d$.The robustness of this stability is quantified by the condition number $\kappa = \sigma_{\max} / \sigma_{\min}$ of $\mathbf{J}_{\mathcal{M}}(v)$. If $\kappa$ is large ($\sigma_{\min} \approx 0$), the hash is insensitive to changes along the corresponding singular vector, leading to ambiguity.

\stitle{Geometric Optimality of FPS} The singular values of $\mathbf{J}_{\mathcal{M}}(v)$ are determined by the geometric arrangement of the anchor projections $\{\mathbf{p}_j\}$. Random sampling can select multiple anchors that are locally redundant, making the manifold-tangent components $\{\Pi_v^\top p_j\}$ nearly collinear and leaving some tangent directions weakly sensed. This drives $\sigma_{\min} \to 0$. FPS greedily maximizes minimum pairwise distance, discouraging near-duplicates in each view and promotes larger $\sigma_{\min}$. This  reduce angular redundancy among $\{\Pi_v^\top p_j\}$ and thus improve the conditioning of $J_{\mathcal M}(v)$ in practice. 

Additionally, FPS acts as a covering strategy (a greedy $k$-center heuristic): selected anchors
are spread across the current anchor pool, increasing the chance that a sampled
view contains anchors that are {\em locally relevant} for many different query
points. 

FPS effectively reduce view-specific collisions and make per-view hash-space
matching more stable, which in turn improves the quality of votes aggregated by the bootstrapping procedure in Section~\ref{sec-linking}. Consistent with this
analysis, FPS significantly outperforms random-based anchor selection
in our ablations (Section~\ref{sec:ablation}).

\subsection{Properties of View Scheduling}
\label{app:view-scheduling}

This appendix expands on the scheduling rule used in Section~\ref{sec-linking}
for choosing the number of views $m_t$ and the anchors per view $s_t$ as the
paired-anchor pool $\Lc_{t-1}$ grows. The schedule provides two practical
properties/guarantees: (i) increasing view diversity over iterations, and (ii) stable
per-anchor coverage and computation.

\stitle{Objectives.}
At iteration $t$, let $n_t := |\Lc_{t-1}|$ denote the size of the current
paired-anchor pool. A view contains $s_t$ paired anchors, and we sample $m_t$
views. Define the {\em per-view anchor fraction}
$\rho_t := s_t / n_t$.

The view schedule  decides $m_{t}$ and $\rho_{t}$ (equivalently, $s_{t}$), while
aiming to satisfy two practical objectives:

\eetitle{(O1) Increasing diversity \& locality.}
As $n_t$ grows, we would like to sample more views (so more chances to hit
short-range neighborhoods) while making each view a smaller fraction of the pool
$\Lc_{t-1}$ (so views are more ``local'' and less dominated by far anchors).

\eetitle{(O2) Stable per-anchor coverage and computation.}
If views were sampled uniformly, a given anchor would appear in an expected
$m_t\rho_t$ views per iteration. We would like this quantity to remain roughly
stable over iterations (so anchors are neither over-used nor ignored), and we
would like the total anchor usage $m_t s_t$ per iteration to scale reasonably.

\stitle{Anchor pool growth ratio $g_t$.}
We measure progress by growth ratio $g_t := \frac{|\Lc_{t-1}|}{\max\{|\Sc|,1\}}
\;\ge\; 1$,  where $\Sc$ is the (tiny) initial seed set. This ratio ensures that the
schedule depends on relative growth of the anchor pool rather than its
absolute size, \eg increasing from 10 to 100 anchors and from 100 to
1000 anchors both correspond to $g_t=10$.

\stitle{Coverage-preserving parameterization.}
To satisfy (O2), we parameterize the schedule via a single scaling function
$ \mathrm{sf}(g)\ge 1$ that is set to satisfy the following:
\[m_t \;\approx\; m_0\,\mathrm{sf}(g_t), \qquad \rho_t \;\approx\; \rho_0/\mathrm{sf}(g_t),\]
where $m_0\in\mathbb{N}$ and $\rho_0\in(0,1]$ are base parameters.
Using this view scheduling would assure approximate invariance (ignoring rounding)
$m_t\rho_t \approx m_0\rho_0$, \ie\ {\em a constant expected per\-/anchor participation rate} under uniform sampling.

This would also imply linear scaling of total anchor usage since 
$m_t s_t = m_t\rho_t\,n_t \approx m_0\rho_0\,|\Lc_{t-1}|$.
Thus it increases the number of views while controlling per-iteration work.

\stitle{Determining $\mathrm{sf}(g)$}.
For convenience, we consider $\mathrm{sf}(g)=1+u(g)$, where we use $u(g)$ to
represent the increment induced by the growth of anchor pool. Thus we require
$u(1) = 0$ as $g=1$ corresponds to the initial state $|\Lc_{t-1}|=|\Sc|$ with no
growth occurred.

To decide $u(g)$, we consider a simple principle: increment in schedule should
depend only on the expansion rate of the anchors not the current scale $g$. 
Let $g$ be the current growth ratio \wrt $\Lc_{r-1}$  and let $g'$ denote
the growth ratio after expansion to $\Lc_{r}$.
Define the  expansion ratio $\lambda:=g'/g$. 
Then this principle says, if the anchor pool grows by a factor
$\lambda\ge 1$, we want the increase in schedule (size and number of views) to depend only on $\lambda$
and not on the current growth scale $g$. In terms of $u(g)=\mathrm{sf}(g)-1$, this
boils down to $u(\lambda g)-u(g) = u(\lambda)$ for $\forall g,\lambda\ge 1$;
equivalently, 
\[u(\lambda g)=u(\lambda)+u(g)\qquad \forall g,\lambda\ge 1.\]
By the classical Cauchy additive functional equation (after a log change of
variables), assuming $u$ is continuous, the unique solutions to this are of form $u(g) = c \log g$  for
some constant $c\ge 0$.

Thus we adopt the logarithmic scale factor
\[\mathrm{sf}_t := \mathrm{sf}(g_t) = 1 + c\log g_t.\]
This  grows sublinearly and avoids exploding $T_t$ as the anchor pool
expands, while still increasing view diversity. We empirically verify the robustness of view-scheduling hyperparameters $m_0$ and $\rho_0$ (see Appendix~\ref{app:sensitivity}).

\revise{
\subsection{Per-view Link Proposal at Scale}
\label{app:per-view-impl}

When $\max(|\Ec_1|,|\Ec_2|)\le \tau$ ($\tau = 5\!\times\!10^{5}$), GEH runs the
global path of \S~\ref{sec-linking}.
Otherwise, at iteration $t$, we replace FPS with a $k$-means partition of
$\Lc_{t-1}$ in $E_1$ to ensure diversity of views. 
The number of partitions is
$m_t = \lceil|\Lc_{t-1}|/d_{\max}\rceil$, 
where $d_{\max} = \max(d_{E_1}, d_{E_2})$, which ensures each
view has sufficient anchors for a
dimensionally well-posed local neighborhood.
Each anchor in $\Lc_{t-1}$ is included in its
$\rho$ nearest clusters ($\rho{=}2$ in our experiments) which matches with the
constant per-anchor participation rate in \S~\ref{app:view-scheduling}.

Within each view $\Ac_{t,k}$, each paired anchor $(a, a') \in \Ac_{t,k}$ contributes its own local                                                                                                           
neighborhood: let                                                                                                                                                              
$\mathrm{NN}_{k_{\mathrm{NN}}}(a, E_1)$ denote the                                                                                                                             
$k_{\mathrm{NN}}$ nearest ambient neighbors of $a$ in $E_1$, and                                                                                                               
analogously $\mathrm{NN}_{k_{\mathrm{NN}}}(a', E_2)$. The view's local                                                                                                         
sets are the unions $S_{\Ac_{t,k}}^{(1)} \;:=\; \bigcup_{(a,a')\in\Ac_{t,k}} \mathrm{NN}_{k_{\mathrm{NN}}}(a, E_1)$
and analogously $S_{\Ac_{t,k}}^{(2)}$.
We then build the distance-to-anchor signatures of \S\ref{sec-hash} only for
points in $S_{\Ac_{t,k}}^{(1)}\cup S_{\Ac_{t,k}}^{(2)}$, compute
$\kw{csls}_{\Ac_{t,k}}$ within
$S_{\Ac_{t,k}}^{(1)}\!\times\!S_{\Ac_{t,k}}^{(2)}$, and emit MNN
proposals $\Pc_{t,k}$ within that local bipartite set.\looseness = -1

The local sets are identified via two $k$-NN indices over
$E_1$ and $E_2$ separately, built once at the start of \mvh (cost
$O(|E_1|+|E_2|)$) and reused across all
$T$ iterations and $\sum_t m_t$ views. Each per-view local-set lookup
is then a single $k_{\mathrm{NN}}$-NN query against this static
index.}

\section{Details of Section~\ref{sec-exp}}
\label{app:exp}

\subsection{Experimental Setup}\label{app:datasets}

\subsubsection{Datasets}
\label{app:beir}
We conduct main experiments on five benchmark datasets from BEIR~\cite{thakur2021beir}, covering biomedical retrieval, financial analysis, citation prediction, argument mining, and fact verification. Below we briefly describe each dataset.

\scifact~\cite{scifact} is a scientific fact-checking benchmark. Queries consist of short scientific claims, while documents are abstracts of scientific papers. The goal is to retrieve supporting or refuting evidence for each claim.
 
\nfcorpus~\cite{nfcorpus} focuses on health-related information retrieval. Queries come from user-generated content such as blog posts, Q\&A threads, and video transcripts, and the corpus is built from medical articles in PubMed.

\arguana~\cite{arguana} addresses argument retrieval. Given an argument as a query, the task is to identify the most relevant counterarguments from a collection of argument pairs mined from online debate portals.

\scidocs~\cite{scidocs} is a citation prediction dataset derived from scientific publications. Queries are scientific papers, and the task is to retrieve related works among a large held-out collection.

\fiqa~\cite{fiqa} comes from the financial domain. The queries are investment-related questions posted on StackExchange, while the corpus contains financial articles and answers from the same platform.

\fever~\cite{fever} is a fact-verification dataset. The queries are short factual claims, and the corpus introductory sections of Wikipedia pages.

Table~\ref{tab:beir_datasets} provides dataset statistics, including query counts, corpus sizes, and the average number of relevant documents per query.

\begin{table*}[t]
\caption{BEIR dataset statistics: \normalfont\footnotesize{\#test queries ($Q$), test corpus size ($|C|$), and avg.\ relevant docs/query ($R$).}}
\label{tab:beir_datasets}
\centering
\begin{tabular}{lrrr}
\toprule
\textbf{Dataset} & $Q$ & $|C|$ & $R$ \\
\midrule
\scifact   & 300   & 5{,}183   & 1.1 \\
\nfcorpus & 323 & 3{,}633  & 38.2 \\
\arguana   & 1{,}406 & 8{,}674  & 1.0 \\
\scidocs   & 1{,}000 & 25{,}657 & 4.9 \\
\fiqa & 648   & 57{,}638  & 2.6 \\
\fever &6{,}666 &  5{,}416{,}568  & 1.2 \\
\bottomrule
\end{tabular}
\end{table*}

\subsubsection{Embedding Models}\label{app:models}

We generate embeddings using a mix of proprietary API services and open-weight embedding models. 

\stitle{\mistral}  We use Mistral's commercial Embeddings API with the mistral-embed model\cite{mistral}. Mistral's model family is        
  designed for efficient inference; we use the hosted embeddings endpoint as provided by Mistral.   
  
\stitle{\openai} We use OpenAI's Embeddings API with text-embedding-3-small\cite{openai}. We embed queries and documents using the same    
  embedding model and endpoint (i.e., no separate query/document encoders or prompt format is required by the API).                          

\stitle{\gte} We use gte-Qwen2-7B-instruct\cite{li2023towards}, a 7B-parameter text embedding model in the General Text Embeddings (GTE)   
  family, trained on top of the \texttt{Qwen2-7B} backbone. We follow the recommended usage: queries are encoded with a query-specific       
  prompt, whereas documents are encoded without instructions.                 

\stitle{\qwen} We use Qwen3-Embedding-8B~\cite{qwen3embedding}, an instruction-aware embedding model built on the Qwen3 foundation models. 
   We follow the recommended asymmetric encoding: queries are encoded with a query-specific prompt, while documents are encoded unchanged.   

\stitle{\kalm}   We use KaLM-Embedding-Gemma3-12B~\cite{zhao2025kalmembeddingv2superiortrainingtechniques}, a 12B-parameter embedding model
   from Tencent built on the Gemma3 foundation. The model uses symmetric encoding for both queries and documents, with L2-normalized output 
  embeddings.  

Table~\ref{tab:main_embed_dims} summarizes the embedding dimensionality of the models used in
our main experiments.

\begin{table}[t]
\centering
\caption{Embedding output dimensionality for the models used in our main experiments.\normalfont\footnotesize{Abbr.\ denotes the model shorthand used throughout the paper.}}
\label{tab:main_embed_dims}
\small
\begin{tabular}{lcr}
\toprule
Model & Abbr. & Dim \\
\midrule
KaLM-Gemma3-12B-2511 & \kalm & 3840 \\
Qwen3-Embedding-8B & \qwen & 4096 \\
text-embedding-3-small & \openai & 1536 \\
mistral-embed & \mistral & 1024 \\
gte-Qwen2-7B-instruct & \gte & 3584 \\
\bottomrule
\end{tabular}
\end{table}

\subsubsection{Baseline Methods}\label{app:baselines}

For all alignment baselines (\linear, \cca, \mlp, \pr, \rcsls), we align $\Ec_1$ with $\Ec_2$ space and infer links using CSLS cosine similarity-based mutual nearest neighbors.

\stitle{Linear transformation (\linear)}\label{app:linear}
We performs alignment by learning a linear map from $\Ec_1$to $\Ec_2$. Given seed \revise{links} $\{(a_i,b_i)\}_{i\in\mathcal{S}}$, where $a_i\in\mathbb{R}^{d_s}$ and $b_i\in\mathbb{R}^{d_t}$ are embeddings of the same item in the source and target spaces, respectively, we learn a bias-free linear transformation $W\in\mathbb{R}^{d_t\times d_s}$ by minimizing mean squared error $\mathcal{L}(W)=\frac{1}{|\mathcal{S}|}\sum_{i\in\mathcal{S}}\|Wa_i-b_i\|_2^2$. We optimize $W$ with Adam (learning rate $10^{-3}$) for 100 epochs.

\stitle{Canonical Correlation Analysis (\cca).} We standardize each space independently, fit CCA on the seed pairs to learn one linear projection per space that maximizes correlation between projected seed embeddings. 

\stitle{Multi-Layer Perceptron (\mlp).}
We train a single-hidden-layer MLP mapping from the source embedding dimension to the target embedding dimension, with hidden width 512 and ReLU activations. We train on seed pairs using a cosine loss, optimization uses Adam with learning rate $10^{-2}$ and weight decay $10^{-5}$ for 100 training epochs.

\stitle{Procrustes (\pr).}
We align the two embedding spaces by solving the orthogonal Procrustes problem on the seed links. Let $\mathbf{X},\mathbf{Y}\in\mathbb{R}^{n\times d}$ denote the corresponding seed embeddings (rows are paired items). We estimate an orthogonal map
$\mathbf{R}^\star=\arg\min_{\mathbf{R}\in\mathbb{R}^{d\times d}}\ \|\mathbf{X}\mathbf{R}-\mathbf{Y}\|_F^2
\quad \text{s.t.}\quad \mathbf{R}^\top\mathbf{R}=\mathbf{I}$.
Let $\mathbf{A}=\mathbf{X}^\top \mathbf{Y}$ and compute its SVD $\mathbf{A}=\mathbf{U}\boldsymbol{\Sigma}\mathbf{V}^\top$. A closed-form optimum is given by
$\mathbf{R}^\star=\mathbf{U}\mathbf{V}^\top$. At inference time, we align embedding $\mathbf{x}$ via $\mathbf{x}\mathbf{R}^\star$.

\stitle{Relaxed Cross-domain Similarity Local Scaling (\rcsls)} We implement RCSLS~\cite{joulin2018loss}, which directly optimizes a linear map to improve CSLS-based retrieval and mitigate hubness.
We initialize with the Procrustes solution and optimize the RCSLS objective on the seed pairs using gradient-based optimization, projecting back to the orthogonal group after each update. 

\revise{\stitle{Unbalanced Gromov-Wasserstein (\ugw).}
We implement Unbalanced Gromov-Wasserstein~\cite{ugw} with POT's log-domain Sinkhorn solver on intra-view cosine distance matrices (normalized by their mean), warm-started by a seed-biased coupling on the supervised pairs. Links are read from the resulting transport plan via mutual argmax.
}

\stitle{Bootstrapping parallel anchors (\ao).}
We implement AO~\cite{ao} to discover links via relative representations and Sinkhorn OT. We
$\ell_2$-normalize embeddings and follow the original optimization schedule (250 steps; one Sinkhorn iteration per step).We set the anchor budget to the true overlap size, $K=\alpha|\Dc|$ (i.e., AO is given the overlap cardinality). All other baselines and \mvh does not assume knowledge of $\alpha$.

\subsubsection{Computational Resources}                 
\label{app:compute}
All experiments were run on a Kubernetes cluster. Each run was allocated a single
compute node with an AMD EPYC 7713P (64 cores), 896\,GB RAM, and one NVIDIA A100
GPU (80\,GB), running Ubuntu 22.04.5 LTS.           

\subsection{Implementation Details and Extended Analysis of \mvh}
\label{app:stopping}

\subsubsection{Stopping Criterion}
We monitor the \emph{mutual-NN ratio}, defined as the fraction of points that participate in at least one mutual nearest-neighbor (MNN) pair in an iteration. Let $\mathcal{P}_t$ be the set of MNN pairs returned at iteration $t$, and let
\[
U_t := \{u\in \Ec_1:\exists v,(u,v)\in\mathcal{P}_t\}\ \cup\ \{v\in \Ec_2:\exists u,(u,v)\in\mathcal{P}_t\}
\]
where $M_t := |U_t|, N := \max\{|\Ec_1|,|\Ec_2|\}$.
We define $\mathrm{MNN\_ratio}_t := M_t/N$ and terminate bootstrapping if any of the following holds:
(i) $M_t=0$ (no mutual pairs);
(ii) after burn-in $T_{\min}$ (default $10$), the ratio stabilizes, i.e.,
\[
\max_{i\in\{t-T_{\min}+1,\dots,t\}}\left|\mathrm{MNN\_ratio}_i-\mathrm{MNN\_ratio}_{i-1}\right| < 0.01;
\]
or (iii) a maximum of 100 iterations is reached.

\subsubsection{Additional Posterior/Precision--Proximity Analyses Across Datasets}
\label{app:iter_vis}
We report $\mathcal{L}_{1}$(link set from the first iteration) under 20\% ground-truth overlap with 15 seed pairs on \mistral and \openai embeddings. Predicted links are grouped into 30 quantile bins by their minimum cosine distance (in the corresponding embedding space) to the anchors that voted for them; we plot per-bin precision and mean posterior. As shown in Figure~\ref{fig:posterior}, links supported by anchors at smaller distances are more accurate, consistent with our Theorem~\ref{thm-local-iso}, and the mean posterior closely tracks empirical precision across bins, indicating that the confidence score is well-calibrated.
\begin{figure*}[t]
  \centering

  \begin{subfigure}[t]{0.48\textwidth}
    \centering
    \includegraphics[width=\linewidth]{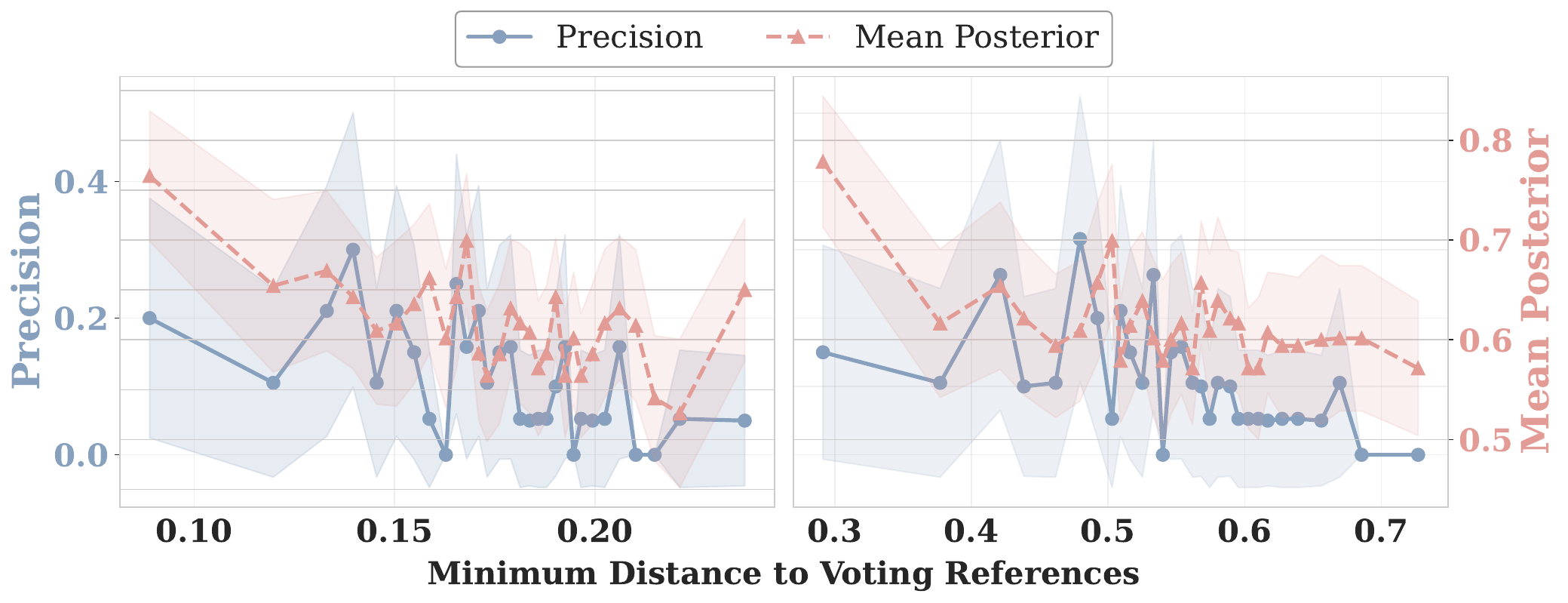}
    \caption{\nfcorpus}
    \label{fig:post-nfcorpus}
  \end{subfigure}\hfill
  \begin{subfigure}[t]{0.48\textwidth}
    \centering
    \includegraphics[width=\linewidth]{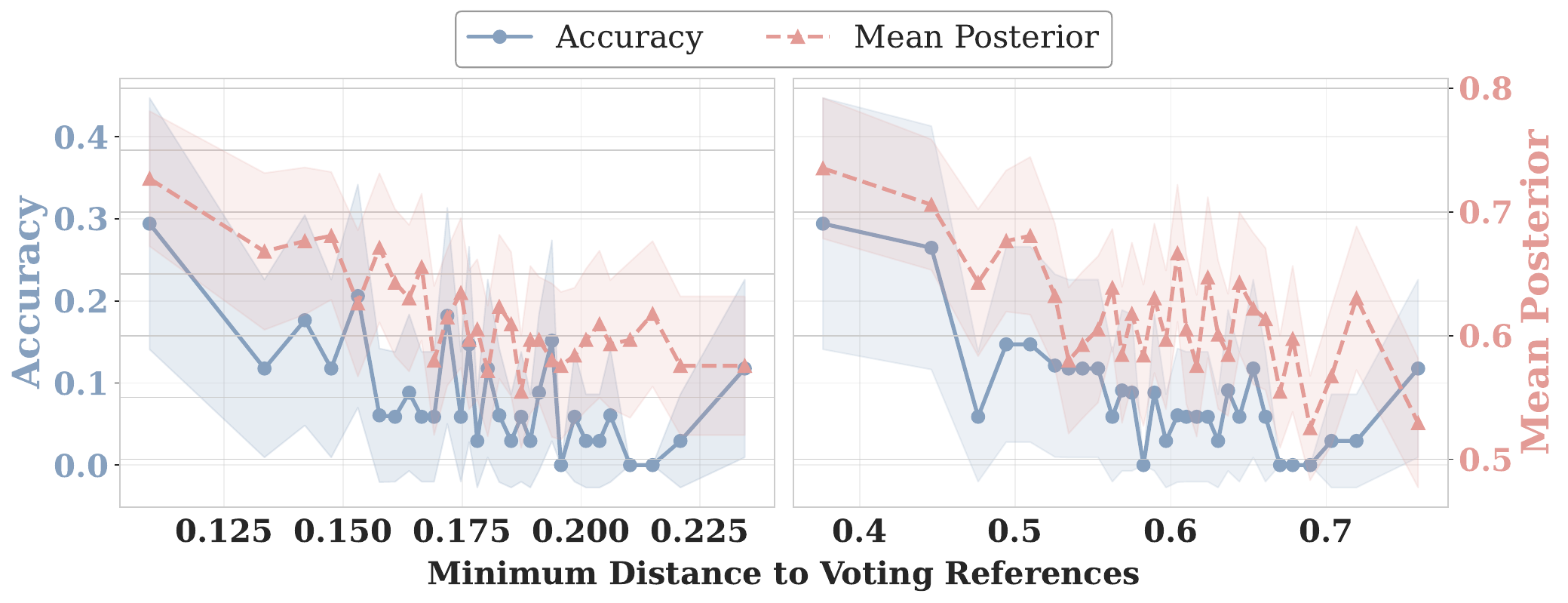}
    \caption{\scifact}
    \label{fig:post-scifact}
  \end{subfigure}

  \vspace{0.6em}

  \begin{subfigure}[t]{0.48\textwidth}
    \centering
    \includegraphics[width=\linewidth]{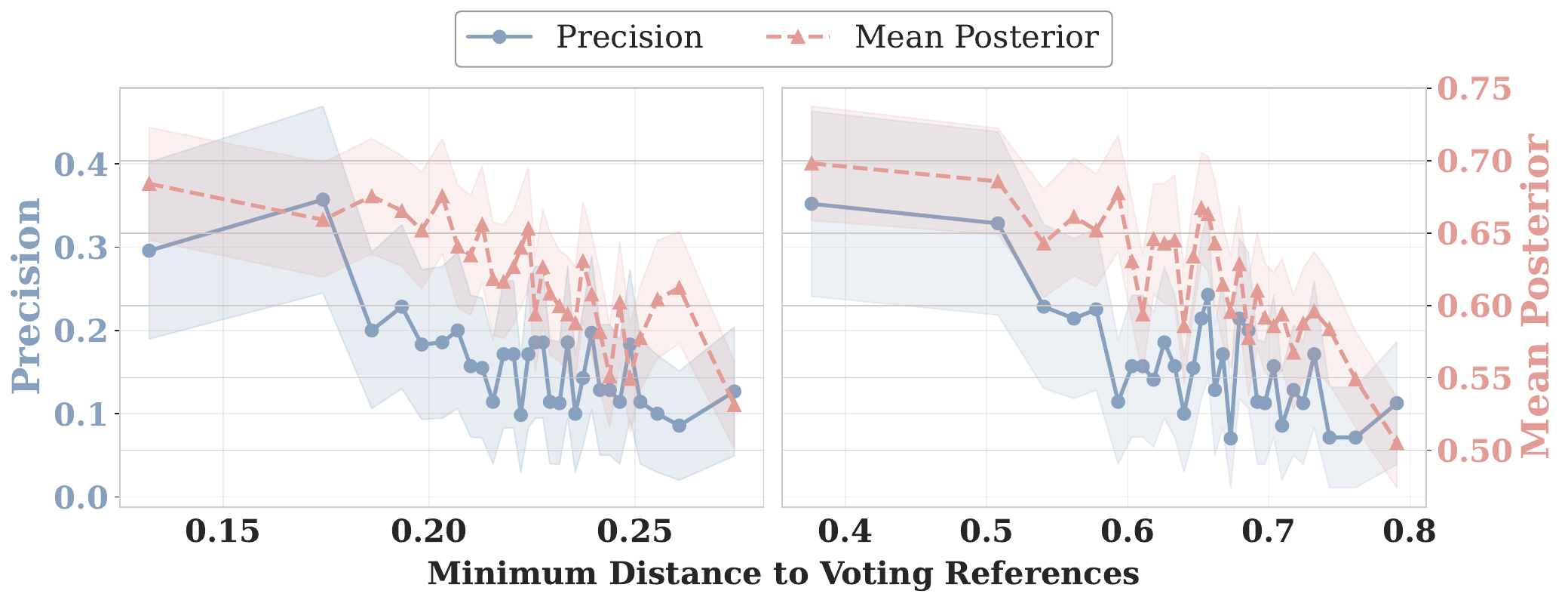}
    \caption{\arguana}
    \label{fig:post-arguana}
  \end{subfigure}\hfill
  \begin{subfigure}[t]{0.48\textwidth}
    \centering
    \includegraphics[width=\linewidth]{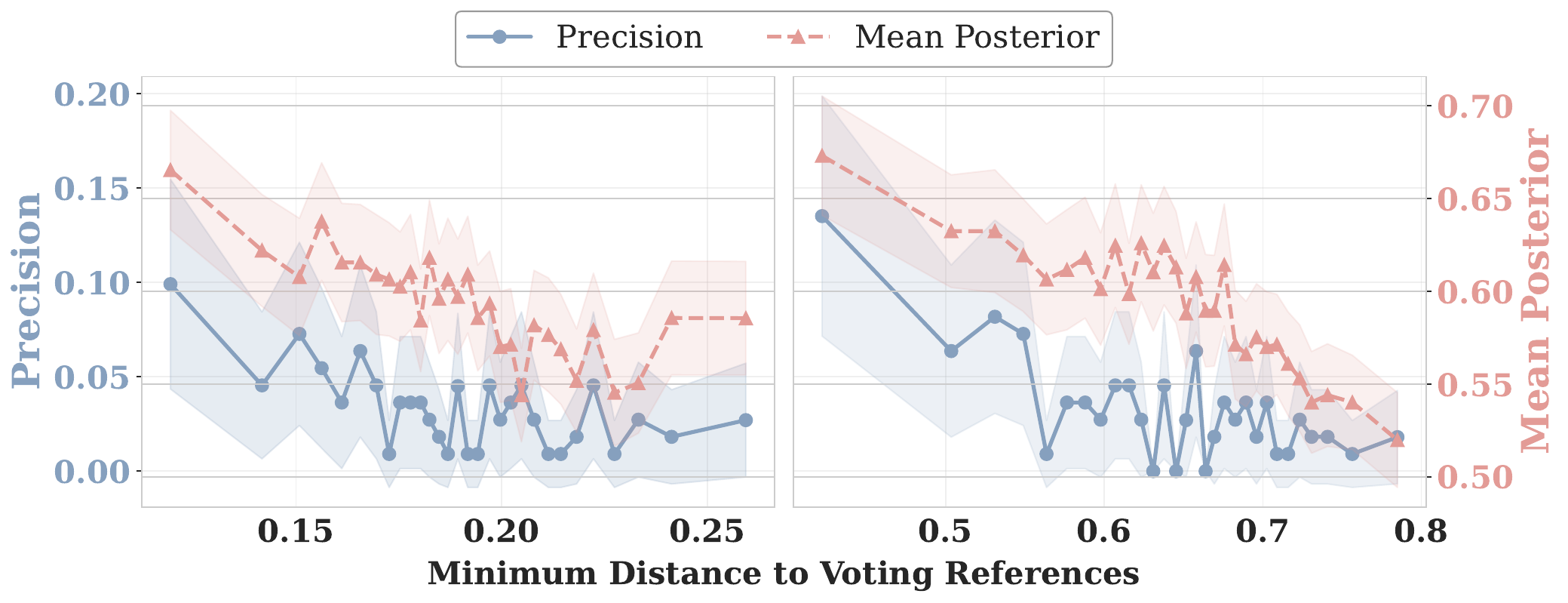}
    \caption{\scidocs}
    \label{fig:post-fourth}
  \end{subfigure}

  \caption{Posterior/Precision vs.\ anchor proximity.\normalfont\footnotesize {
    For \mistral$\leftrightarrow$\openai{} linking at $\alpha=0.2$ overlap with $|\Sc|=15$ seeds, we bin predicted links by their minimum distance
    to the anchors that voted for them (30 quantile bins) and plot per-bin empirical precision and mean posterior confidence.}}
  \label{fig:posterior}
\end{figure*}
\subsubsection{Robustness to $\Sc$ Initialization}

We analyze the sensitivity of \mvh to the structural properties of the initial supervision set $\Sc$. Since \mvh relies on a small set of anchors sampled from the overlap region between two embedding sets, we test whether different strategies for selecting these seed pairs materially affect final performance.

\stitle{Setup.} We evaluate between \qwen and \openai embeddings across five datasets. We hold the overlap ratio $\alpha$ and number of seed anchors $|\Sc|$ constant while varying the sampling strategy used to select$|\Sc|$ from the overlap.

We compare 4 different strategies:

\begin{itemize}[leftmargin=3ex, itemsep=0ex, topsep=0.0ex, partopsep=0.0ex]
\item \textbf{Nearest}: Randomly choose one anchor and take its $k-1$ nearest neighbors (in \qwen space). This produces a localized supervision pattern.
\item \textbf{Random}: Seeds are sampled uniformly without replacement. This is the default strategy employed in our main experiments, requiring no prior knowledge of the overlap manifold.
\item \textbf{FPS}: Greedily build a seed set by repeatedly selecting the candidate that maximizes its minimum cosine distance to previously selected seeds. This yields a highly diverse seed set and has a “spread-out” supervision.
\item \textbf{Centroids}: We cluster the overlapping vectors (in \qwen space) using $k$-means (where  $k=|\Sc|$) and select the vectors nearest to the centroids. This ensures seed anchors are representative of the distribution.
\end{itemize}
\revise{Table~\ref{tab:init_strategies_per_setup} reports F1 across all 45 (dataset, $\alpha$, $|\Sc|$) configurations. Performance is stable across strategies (the largest max$-$min spread is 3.0\,pp), confirming that \mvh is robust to the choice of seed anchors.}

\begin{table*}[!htbp]
  \centering
  \footnotesize
  \setlength{\tabcolsep}{6pt}
  \renewcommand{\arraystretch}{1.1}
  \caption{\revise{F1 ($\%$) by seed-initialization strategy {\normalfont\footnotesize on \qwen$\leftrightarrow$\openai{} linking across five datasets, three overlap ratios $\alpha\in\{0.15,0.20,0.30\}$, and three seed budgets $|\Sc|\in\{15,20,30\}$. Bold marks the best strategy within each row. The \textit{Max Gap} column reports the spread (max$-$min) across strategies in percentage points (pp); the bolded value is the largest spread observed across all 45 setups.}}}
  \label{tab:init_strategies_per_setup}
  \begin{tabular*}{\linewidth}{@{\extracolsep{\fill}}llc ccccc@{}}
  \toprule
  Dataset & $\alpha$ & $|\Sc|$ & Nearest & Random & FPS & Centroids & \textit{Max Gap} \\
  \midrule
  \multirow{9}{*}{\fiqa}
    & \multirow{3}{*}{0.15} & 15 & \textbf{62.5} & 61.4          & 60.9          & 61.5          & 1.6 \\
    &                       & 20 & \textbf{62.6} & 61.3          & 61.3          & 61.2          & 1.4 \\
    &                       & 30 & 63.0          & 62.2          & 62.1          & \textbf{63.0} & 0.9 \\
    \cmidrule(lr){2-8}
    & \multirow{3}{*}{0.20} & 15 & \textbf{68.4} & 67.5          & 67.3          & 67.0          & 1.4 \\
    &                       & 20 & \textbf{69.0} & 68.1          & 67.2          & 67.8          & 1.8 \\
    &                       & 30 & \textbf{68.8} & 67.8          & 68.4          & 67.6          & 1.2 \\
    \cmidrule(lr){2-8}
    & \multirow{3}{*}{0.30} & 15 & \textbf{75.5} & 74.2          & 74.2          & 74.3          & 1.3 \\
    &                       & 20 & \textbf{75.3} & 75.1          & 74.4          & 74.9          & 0.9 \\
    &                       & 30 & \textbf{75.6} & 74.0          & 74.6          & 75.4          & 1.6 \\
  \midrule
  \multirow{9}{*}{\scidocs}
    & \multirow{3}{*}{0.15} & 15 & \textbf{75.2} & 74.6          & 73.6          & 74.3          & 1.6 \\
    &                       & 20 & 73.8          & \textbf{75.4} & 74.6          & 75.2          & 1.6 \\
    &                       & 30 & 74.0          & 74.7          & 74.5          & \textbf{74.9} & 0.9 \\
    \cmidrule(lr){2-8}
    & \multirow{3}{*}{0.20} & 15 & 80.1          & 80.3          & 79.5          & \textbf{80.5} & 1.0 \\
    &                       & 20 & 79.7          & \textbf{80.3} & 80.2          & 80.2          & 0.6 \\
    &                       & 30 & 80.1          & \textbf{80.6} & 80.4          & 79.9          & 0.7 \\
    \cmidrule(lr){2-8}
    & \multirow{3}{*}{0.30} & 15 & 86.9          & \textbf{87.2} & 86.7          & 86.6          & 0.6 \\
    &                       & 20 & 86.7          & \textbf{87.0} & 86.9          & 86.7          & 0.3 \\
    &                       & 30 & 87.0          & \textbf{87.2} & 86.9          & 86.6          & 0.6 \\
  \midrule
  \multirow{9}{*}{\arguana}
    & \multirow{3}{*}{0.15} & 15 & 67.1          & 67.6          & 67.1          & \textbf{68.7} & 1.6 \\
    &                       & 20 & 67.9          & 68.0          & 67.6          & \textbf{68.3} & 0.7 \\
    &                       & 30 & \textbf{68.0} & 67.3          & 67.1          & 67.7          & 0.9 \\
    \cmidrule(lr){2-8}
    & \multirow{3}{*}{0.20} & 15 & 74.2          & 74.9          & 74.5          & \textbf{74.9} & 0.7 \\
    &                       & 20 & 74.1          & \textbf{75.3} & 74.3          & 74.8          & 1.2 \\
    &                       & 30 & 74.1          & 74.6          & 73.9          & \textbf{75.0} & 1.1 \\
    \cmidrule(lr){2-8}
    & \multirow{3}{*}{0.30} & 15 & 82.3          & 82.1          & 82.8          & \textbf{83.1} & 1.0 \\
    &                       & 20 & 82.4          & 82.0          & 82.4          & \textbf{82.9} & 0.9 \\
    &                       & 30 & \textbf{82.4} & 82.0          & 82.2          & 82.3          & 0.4 \\
  \midrule
  \multirow{9}{*}{\scifact}
    & \multirow{3}{*}{0.15} & 15 & 74.7          & 76.5          & \textbf{77.7} & 77.5          & \textbf{3.0} \\
    &                       & 20 & 76.3          & \textbf{77.4} & 76.7          & 76.3          & 1.1 \\
    &                       & 30 & 75.3          & \textbf{75.4} & 75.1          & 74.4          & 1.0 \\
    \cmidrule(lr){2-8}
    & \multirow{3}{*}{0.20} & 15 & 81.8          & 83.1          & 82.4          & \textbf{83.2} & 1.4 \\
    &                       & 20 & 82.1          & 83.1          & \textbf{83.5} & 82.5          & 1.4 \\
    &                       & 30 & 82.3          & 81.8          & \textbf{82.4} & 81.4          & 1.0 \\
    \cmidrule(lr){2-8}
    & \multirow{3}{*}{0.30} & 15 & 89.8          & 89.9          & 89.6          & \textbf{90.0} & 0.4 \\
    &                       & 20 & \textbf{90.0} & 90.0          & 89.4          & 90.0          & 0.6 \\
    &                       & 30 & 89.4          & \textbf{89.7} & 89.2          & 89.3          & 0.5 \\
  \midrule
  \multirow{9}{*}{\nfcorpus}
    & \multirow{3}{*}{0.15} & 15 & 73.1          & \textbf{74.4} & 74.0          & 73.9          & 1.3 \\
    &                       & 20 & \textbf{73.3} & 73.0          & 73.3          & 72.8          & 0.5 \\
    &                       & 30 & 71.4          & \textbf{72.4} & 71.8          & 71.0          & 1.4 \\
    \cmidrule(lr){2-8}
    & \multirow{3}{*}{0.20} & 15 & \textbf{80.2} & 79.4          & 79.7          & 79.7          & 0.8 \\
    &                       & 20 & 79.3          & 79.6          & \textbf{80.0} & 79.4          & 0.7 \\
    &                       & 30 & 77.8          & \textbf{79.5} & 78.9          & 78.1          & 1.7 \\
    \cmidrule(lr){2-8}
    & \multirow{3}{*}{0.30} & 15 & 87.7          & \textbf{88.3} & 87.9          & 87.0          & 1.3 \\
    &                       & 20 & 87.7          & \textbf{88.3} & 87.5          & 87.7          & 0.8 \\
    &                       & 30 & \textbf{88.1} & 87.4          & 87.4          & 87.2          & 0.9 \\
  \bottomrule
  \end{tabular*}
\end{table*}

\eat{

\subsubsection{Per-Component Ablation Details}
\label{app:ablation}

We ablate the five components of \mvh\ on \scidocs\ (\mistral\ {\sc vs.}\ \openai) with $\alpha\!=\!0.15$ and $|\Sc|\!=\!15$.
For each variant we run $10$ independent trials, each with a fresh random seed controlling both the draw of $\Sc$ and the internal randomness of view generation, and report the across-trial mean and standard deviation ($\mu\pm\sigma$). Every removal degrades both precision and recall, so no single component is dispensable.%

Each variant is \mvh\ but with the named component swapped for a simpler default. The replacement defaults are:
\begin{itemize}[leftmargin=3ex,itemsep=0.2ex,topsep=0.4ex,partopsep=0ex]
  \item \textbf{$-$ Kernel}: use the raw distance vector $r_\Ac$ instead of the kernelized signature $h_\Ac$ (Section~\ref{sec-linking}).
  \item \textbf{$-$ FPS}: draw view anchors uniformly at random from $\Lc_{t-1}$ instead of by FPS (Section~\ref{sec-linking}).
  \item \textbf{$-$ FPS \& Kernel}: both per-view defaults applied jointly (raw signatures plus random view draws).
  \item \textbf{$-$ View scheduling}: freeze the schedule at iteration zero, $\rho_t\equiv\rho_0$ and $m_t\equiv m_0$ (Section~\ref{sec-linking}).
  \item \textbf{$-$ Multi-view voting}: use single-view MNN proposal, $m_t{=}1$ and $\Ac_{t,1}{=}\Lc_{t-1}$ (Section~\ref{sec-hash}).
  \item \textbf{$-$ Bootstrapping}: run a single iteration on the $\Sc$; no anchor-pool growth (Section~\ref{sec-linking}).
\end{itemize}

Table~\ref{tab:scidocs-ablation} reports the results. Bootstrapping and multi-view voting account for most of the absolute
performance. Without bootstrapping, only the 15 seeds are available as
anchors and \mvh\ collapses to $1.9/2.1/1.9$; without multi-view voting,
single-view link proposals cannot separate true links from
distortion-driven collisions and performance falls to $24.0/39.4/29.8$.

Removing FPS or the kernel signature makes performance highly variable (e.g., recall $\sigma$ rises from $0.7$ to $19$--$35$): these two per-view components act as variance reducers, FPS by spreading anchors so each view stays well-conditioned and the kernel by suppressing the long-range distance regime that decorrelates across encoders (Section~\ref{subsec-evidence}).

\begin{table}[t]
  \centering
  \small
  \setlength{\tabcolsep}{6pt}
  \caption{
    Ablation of \mvh\ components
  {\normalfont\footnotesize on \scidocs\ with \mistral\ {\sc vs.} \openai, at $\alpha$=0.15 and seeds $|\Sc|$=15: 
  \mvh denotes the complete pipeline: FPS view sampling, Kernelized signature, adaptive view scheduling, multi-view posterior aggregation, and bootstrapping. 
  Each subsequent row removes one component relative to \mvh; the row \mbox{$-$ FPS \& Kernel} removes both per-view components jointly. 
  Each cell reports the seed-level mean $\pm$ standard deviation (\%) of the corresponding metric. 
  A $\dag$ marks cells where the seed-level coefficient of variation ($\sigma/\mu$) exceeds $10\%$.
  }
  }
  \label{tab:scidocs-ablation}
  \begin{tabular}{lccc}
  \toprule
  Variant & Precision (\%) & Recall (\%) & F1 (\%) \\
  \midrule
  \mvh                               &  62.1 $\pm$  1.1     &  81.7 $\pm$  0.7     &  70.5 $\pm$  0.6     \\
  \;\;$-$ Kernel                     &  61.0 $\pm$  8.3\dag &  52.9 $\pm$ 35.0\dag &  51.0 $\pm$ 33.3\dag \\
  \;\;$-$ FPS                        &  32.8 $\pm$ 26.7\dag &  16.3 $\pm$ 19.2\dag &  20.9 $\pm$ 23.3\dag \\
  \;\;$-$ FPS \& Kernel              &  28.3 $\pm$ 23.3\dag &   0.8 $\pm$  0.5\dag &   1.4 $\pm$  1.1\dag \\
  \;\;$-$ Adaptive schedule          &  33.4 $\pm$  0.6     &  61.8 $\pm$  1.6     &  43.4 $\pm$  0.9     \\
  \;\;$-$ Multi-view voting          &  24.0 $\pm$  2.6\dag &  39.4 $\pm$  5.0\dag &  29.8 $\pm$  3.4\dag \\
  \;\;$-$ Bootstrapping              &   1.9 $\pm$  0.6\dag &   2.1 $\pm$  1.2\dag &   1.9 $\pm$  0.8\dag \\
  \bottomrule
  \end{tabular}
\end{table}
}
\subsubsection{Sensitivity to $c$, $\rho_0$, and $k_{\mathrm{CSLS}}$}\label{app:sensitivity}
We evaluate robustness on \scifact{} and \nfcorpus{} using \mistral$\leftrightarrow$\openai{} embeddings, fixing the overlap ratio to $\alpha=0.3$ and the seed budget to $|\Sc|=15$. We sweep three hyperparameters: the view-growth constant $c$ in $\mathrm{sf}(g)=1+c\log g$, the base per-view anchor fraction $\rho_0$, and the CSLS neighborhood size $k_{\mathrm{CSLS}}$ used for MNN retrieval. When varying $\rho_0$, we set $m_0=\lceil 2/\rho_0\rceil$ to keep the expected per-anchor coverage approximately constant ($m_0\rho_0\approx 2$, up to rounding). For each sweep, we report the mean F1 across the two datasets and define the stable range as configurations achieving at least $0.97\times$ the best mean F1 in that sweep. As shown in Figure~\ref{fig-sensitivity}, performance is stable over broad intervals: $c\in[0.2,1]$, $k_{\mathrm{CSLS}}\in[4,50]$, and $\rho_0\in[0.2,0.5]$. We use $c=0.3$, $\rho_0=0.4$, and $k_{\mathrm{CSLS}}=50$ in all experiments.
\begin{figure*}
\centering
\centerline{\includegraphics[width=0.9\linewidth]{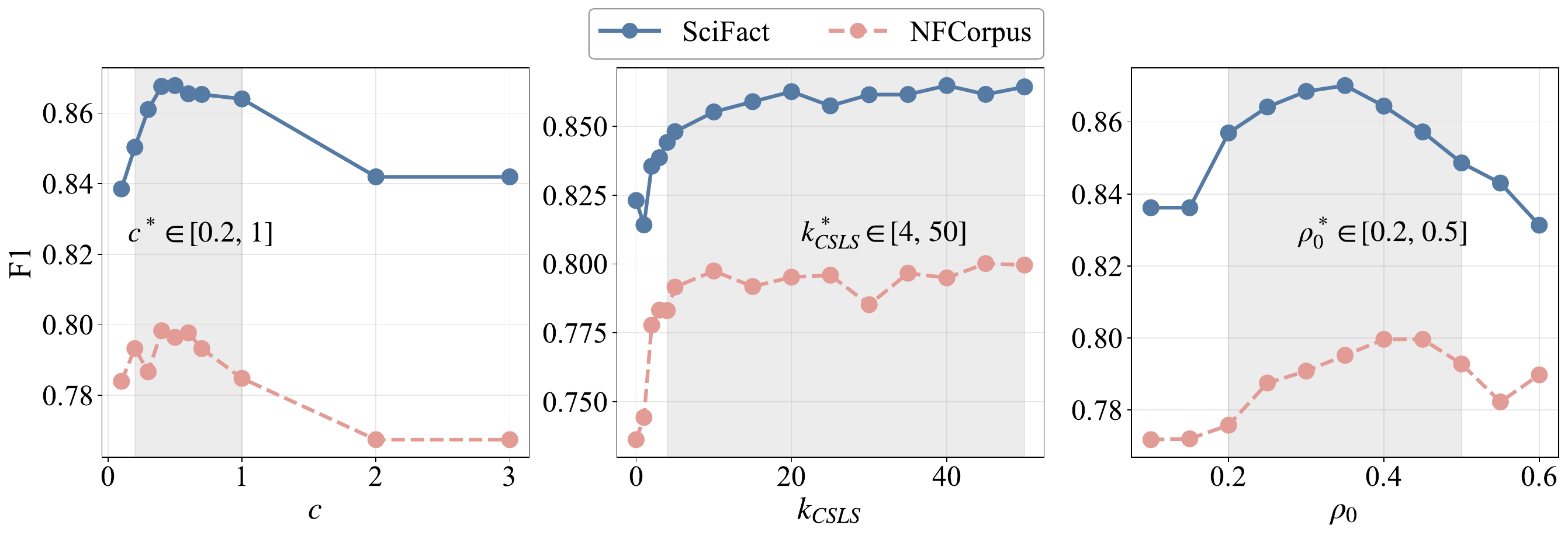}}
\caption{\textbf{Sensitivity to view scheduling and CSLS hyperparameters.}\normalfont\footnotesize{
F1 on \scifact{} and \nfcorpus{} for \mistral$\leftrightarrow$\openai{} linking with overlap ratio $\alpha=0.3$ and $|\Sc|=15$ seeds.
We vary (left) the logarithmic growth constant $c$ in $\mathrm{sf}(g)=1+c\log g$, (middle) the CSLS neighborhood size $k_{\mathrm{CSLS}}$,
and (right) the base per-view anchor fraction $\rho_0$. The shaded gray region denotes the near-optimal range achieving at least $97\%$ of the peak F1 for each sweep.}}
\label{fig-sensitivity}
\end{figure*}

\subsection{Additional Results of \mvh}

\subsubsection{Additional Results on Vector Linking}\label{app:more_model_pairs}
We report the complete experimental grid over $5$ model pairs and $5$ datasets, across $3$ overlap ratios and $3$ seed budgets.
Across this grid, \mvh achieves the best performance among all methods in the vast majority of settings.

\begin{table*}[!htbp]
    \centering
    \scriptsize
\setlength{\tabcolsep}{1.5ex} %
\renewcommand{\arraystretch}{1.1}%
    \caption{Vector linking on \nfcorpus{} (\gte{}$\leftrightarrow$\mistral{}):
    {\normalfont each cell reports \textbf{precision/recall/F1} (\%). Best values per metric are \textbf{bolded}.}}
    \label{tab:single_nfcorpus_gte_mistral}
\begin{tabular}{l|ccc|ccc|ccc}
    \hline
    \multirow{2}{*}{Method} & \multicolumn{3}{c|}{Overlap .15} & \multicolumn{3}{c|}{Overlap .20} & \multicolumn{3}{c}{Overlap .30} \\
     & Seeds 15 & Seeds 20 & Seeds 30 & Seeds 15 & Seeds 20 & Seeds 30 & Seeds 15 & Seeds 20 & Seeds 30 \\
    \hline
    \linear & 1.1/1.2/1.1 & 1.8/2.0/1.9 & 3.6/3.9/3.8 & 1.1/1.2/1.2 & 1.9/2.3/2.1 & 4.3/4.7/4.5 & 0.9/1.3/1.1 & 1.5/2.5/1.9 & 4.7/6.6/5.5 \\
    \cca & 33.3/12.8/18.5 & 38.4/20.8/26.9 & 47.9/42.5/45.1 & 31.7/8.0/12.8 & 44.3/17.5/25.1 & 55.7/37.2/44.6 & 46.0/7.5/13.0 & 54.1/18.5/27.6 & 68.1/36.6/47.6 \\
    \mlp & 10.8/2.1/3.5 & 23.0/6.1/9.6 & 35.7/18.4/24.3 & 15.6/2.1/3.7 & 25.1/5.9/9.6 & 41.8/17.9/25.1 & 27.0/2.8/5.1 & 43.0/6.9/11.9 & 54.0/17.7/26.6 \\
    \rcsls & \textbf{62.8}/2.5/4.8 & 40.4/3.6/6.5 & 33.4/7.4/12.1 & 60.9/2.0/3.9 & 51.3/3.1/5.9 & 44.7/9.3/15.4 & 63.8/1.3/2.6 & 59.2/2.0/3.9 & 50.9/4.9/9.0 \\
    \pr & 34.8/13.0/19.0 & 40.1/22.3/28.6 & 48.6/41.0/44.5 & 37.6/9.0/14.5 & 46.0/18.5/26.4 & 56.6/38.2/45.6 & 53.0/9.9/16.6 & 61.8/22.1/32.5 & 72.5/37.7/49.6 \\
    \revise{\ugw} & \revise{2.3/1.7/2.0} & \revise{2.1/1.5/1.8} & \revise{2.3/1.7/2.0} & \revise{17.8/13.1/15.1} & \revise{17.7/13.0/15.0} & \revise{18.0/13.3/15.3} & \revise{20.7/10.9/14.3} & \revise{20.5/10.8/14.1} & \revise{19.8/10.4/13.6} \\
    \ao & 3.3/3.3/3.3 & 22.8/22.8/22.8 & 26.1/26.1/26.1 & 5.6/5.6/5.6 & 11.4/11.4/11.4 & 25.3/25.3/25.3 & 1.7/1.7/1.7 & 2.5/2.5/2.5 & 39.5/39.5/39.5 \\
    \textbf{\mvh} & 58.8/\textbf{78.3}/\textbf{67.2} & \textbf{58.1}/\textbf{78.3}/\textbf{66.7} & \textbf{57.7}/\textbf{77.3}/\textbf{66.1} & \textbf{67.5}/\textbf{75.4}/\textbf{71.2} & \textbf{65.5}/\textbf{74.5}/\textbf{69.7} & \textbf{66.0}/\textbf{75.6}/\textbf{70.5} & \textbf{79.6}/\textbf{78.8}/\textbf{79.2} & \textbf{79.7}/\textbf{78.7}/\textbf{79.2} & \textbf{77.5}/\textbf{78.2}/\textbf{77.9} \\
    \hline
    \end{tabular}
\end{table*}

\begin{table*}[!htbp]
    \centering
    \scriptsize
\setlength{\tabcolsep}{1.5ex} %
\renewcommand{\arraystretch}{1.1}%
    \caption{Vector linking on \scidocs{} (\gte{}$\leftrightarrow$\mistral{}):
    {\normalfont each cell reports \textbf{precision/recall/F1} (\%). Best values per metric are \textbf{bolded}.}}
    \label{tab:single_scidocs_gte_mistral}
    \begin{tabular}{l|ccc|ccc|ccc}
    \hline
    \multirow{2}{*}{Method} & \multicolumn{3}{c|}{Overlap .15} & \multicolumn{3}{c|}{Overlap .20} & \multicolumn{3}{c}{Overlap .30} \\
     & Seeds 15 & Seeds 20 & Seeds 30 & Seeds 15 & Seeds 20 & Seeds 30 & Seeds 15 & Seeds 20 & Seeds 30 \\
    \hline
    \linear & 0.3/0.1/0.2 & 0.5/0.3/0.4 & 0.9/0.9/0.9 & 0.2/0.2/0.2 & 0.4/0.3/0.3 & 1.0/0.9/1.0 & 0.3/0.3/0.3 & 0.5/0.5/0.5 & 1.2/1.5/1.3 \\
    \cca & 13.0/2.1/3.6 & 23.0/5.4/8.8 & 37.2/21.6/27.3 & 17.2/2.3/4.0 & 29.8/7.0/11.4 & 43.6/19.2/26.6 & 18.2/2.0/3.6 & 37.0/6.7/11.3 & 59.1/21.1/31.1 \\
    \mlp & 7.0/0.4/0.7 & 10.8/1.0/1.8 & 26.0/6.0/9.8 & 11.1/0.5/1.0 & 18.9/1.2/2.2 & 26.8/4.0/7.0 & 11.4/0.4/0.8 & 24.6/1.4/2.7 & 33.8/3.9/7.0 \\
    \rcsls & 60.9/0.3/0.7 & 45.3/0.5/1.0 & 19.5/1.0/1.9 & 55.2/0.2/0.5 & 38.0/0.3/0.7 & 20.0/0.7/1.4 & 44.9/0.2/0.3 & 27.9/0.3/0.5 & 20.4/0.6/1.2 \\
    \pr & 17.8/2.1/3.7 & 29.5/6.2/10.3 & 47.4/22.8/30.8 & 23.7/2.7/4.9 & 36.2/7.4/12.3 & 52.5/21.8/30.8 & 26.7/2.3/4.2 & 48.1/7.5/12.9 & 67.6/21.7/32.8 \\
    \revise{\ugw} & \revise{5.9/0.6/1.1} & \revise{5.9/0.6/1.1} & \revise{5.9/0.6/1.1} & \revise{4.4/0.3/0.6} & \revise{4.4/0.3/0.6} & \revise{4.4/0.3/0.6} & \revise{7.0/0.4/0.7} & \revise{7.1/0.4/0.7} & \revise{7.1/0.4/0.7} \\
    \ao & 0.4/0.4/0.4 & 0.5/0.5/0.5 & 0.8/0.8/0.8 & 0.3/0.3/0.3 & 0.4/0.4/0.4 & 1.1/1.1/1.1 & 0.2/0.2/0.2 & 0.3/0.3/0.3 & 0.4/0.4/0.4 \\
    \textbf{\mvh} & \textbf{62.8}/\textbf{55.2}/\textbf{58.7} & \textbf{62.5}/\textbf{51.7}/\textbf{56.6} & \textbf{62.0}/\textbf{62.3}/\textbf{62.2} & \textbf{70.0}/\textbf{60.0}/\textbf{64.6} & \textbf{71.7}/\textbf{60.8}/\textbf{65.8} & \textbf{71.7}/\textbf{63.8}/\textbf{67.5} & \textbf{80.8}/\textbf{64.3}/\textbf{71.6} & \textbf{82.2}/\textbf{67.2}/\textbf{74.0} & \textbf{81.3}/\textbf{68.0}/\textbf{74.1} \\
    \hline
    \end{tabular}
\end{table*}

\begin{table*}[!htbp]
    \centering
    \scriptsize
\setlength{\tabcolsep}{1.5ex} %
\renewcommand{\arraystretch}{1.1}%
    \caption{Vector linking on \arguana{} (\gte{}$\leftrightarrow$\mistral{}):
    {\normalfont each cell reports \textbf{precision/recall/F1} (\%). Best values per metric are \textbf{bolded}.}}
    \label{tab:single_arguana_gte_mistral}
    \begin{tabular}{l|ccc|ccc|ccc}
    \hline
    \multirow{2}{*}{Method} & \multicolumn{3}{c|}{Overlap .15} & \multicolumn{3}{c|}{Overlap .20} & \multicolumn{3}{c}{Overlap .30} \\
     & Seeds 15 & Seeds 20 & Seeds 30 & Seeds 15 & Seeds 20 & Seeds 30 & Seeds 15 & Seeds 20 & Seeds 30 \\
    \hline
    \linear & 0.4/0.4/0.4 & 0.6/0.5/0.5 & 1.0/0.8/0.9 & 0.4/0.3/0.3 & 0.6/0.5/0.5 & 1.1/1.2/1.2 & 0.4/0.3/0.3 & 0.6/0.7/0.6 & 1.5/1.5/1.5 \\
    \cca & 10.8/3.1/4.8 & 17.9/7.4/10.5 & 33.0/19.6/24.6 & 14.1/3.4/5.4 & 24.0/8.2/12.3 & 37.6/18.2/24.6 & 25.2/3.9/6.8 & 35.5/8.5/13.7 & 50.3/21.3/30.0 \\
    \mlp & 9.7/1.0/1.8 & 12.0/1.8/3.1 & 19.9/5.3/8.4 & 13.6/1.2/2.2 & 17.3/2.3/4.1 & 27.7/6.2/10.2 & 11.0/0.7/1.2 & 18.0/1.6/3.0 & 33.5/5.1/8.8 \\
    \rcsls & 38.3/0.6/1.2 & 41.8/1.1/2.1 & 28.0/2.3/4.2 & 39.3/0.5/1.0 & 30.6/0.8/1.6 & 26.5/1.9/3.5 & 26.2/0.2/0.4 & 26.3/0.3/0.6 & 30.7/1.2/2.3 \\
    \pr & 15.9/4.8/7.4 & 24.2/10.1/14.3 & 36.9/21.9/27.5 & 20.5/4.9/7.9 & 28.9/9.5/14.3 & 42.7/21.4/28.5 & 30.8/4.8/8.4 & 39.4/9.9/15.8 & 60.6/26.6/36.9 \\
    \revise{\ugw} & \revise{0.0/0.0/0.0} & \revise{0.0/0.0/0.0} & \revise{0.0/0.0/0.0} & \revise{0.0/0.0/0.0} & \revise{0.0/0.0/0.0} & \revise{0.0/0.0/0.0} & \revise{5.9/0.1/0.2} & \revise{5.9/0.1/0.2} & \revise{6.1/0.1/0.2} \\
    \ao & 1.6/1.6/1.6 & 1.6/1.6/1.6 & 2.9/2.9/2.9 & 0.9/0.9/0.9 & 1.5/1.5/1.5 & 30.6/30.6/30.6 & 0.6/0.6/0.6 & 1.6/1.6/1.6 & 1.5/1.5/1.5 \\
    \textbf{\mvh} & \textbf{53.2}/\textbf{83.4}/\textbf{65.0} & \textbf{53.7}/\textbf{81.7}/\textbf{64.9} & \textbf{55.5}/\textbf{81.8}/\textbf{66.1} & \textbf{64.0}/\textbf{84.2}/\textbf{72.7} & \textbf{64.3}/\textbf{84.0}/\textbf{72.8} & \textbf{64.1}/\textbf{83.5}/\textbf{72.5} & \textbf{77.1}/\textbf{84.5}/\textbf{80.7} & \textbf{77.5}/\textbf{83.5}/\textbf{80.4} & \textbf{77.7}/\textbf{83.3}/\textbf{80.4} \\
    \hline
    \end{tabular}
\end{table*}

\begin{table*}[!htbp]
    \centering
    \scriptsize
\setlength{\tabcolsep}{1.5ex} %
\renewcommand{\arraystretch}{1.1}%
    \caption{Vector linking on \scifact{} (\gte{}$\leftrightarrow$\mistral{}):
    {\normalfont each cell reports \textbf{precision/recall/F1} (\%). Best values per metric are \textbf{bolded}.}}
    \label{tab:single_scifact_gte_mistral}
    \begin{tabular}{l|ccc|ccc|ccc}
    \hline
    \multirow{2}{*}{Method} & \multicolumn{3}{c|}{Overlap .15} & \multicolumn{3}{c|}{Overlap .20} & \multicolumn{3}{c}{Overlap .30} \\
     & Seeds 15 & Seeds 20 & Seeds 30 & Seeds 15 & Seeds 20 & Seeds 30 & Seeds 15 & Seeds 20 & Seeds 30 \\
    \hline
    \linear & 0.8/0.6/0.7 & 1.3/1.7/1.5 & 3.0/3.4/3.2 & 0.8/1.0/0.9 & 1.4/1.7/1.6 & 3.1/3.9/3.4 & 0.7/1.0/0.8 & 1.8/2.4/2.0 & 3.7/5.3/4.4 \\
    \cca & 22.2/5.4/8.7 & 34.6/14.3/20.2 & 43.8/34.4/38.5 & 29.6/7.9/12.5 & 43.6/16.5/24.0 & 57.9/38.9/46.6 & 36.1/5.9/10.2 & 49.6/14.0/21.9 & 70.4/35.8/47.5 \\
    \mlp & 8.2/1.1/1.9 & 15.1/3.3/5.4 & 32.3/15.1/20.6 & 20.3/2.5/4.4 & 28.3/5.7/9.5 & 41.4/13.9/20.8 & 26.0/1.7/3.2 & 38.0/4.2/7.6 & 57.8/16.6/25.8 \\
    \rcsls & \textbf{69.9}/1.8/3.5 & 54.3/3.0/5.6 & 40.4/6.3/11.0 & \textbf{82.4}/1.5/3.0 & 71.7/2.1/4.1 & 44.0/4.7/8.4 & 66.9/1.0/2.0 & 47.7/1.3/2.6 & 42.9/3.5/6.4 \\
    \pr & 25.3/5.8/9.4 & 42.5/15.8/23.1 & 48.0/37.4/42.0 & 31.3/8.0/12.8 & 49.1/16.3/24.5 & 61.1/38.1/46.9 & 45.6/7.7/13.2 & 56.9/15.6/24.5 & 72.9/35.9/48.1 \\
    \revise{\ugw} & \revise{6.2/1.8/2.8} & \revise{6.2/1.8/2.9} & \revise{5.8/1.7/2.7} & \revise{9.4/2.3/3.6} & \revise{8.8/2.1/3.3} & \revise{7.9/1.9/3.0} & \revise{13.1/2.3/3.9} & \revise{13.0/2.3/3.9} & \revise{12.9/2.2/3.8} \\
    \ao & 4.2/4.2/4.2 & 4.9/4.9/4.9 & 25.7/25.7/25.7 & 1.7/1.7/1.7 & 5.8/5.8/5.8 & 6.8/6.8/6.8 & 1.1/1.1/1.1 & 1.7/1.7/1.7 & 1.9/1.9/1.9 \\
    \textbf{\mvh} & 62.0/\textbf{82.0}/\textbf{70.6} & \textbf{62.3}/\textbf{81.6}/\textbf{70.7} & \textbf{60.1}/\textbf{81.8}/\textbf{69.3} & 72.3/\textbf{84.8}/\textbf{78.1} & \textbf{72.6}/\textbf{83.3}/\textbf{77.6} & \textbf{70.4}/\textbf{85.2}/\textbf{77.1} & \textbf{80.4}/\textbf{83.6}/\textbf{82.0} & \textbf{81.5}/\textbf{83.9}/\textbf{82.7} & \textbf{80.9}/\textbf{84.3}/\textbf{82.6} \\
    \hline
    \end{tabular}
\end{table*}

\begin{table*}[!htbp]
    \centering
    \scriptsize
\setlength{\tabcolsep}{1.5ex} %
\renewcommand{\arraystretch}{1.1}%
    \caption{Vector linking on \fiqa{} (\gte{}$\leftrightarrow$\mistral{}):
    {\normalfont each cell reports \textbf{precision/recall/F1} (\%). Best values per metric are \textbf{bolded}.}}
    \label{tab:single_fiqa_gte_mistral}
    \begin{tabular}{l|ccc|ccc|ccc}
    \hline
    \multirow{2}{*}{Method} & \multicolumn{3}{c|}{Overlap .15} & \multicolumn{3}{c|}{Overlap .20} & \multicolumn{3}{c}{Overlap .30} \\
     & Seeds 15 & Seeds 20 & Seeds 30 & Seeds 15 & Seeds 20 & Seeds 30 & Seeds 15 & Seeds 20 & Seeds 30 \\
    \hline
    \linear & 0.1/0.1/0.1 & 0.1/0.1/0.1 & 0.3/0.2/0.2 & 0.1/0.1/0.1 & 0.1/0.1/0.1 & 0.3/0.3/0.3 & 0.1/0.1/0.1 & 0.2/0.1/0.2 & 0.3/0.3/0.3 \\
    \cca & 5.9/1.1/1.9 & 11.4/3.0/4.7 & 23.6/10.9/14.9 & 9.7/1.3/2.3 & 16.4/3.6/5.9 & 32.9/13.3/18.9 & 8.7/0.5/1.0 & 18.2/2.3/4.1 & 34.2/9.2/14.4 \\
    \mlp & 3.2/0.2/0.3 & 5.5/0.5/0.9 & 10.5/1.5/2.6 & 4.1/0.2/0.4 & 9.0/0.7/1.3 & 12.4/1.6/2.8 & 4.9/0.1/0.2 & 11.2/0.4/0.8 & 19.1/1.4/2.5 \\
    \rcsls & 25.6/0.1/0.2 & 17.2/0.2/0.3 & 15.8/0.4/0.9 & 29.0/0.1/0.1 & 23.4/0.1/0.2 & 16.9/0.4/0.7 & 22.1/0.0/0.1 & 21.4/0.1/0.2 & 14.3/0.2/0.4 \\
    \pr & 9.2/1.4/2.5 & 17.6/4.0/6.5 & 31.8/12.5/17.9 & 15.5/1.9/3.4 & 25.0/5.0/8.4 & 43.9/16.0/23.4 & 14.0/0.8/1.6 & 25.8/2.7/4.9 & 45.7/10.0/16.4 \\
    \revise{\ugw} & \revise{0.0/0.0/0.0} & \revise{0.0/0.0/0.0} & \revise{0.0/0.0/0.0} & \revise{1.1/0.0/0.0} & \revise{1.1/0.0/0.0} & \revise{1.1/0.0/0.0} & \revise{6.4/0.0/0.1} & \revise{6.4/0.0/0.1} & \revise{6.4/0.0/0.1} \\
    \ao & 0.2/0.2/0.2 & 0.2/0.2/0.2 & 0.3/0.3/0.3 & 0.1/0.1/0.1 & 0.2/0.2/0.2 & 0.3/0.3/0.3 & 0.1/0.1/0.1 & 0.1/0.1/0.1 & 0.2/0.2/0.2 \\
    \textbf{\mvh} & \textbf{54.8}/\textbf{75.1}/\textbf{63.3} & \textbf{54.8}/\textbf{74.6}/\textbf{63.2} & \textbf{55.4}/\textbf{73.3}/\textbf{63.1} & \textbf{64.9}/\textbf{75.0}/\textbf{69.6} & \textbf{64.5}/\textbf{74.6}/\textbf{69.2} & \textbf{65.0}/\textbf{73.6}/\textbf{69.0} & \textbf{78.6}/\textbf{77.9}/\textbf{78.2} & \textbf{78.4}/\textbf{77.3}/\textbf{77.8} & \textbf{78.5}/\textbf{75.9}/\textbf{77.2} \\
    \hline
    \end{tabular}
\end{table*}

\begin{table*}[!htbp]
    \centering
    \scriptsize
\setlength{\tabcolsep}{1.5ex} %
\renewcommand{\arraystretch}{1.1}%
    \caption{Vector linking on \nfcorpus{} (\gte{}$\leftrightarrow$\openai{}):
    {\normalfont each cell reports \textbf{precision/recall/F1} (\%). Best values per metric are \textbf{bolded}.}}
    \label{tab:single_nfcorpus_gte_openai}
    \begin{tabular}{l|ccc|ccc|ccc}
    \hline
    \multirow{2}{*}{Method} & \multicolumn{3}{c|}{Overlap .15} & \multicolumn{3}{c|}{Overlap .20} & \multicolumn{3}{c}{Overlap .30} \\
     & Seeds 15 & Seeds 20 & Seeds 30 & Seeds 15 & Seeds 20 & Seeds 30 & Seeds 15 & Seeds 20 & Seeds 30 \\
    \hline
    \linear & 2.2/5.9/3.2 & 4.5/9.1/6.0 & 9.2/16.0/11.7 & 2.2/5.3/3.1 & 4.1/10.2/5.9 & 8.7/19.2/11.9 & 3.5/7.5/4.8 & 5.7/14.3/8.1 & 11.2/24.7/15.4 \\
    \cca & 23.0/8.7/12.6 & 35.8/20.2/25.8 & 42.4/41.0/41.7 & 32.4/9.1/14.2 & 37.6/16.3/22.8 & 50.6/33.6/40.4 & 46.4/9.5/15.8 & 50.9/16.2/24.6 & 65.5/33.2/44.1 \\
    \mlp & 10.9/2.4/3.9 & 18.3/5.7/8.7 & 24.6/11.0/15.2 & 13.1/2.0/3.4 & 24.4/5.8/9.4 & 28.1/10.0/14.8 & 23.7/2.1/3.8 & 29.1/4.8/8.2 & 44.1/10.9/17.4 \\
    \rcsls & 24.6/4.2/7.2 & 26.1/8.0/12.2 & 37.3/23.4/28.8 & 28.3/3.7/6.6 & 36.4/7.9/12.9 & 47.1/23.1/31.0 & 27.2/2.3/4.2 & 35.6/6.3/10.7 & 53.0/20.6/29.6 \\
    \pr & 32.3/12.1/17.6 & 40.7/22.1/28.6 & 47.9/42.7/45.2 & 33.8/9.8/15.2 & 45.6/19.8/27.6 & 54.9/38.5/45.2 & 48.2/9.9/16.4 & 56.5/18.2/27.6 & 69.5/36.4/47.8 \\
    \revise{\ugw} & \revise{10.9/7.0/8.5} & \revise{10.7/6.9/8.3} & \revise{11.0/7.2/8.7} & \revise{11.7/6.2/8.1} & \revise{11.8/6.2/8.1} & \revise{11.7/6.2/8.1} & \revise{17.2/6.1/8.9} & \revise{17.1/6.0/8.9} & \revise{16.6/5.9/8.7} \\
    \ao & 4.4/4.4/4.4 & 20.9/20.9/20.9 & 25.5/25.5/25.5 & 3.6/3.6/3.6 & 14.2/14.2/14.2 & 32.3/32.3/32.3 & 3.2/3.2/3.2 & 34.7/34.7/34.7 & 40.4/40.4/40.4 \\
    \textbf{\mvh} & \textbf{58.4}/\textbf{72.6}/\textbf{64.8} & \textbf{57.7}/\textbf{78.9}/\textbf{66.6} & \textbf{54.7}/\textbf{73.4}/\textbf{62.7} & \textbf{68.2}/\textbf{79.4}/\textbf{73.3} & \textbf{68.5}/\textbf{81.0}/\textbf{74.3} & \textbf{66.6}/\textbf{77.9}/\textbf{71.8} & \textbf{80.1}/\textbf{83.3}/\textbf{81.7} & \textbf{79.9}/\textbf{81.6}/\textbf{80.7} & \textbf{78.9}/\textbf{78.7}/\textbf{78.8} \\
    \hline
    \end{tabular}
\end{table*}

\begin{table*}[!htbp]
    \centering
    \scriptsize
\setlength{\tabcolsep}{1.5ex} %
\renewcommand{\arraystretch}{1.1}%
    \caption{Vector linking on \scidocs{} (\gte{}$\leftrightarrow$\openai{}):
    {\normalfont each cell reports \textbf{precision/recall/F1} (\%). Best values per metric are \textbf{bolded}.}}
    \label{tab:single_scidocs_gte_openai}
    \begin{tabular}{l|ccc|ccc|ccc}
    \hline
    \multirow{2}{*}{Method} & \multicolumn{3}{c|}{Overlap .15} & \multicolumn{3}{c|}{Overlap .20} & \multicolumn{3}{c}{Overlap .30} \\
     & Seeds 15 & Seeds 20 & Seeds 30 & Seeds 15 & Seeds 20 & Seeds 30 & Seeds 15 & Seeds 20 & Seeds 30 \\
    \hline
    \linear & 1.8/1.2/1.4 & 2.6/2.7/2.7 & 5.2/6.5/5.7 & 1.9/1.3/1.6 & 2.8/2.6/2.7 & 5.9/8.3/6.8 & 1.5/1.8/1.6 & 2.8/4.4/3.4 & 6.3/11.9/8.2 \\
    \cca & 8.5/1.6/2.7 & 15.8/5.2/7.8 & 29.1/16.2/20.8 & 11.7/1.7/2.9 & 23.0/5.8/9.2 & 33.9/14.2/20.1 & 12.3/1.6/2.8 & 28.1/5.2/8.8 & 48.9/16.8/25.0 \\
    \mlp & 4.5/0.4/0.7 & 6.7/0.8/1.5 & 13.0/2.3/3.9 & 7.1/0.4/0.7 & 10.4/0.9/1.7 & 16.7/2.6/4.6 & 10.4/0.4/0.7 & 17.4/1.1/2.0 & 28.7/3.3/6.0 \\
    \rcsls & 14.0/0.6/1.1 & 13.4/1.5/2.7 & 22.8/5.9/9.4 & 13.6/0.4/0.8 & 14.3/1.0/1.9 & 26.2/4.9/8.2 & 14.7/0.4/0.8 & 18.8/1.1/2.2 & 36.0/5.5/9.6 \\
    \pr & 11.6/1.5/2.7 & 20.7/5.0/8.0 & 35.9/16.2/22.3 & 15.9/2.0/3.6 & 27.6/5.7/9.5 & 44.9/17.2/24.8 & 20.1/2.1/3.8 & 36.6/6.0/10.4 & 59.6/18.5/28.2 \\
    \revise{\ugw} & \revise{3.4/0.1/0.2} & \revise{3.4/0.1/0.2} & \revise{3.3/0.1/0.2} & \revise{3.2/0.1/0.2} & \revise{3.2/0.1/0.2} & \revise{3.2/0.1/0.2} & \revise{4.1/0.1/0.1} & \revise{4.1/0.1/0.1} & \revise{4.2/0.1/0.1} \\
    \ao & 0.4/0.4/0.4 & 0.6/0.6/0.6 & 20.3/20.3/20.3 & 1.1/1.1/1.1 & 0.5/0.5/0.5 & 2.4/2.4/2.4 & 0.2/0.2/0.2 & 0.3/0.3/0.3 & 0.5/0.5/0.5 \\
    \textbf{\mvh} & \textbf{60.4}/\textbf{84.3}/\textbf{70.4} & \textbf{61.5}/\textbf{82.1}/\textbf{70.3} & \textbf{63.9}/\textbf{79.1}/\textbf{70.6} & \textbf{69.9}/\textbf{84.4}/\textbf{76.5} & \textbf{69.7}/\textbf{83.2}/\textbf{75.9} & \textbf{71.1}/\textbf{82.3}/\textbf{76.3} & \textbf{82.2}/\textbf{84.5}/\textbf{83.3} & \textbf{82.2}/\textbf{83.7}/\textbf{82.9} & \textbf{82.4}/\textbf{83.2}/\textbf{82.8} \\
    \hline
    \end{tabular}
\end{table*}

\begin{table*}[!htbp]
    \centering
    \scriptsize
\setlength{\tabcolsep}{1.5ex} %
\renewcommand{\arraystretch}{1.1}%
    \caption{Vector linking on \arguana{} (\gte{}$\leftrightarrow$\openai{}):
    {\normalfont each cell reports \textbf{precision/recall/F1} (\%). Best values per metric are \textbf{bolded}.}}
    \label{tab:single_arguana_gte_openai}
    \begin{tabular}{l|ccc|ccc|ccc}
    \hline
    \multirow{2}{*}{Method} & \multicolumn{3}{c|}{Overlap .15} & \multicolumn{3}{c|}{Overlap .20} & \multicolumn{3}{c}{Overlap .30} \\
     & Seeds 15 & Seeds 20 & Seeds 30 & Seeds 15 & Seeds 20 & Seeds 30 & Seeds 15 & Seeds 20 & Seeds 30 \\
    \hline
    \linear & 2.2/1.2/1.5 & 2.9/2.5/2.7 & 5.6/6.1/5.9 & 2.2/1.6/1.8 & 3.2/3.2/3.2 & 6.4/8.4/7.2 & 2.5/2.3/2.4 & 3.4/5.0/4.1 & 6.4/11.6/8.2 \\
    \cca & 13.0/3.8/5.8 & 19.7/8.4/11.8 & 30.7/18.9/23.4 & 16.1/4.1/6.5 & 28.6/10.2/15.0 & 38.7/19.3/25.8 & 20.5/3.6/6.2 & 29.6/7.3/11.7 & 50.5/21.3/29.9 \\
    \mlp & 5.1/0.7/1.2 & 10.2/1.7/2.8 & 18.0/4.2/6.8 & 6.9/0.9/1.5 & 13.8/2.2/3.7 & 21.0/4.1/6.9 & 9.7/0.8/1.4 & 17.6/1.8/3.2 & 29.0/3.7/6.5 \\
    \rcsls & 15.3/1.7/3.1 & 16.8/3.8/6.2 & 25.9/12.4/16.8 & 17.0/1.4/2.6 & 18.0/3.6/6.0 & 33.3/12.7/18.4 & 17.4/1.1/2.2 & 21.1/2.7/4.8 & 40.1/10.6/16.7 \\
    \pr & 13.9/4.0/6.3 & 23.4/9.8/13.8 & 34.8/21.7/26.7 & 21.0/5.5/8.8 & 31.2/10.5/15.7 & 43.0/21.3/28.5 & 28.8/5.0/8.5 & 36.8/9.6/15.3 & 57.0/25.2/34.9 \\
    \revise{\ugw} & \revise{8.9/0.3/0.6} & \revise{8.9/0.3/0.6} & \revise{8.9/0.3/0.6} & \revise{2.1/0.1/0.1} & \revise{2.1/0.1/0.1} & \revise{2.1/0.1/0.1} & \revise{6.4/0.1/0.2} & \revise{6.4/0.1/0.2} & \revise{6.4/0.1/0.2} \\
    \ao & 1.3/1.3/1.3 & 11.8/11.8/11.8 & 26.5/26.5/26.5 & 1.7/1.7/1.7 & 31.5/31.5/31.5 & 31.4/31.4/31.4 & 1.4/1.4/1.4 & 43.0/43.0/43.0 & 42.8/42.8/42.8 \\
    \textbf{\mvh} & \textbf{52.2}/\textbf{81.0}/\textbf{63.4} & \textbf{53.6}/\textbf{81.0}/\textbf{64.5} & \textbf{53.1}/\textbf{80.0}/\textbf{63.8} & \textbf{62.6}/\textbf{82.1}/\textbf{71.0} & \textbf{64.2}/\textbf{81.2}/\textbf{71.7} & \textbf{62.8}/\textbf{80.1}/\textbf{70.4} & \textbf{76.7}/\textbf{81.7}/\textbf{79.1} & \textbf{77.1}/\textbf{80.3}/\textbf{78.7} & \textbf{76.5}/\textbf{79.5}/\textbf{78.0} \\
    \hline
    \end{tabular}
\end{table*}

\begin{table*}[!htbp]
    \centering
    \scriptsize
\setlength{\tabcolsep}{1.5ex} %
\renewcommand{\arraystretch}{1.1}%
    \caption{Vector linking on \scifact{} (\gte{}$\leftrightarrow$\openai{}):
    {\normalfont each cell reports \textbf{precision/recall/F1} (\%). Best values per metric are \textbf{bolded}.}}
    \label{tab:single_scifact_gte_openai}
    \begin{tabular}{l|ccc|ccc|ccc}
    \hline
    \multirow{2}{*}{Method} & \multicolumn{3}{c|}{Overlap .15} & \multicolumn{3}{c|}{Overlap .20} & \multicolumn{3}{c}{Overlap .30} \\
     & Seeds 15 & Seeds 20 & Seeds 30 & Seeds 15 & Seeds 20 & Seeds 30 & Seeds 15 & Seeds 20 & Seeds 30 \\
    \hline
    \linear & 3.1/4.2/3.5 & 5.4/7.7/6.4 & 10.1/14.3/11.8 & 3.2/4.9/3.9 & 4.9/8.9/6.3 & 9.2/16.8/11.9 & 2.2/5.0/3.1 & 5.3/11.8/7.3 & 11.0/23.9/15.1 \\
    \cca & 21.8/5.8/9.1 & 32.3/14.8/20.3 & 42.7/34.5/38.2 & 24.8/6.9/10.7 & 39.2/14.2/20.8 & 51.9/34.7/41.6 & 29.3/5.2/8.8 & 43.1/11.9/18.7 & 65.5/35.4/46.0 \\
    \mlp & 13.1/1.8/3.1 & 14.8/3.6/5.8 & 25.1/10.7/15.0 & 20.0/3.7/6.3 & 23.8/5.8/9.3 & 32.9/12.4/18.0 & 14.2/1.4/2.5 & 22.6/3.3/5.7 & 39.9/10.5/16.6 \\
    \rcsls & 27.3/3.5/6.1 & 28.1/7.3/11.5 & 39.4/21.9/28.2 & 25.2/2.4/4.3 & 28.3/5.3/8.9 & 42.0/18.1/25.3 & 28.6/2.1/3.8 & 34.3/5.0/8.8 & 54.5/17.6/26.6 \\
    \pr & 27.4/7.1/11.3 & 35.6/14.5/20.6 & 45.8/34.9/39.6 & 27.7/7.7/12.1 & 45.2/15.1/22.7 & 58.6/35.9/44.5 & 37.9/5.7/9.9 & 52.5/13.7/21.7 & 68.6/34.8/46.1 \\
    \revise{\ugw} & \revise{10.7/3.1/4.9} & \revise{10.3/3.0/4.7} & \revise{10.4/3.1/4.8} & \revise{11.6/2.6/4.3} & \revise{11.7/2.7/4.3} & \revise{11.8/2.7/4.4} & \revise{15.9/2.5/4.4} & \revise{15.9/2.5/4.4} & \revise{15.5/2.5/4.3} \\
    \ao & 19.6/19.6/19.6 & 15.3/15.3/15.3 & 26.6/26.6/26.6 & 3.2/3.2/3.2 & 15.5/15.5/15.5 & 31.4/31.4/31.4 & 5.6/5.6/5.6 & 36.2/36.2/36.2 & 44.6/44.6/44.6 \\
    \textbf{\mvh} & \textbf{61.5}/\textbf{85.0}/\textbf{71.4} & \textbf{61.5}/\textbf{87.2}/\textbf{72.1} & \textbf{60.6}/\textbf{84.7}/\textbf{70.6} & \textbf{70.7}/\textbf{86.0}/\textbf{77.6} & \textbf{70.2}/\textbf{86.9}/\textbf{77.7} & \textbf{69.8}/\textbf{84.1}/\textbf{76.3} & \textbf{83.2}/\textbf{89.1}/\textbf{86.0} & \textbf{83.0}/\textbf{88.7}/\textbf{85.8} & \textbf{82.2}/\textbf{89.1}/\textbf{85.5} \\
    \hline
    \end{tabular}
\end{table*}

\begin{table*}[!htbp]
    \centering
    \scriptsize
\setlength{\tabcolsep}{1.5ex} %
\renewcommand{\arraystretch}{1.1}%
    \caption{Vector linking on \fiqa{} (\gte{}$\leftrightarrow$\openai{}):
    {\normalfont each cell reports \textbf{precision/recall/F1} (\%). Best values per metric are \textbf{bolded}.}}
    \label{tab:single_fiqa_gte_openai}
    \begin{tabular}{l|ccc|ccc|ccc}
    \hline
    \multirow{2}{*}{Method} & \multicolumn{3}{c|}{Overlap .15} & \multicolumn{3}{c|}{Overlap .20} & \multicolumn{3}{c}{Overlap .30} \\
     & Seeds 15 & Seeds 20 & Seeds 30 & Seeds 15 & Seeds 20 & Seeds 30 & Seeds 15 & Seeds 20 & Seeds 30 \\
    \hline
    \linear & 2.3/0.3/0.5 & 2.1/0.8/1.2 & 3.9/2.9/3.3 & 3.1/0.4/0.7 & 3.2/0.9/1.4 & 4.2/3.0/3.5 & 2.6/0.6/1.0 & 3.2/1.5/2.1 & 4.8/4.8/4.8 \\
    \cca & 4.3/0.8/1.4 & 8.4/2.4/3.7 & 17.3/8.3/11.3 & 6.4/0.9/1.6 & 13.0/3.0/4.9 & 26.3/10.7/15.2 & 7.2/0.5/0.9 & 12.6/1.6/2.8 & 26.3/6.8/10.8 \\
    \mlp & 1.8/0.2/0.3 & 2.7/0.3/0.6 & 7.8/0.8/1.5 & 2.3/0.2/0.4 & 7.6/0.8/1.5 & 14.4/2.0/3.6 & 2.6/0.1/0.2 & 5.8/0.3/0.6 & 13.5/0.8/1.5 \\
    \rcsls & 3.1/0.2/0.4 & 4.8/0.8/1.3 & 11.8/3.7/5.6 & 3.0/0.2/0.4 & 5.1/0.6/1.1 & 14.4/3.3/5.4 & 3.8/0.2/0.3 & 6.7/0.6/1.2 & 20.6/3.8/6.4 \\
    \pr & 6.3/1.1/1.9 & 12.1/3.0/4.8 & 26.2/10.7/15.2 & 9.3/1.2/2.2 & 19.4/4.1/6.7 & 35.8/13.0/19.0 & 8.8/0.5/1.0 & 18.6/2.1/3.7 & 36.5/8.0/13.1 \\
    \revise{\ugw} & \revise{0.0/0.0/0.0} & \revise{0.0/0.0/0.0} & \revise{0.0/0.0/0.0} & \revise{0.0/0.0/0.0} & \revise{0.0/0.0/0.0} & \revise{0.0/0.0/0.0} & \revise{2.5/0.0/0.0} & \revise{2.5/0.0/0.0} & \revise{2.5/0.0/0.0} \\
    \ao & 0.3/0.3/0.3 & 0.2/0.2/0.2 & 0.4/0.4/0.4 & 0.4/0.4/0.4 & 0.2/0.2/0.2 & 0.4/0.4/0.4 & 0.1/0.1/0.1 & 0.2/0.2/0.2 & 0.3/0.3/0.3 \\
    \textbf{\mvh} & \textbf{51.9}/\textbf{67.6}/\textbf{58.7} & \textbf{53.1}/\textbf{68.3}/\textbf{59.8} & \textbf{53.7}/\textbf{67.1}/\textbf{59.7} & \textbf{61.0}/\textbf{65.7}/\textbf{63.2} & \textbf{61.8}/\textbf{66.7}/\textbf{64.1} & \textbf{62.4}/\textbf{66.6}/\textbf{64.4} & \textbf{75.4}/\textbf{68.6}/\textbf{71.8} & \textbf{75.9}/\textbf{68.1}/\textbf{71.8} & \textbf{75.4}/\textbf{67.5}/\textbf{71.2} \\
    \hline
    \end{tabular}
\end{table*}

\begin{table*}[!htbp]
    \centering
    \scriptsize
\setlength{\tabcolsep}{1.5ex} %
\renewcommand{\arraystretch}{1.1}%
    \caption{Vector linking on \nfcorpus{} (\mistral{}$\leftrightarrow$\openai{}):
    {\normalfont each cell reports \textbf{precision/recall/F1} (\%). Best values per metric are \textbf{bolded}.}}
    \label{tab:single_nfcorpus_mistral_openai}
    \begin{tabular}{l|ccc|ccc|ccc}
    \hline
    \multirow{2}{*}{Method} & \multicolumn{3}{c|}{Overlap .15} & \multicolumn{3}{c|}{Overlap .20} & \multicolumn{3}{c}{Overlap .30} \\
     & Seeds 15 & Seeds 20 & Seeds 30 & Seeds 15 & Seeds 20 & Seeds 30 & Seeds 15 & Seeds 20 & Seeds 30 \\
    \hline
    \linear & 4.3/6.6/5.2 & 5.6/10.7/7.3 & 10.2/16.0/12.5 & 2.9/5.6/3.8 & 4.6/10.7/6.5 & 10.1/21.0/13.6 & 4.1/9.1/5.7 & 5.7/15.0/8.2 & 11.4/25.9/15.9 \\
    \cca & 18.6/7.5/10.7 & 28.3/15.9/20.3 & 39.4/34.9/37.0 & 22.7/6.3/9.9 & 37.1/14.8/21.2 & 46.4/31.3/37.4 & 34.1/6.7/11.2 & 50.3/15.8/24.1 & 58.8/30.4/40.1 \\
    \mlp & 14.5/2.1/3.6 & 21.1/5.7/9.0 & 28.0/10.9/15.7 & 18.3/1.5/2.8 & 26.6/4.1/7.1 & 40.4/15.3/22.3 & 29.8/2.0/3.7 & 33.4/4.7/8.3 & 48.1/10.3/17.0 \\
    \rcsls & 53.8/3.7/6.9 & 42.5/6.1/10.7 & 43.4/17.0/24.4 & 42.0/2.7/5.0 & 45.5/5.7/10.1 & 52.6/16.4/25.0 & 46.9/2.0/3.8 & 47.6/4.3/7.8 & 55.4/13.0/21.0 \\
    \pr & 29.6/10.4/15.4 & 35.1/18.5/24.2 & 46.3/40.2/43.0 & 31.1/9.0/13.9 & 41.0/15.7/22.7 & 51.5/33.9/40.9 & 41.7/7.9/13.3 & 58.5/18.4/28.0 & 71.4/37.3/49.0 \\
    \revise{\ugw} & \revise{3.5/1.7/2.3} & \revise{3.5/1.7/2.3} & \revise{3.5/1.7/2.3} & \revise{0.5/0.1/0.2} & \revise{0.5/0.1/0.2} & \revise{0.5/0.1/0.2} & \revise{9.9/2.1/3.5} & \revise{9.6/2.1/3.4} & \revise{9.6/2.1/3.4} \\
    \ao & 4.8/4.8/4.8 & 17.1/17.1/17.1 & 26.2/26.2/26.2 & 4.4/4.4/4.4 & 13.1/13.1/13.1 & 32.5/32.5/32.5 & 2.5/2.5/2.5 & 41.0/41.0/41.0 & 41.9/41.9/41.9 \\
    \textbf{\mvh} & \textbf{57.6}/\textbf{74.8}/\textbf{65.1} & \textbf{57.5}/\textbf{75.7}/\textbf{65.4} & \textbf{55.4}/\textbf{76.1}/\textbf{64.1} & \textbf{67.3}/\textbf{77.4}/\textbf{72.0} & \textbf{66.8}/\textbf{76.2}/\textbf{71.2} & \textbf{66.2}/\textbf{77.5}/\textbf{71.4} & \textbf{79.1}/\textbf{77.0}/\textbf{78.0} & \textbf{79.5}/\textbf{78.0}/\textbf{78.7} & \textbf{78.3}/\textbf{78.2}/\textbf{78.3} \\
    \hline
    \end{tabular}
\end{table*}

\begin{table*}[!htbp]
    \centering
    \scriptsize
\setlength{\tabcolsep}{1.5ex} %
\renewcommand{\arraystretch}{1.1}%
    \caption{Vector linking on \scidocs{} (\mistral{}$\leftrightarrow$\openai{}):
    {\normalfont each cell reports \textbf{precision/recall/F1} (\%). Best values per metric are \textbf{bolded}.}}
    \label{tab:single_scidocs_mistral_openai}
    \begin{tabular}{l|ccc|ccc|ccc}
    \hline
    \multirow{2}{*}{Method} & \multicolumn{3}{c|}{Overlap .15} & \multicolumn{3}{c|}{Overlap .20} & \multicolumn{3}{c}{Overlap .30} \\
     & Seeds 15 & Seeds 20 & Seeds 30 & Seeds 15 & Seeds 20 & Seeds 30 & Seeds 15 & Seeds 20 & Seeds 30 \\
    \hline
    \linear & 3.1/1.6/2.1 & 4.4/3.6/3.9 & 7.3/7.2/7.2 & 2.8/1.6/2.0 & 4.2/3.1/3.6 & 7.7/9.0/8.3 & 2.2/1.7/1.9 & 4.1/4.9/4.4 & 7.5/11.9/9.2 \\
    \cca & 7.3/1.8/2.8 & 13.1/4.9/7.2 & 21.0/12.5/15.7 & 8.5/1.3/2.3 & 13.8/3.7/5.9 & 27.7/12.2/16.9 & 11.0/1.4/2.4 & 19.8/3.9/6.5 & 39.1/13.6/20.1 \\
    \mlp & 4.2/0.3/0.6 & 6.5/0.8/1.5 & 16.6/3.3/5.5 & 12.2/0.4/0.8 & 14.3/1.1/2.0 & 20.2/3.8/6.5 & 10.4/0.3/0.5 & 18.5/0.7/1.4 & 29.3/2.8/5.1 \\
    \rcsls & 34.6/0.5/0.9 & 23.6/1.0/2.0 & 28.1/3.8/6.7 & 38.2/0.3/0.7 & 25.4/0.6/1.2 & 30.4/2.9/5.2 & 32.6/0.3/0.6 & 31.1/0.7/1.4 & 39.2/2.9/5.3 \\
    \pr & 11.8/2.0/3.5 & 16.3/4.9/7.5 & 29.8/14.3/19.3 & 11.6/1.6/2.7 & 20.7/4.7/7.7 & 37.5/14.9/21.3 & 15.6/1.6/2.9 & 28.4/4.8/8.2 & 51.6/15.6/24.0 \\
    \revise{\ugw} & \revise{0.9/0.0/0.1} & \revise{0.9/0.0/0.1} & \revise{0.9/0.0/0.1} & \revise{0.8/0.0/0.0} & \revise{0.8/0.0/0.0} & \revise{0.8/0.0/0.0} & \revise{3.8/0.1/0.1} & \revise{3.8/0.1/0.1} & \revise{3.8/0.1/0.1} \\
    \ao & 0.4/0.4/0.4 & 0.6/0.6/0.6 & 17.7/17.7/17.7 & 5.1/5.1/5.1 & 0.6/0.6/0.6 & 1.3/1.3/1.3 & 0.2/0.2/0.2 & 0.3/0.3/0.3 & 0.5/0.5/0.5 \\
    \textbf{\mvh} & \textbf{62.1}/\textbf{81.7}/\textbf{70.5} & \textbf{63.2}/\textbf{77.6}/\textbf{69.6} & \textbf{64.9}/\textbf{72.6}/\textbf{68.5} & \textbf{71.8}/\textbf{82.2}/\textbf{76.6} & \textbf{72.2}/\textbf{79.0}/\textbf{75.4} & \textbf{73.4}/\textbf{75.4}/\textbf{74.4} & \textbf{82.8}/\textbf{81.6}/\textbf{82.2} & \textbf{83.6}/\textbf{81.2}/\textbf{82.4} & \textbf{83.6}/\textbf{79.4}/\textbf{81.5} \\
    \hline
    \end{tabular}
\end{table*}

\begin{table*}[!htbp]
    \centering
    \scriptsize
\setlength{\tabcolsep}{1.5ex} %
\renewcommand{\arraystretch}{1.1}%
    \caption{Vector linking on \arguana{} (\mistral{}$\leftrightarrow$\openai{}):
    {\normalfont each cell reports \textbf{precision/recall/F1} (\%). Best values per metric are \textbf{bolded}.}}
    \label{tab:single_arguana_mistral_openai}
    \begin{tabular}{l|ccc|ccc|ccc}
    \hline
    \multirow{2}{*}{Method} & \multicolumn{3}{c|}{Overlap .15} & \multicolumn{3}{c|}{Overlap .20} & \multicolumn{3}{c}{Overlap .30} \\
     & Seeds 15 & Seeds 20 & Seeds 30 & Seeds 15 & Seeds 20 & Seeds 30 & Seeds 15 & Seeds 20 & Seeds 30 \\
    \hline
    \linear & 2.9/3.0/3.0 & 4.4/6.7/5.3 & 9.5/13.0/11.0 & 3.6/4.4/4.0 & 4.8/8.4/6.1 & 9.6/17.4/12.3 & 3.2/6.6/4.3 & 4.4/11.0/6.3 & 8.9/22.6/12.8 \\
    \cca & 21.8/7.6/11.3 & 31.4/15.1/20.4 & 43.0/34.2/38.1 & 28.0/8.4/13.0 & 40.0/17.1/23.9 & 51.9/36.1/42.6 & 37.5/7.0/11.7 & 49.8/15.6/23.7 & 65.2/36.9/47.2 \\
    \mlp & 10.5/0.8/1.5 & 23.9/2.6/4.6 & 16.3/1.7/3.0 & 19.0/1.0/1.9 & 22.2/1.5/2.8 & 31.5/4.5/7.9 & 23.7/0.7/1.3 & 33.4/1.8/3.5 & 50.9/8.1/13.9 \\
    \rcsls & 36.5/1.7/3.3 & 33.6/3.2/5.8 & 40.7/9.8/15.8 & 38.0/1.4/2.7 & 37.7/3.0/5.6 & 47.5/9.7/16.1 & 37.8/1.2/2.3 & 44.0/2.6/5.0 & 54.2/7.8/13.7 \\
    \pr & 25.7/8.7/13.0 & 36.2/17.5/23.6 & 47.7/36.2/41.1 & 33.1/10.0/15.3 & 46.6/19.8/27.8 & 59.1/40.4/48.0 & 42.7/8.3/13.9 & 59.7/19.6/29.5 & 71.8/42.4/53.3 \\
    \revise{\ugw} & \revise{9.4/0.6/1.2} & \revise{9.2/0.6/1.2} & \revise{10.2/0.7/1.3} & \revise{9.1/0.5/0.9} & \revise{9.2/0.5/0.9} & \revise{9.1/0.5/0.9} & \revise{14.0/0.5/0.9} & \revise{14.3/0.5/0.9} & \revise{14.1/0.5/0.9} \\
    \ao & 1.3/1.3/1.3 & 27.8/27.8/27.8 & 28.1/28.1/28.1 & 1.8/1.8/1.8 & 34.2/34.2/34.2 & 34.3/34.3/34.3 & 2.2/2.2/2.2 & 45.8/45.8/45.8 & 46.2/46.2/46.2 \\
    \textbf{\mvh} & \textbf{58.1}/\textbf{96.0}/\textbf{72.4} & \textbf{58.0}/\textbf{95.5}/\textbf{72.2} & \textbf{57.4}/\textbf{95.3}/\textbf{71.7} & \textbf{68.0}/\textbf{96.3}/\textbf{79.7} & \textbf{67.7}/\textbf{95.9}/\textbf{79.4} & \textbf{67.5}/\textbf{95.7}/\textbf{79.2} & \textbf{80.4}/\textbf{96.1}/\textbf{87.6} & \textbf{80.8}/\textbf{95.9}/\textbf{87.7} & \textbf{80.2}/\textbf{95.9}/\textbf{87.4} \\
    \hline
    \end{tabular}
\end{table*}

\begin{table*}[!htbp]
    \centering
    \scriptsize
\setlength{\tabcolsep}{1.5ex} %
\renewcommand{\arraystretch}{1.1}%
    \caption{Vector linking on \fiqa{} (\mistral{}$\leftrightarrow$\openai{}):
    {\normalfont each cell reports \textbf{precision/recall/F1} (\%). Best values per metric are \textbf{bolded}.}}
    \label{tab:single_fiqa_mistral_openai}
    \begin{tabular}{l|ccc|ccc|ccc}
    \hline
    \multirow{2}{*}{Method} & \multicolumn{3}{c|}{Overlap .15} & \multicolumn{3}{c|}{Overlap .20} & \multicolumn{3}{c}{Overlap .30} \\
     & Seeds 15 & Seeds 20 & Seeds 30 & Seeds 15 & Seeds 20 & Seeds 30 & Seeds 15 & Seeds 20 & Seeds 30 \\
    \hline
    \linear & 1.7/0.9/1.2 & 2.8/2.6/2.7 & 5.7/7.6/6.5 & 1.7/1.0/1.3 & 3.0/3.3/3.1 & 5.3/9.0/6.7 & 1.8/2.0/1.9 & 3.1/4.6/3.7 & 7.0/12.6/9.0 \\
    \cca & 9.8/2.1/3.4 & 18.0/6.1/9.1 & 31.9/19.6/24.3 & 15.4/2.3/4.0 & 26.0/7.3/11.4 & 42.5/22.7/29.6 & 16.0/1.3/2.5 & 29.2/5.0/8.6 & 46.7/16.4/24.3 \\
    \mlp & 4.1/0.1/0.2 & 9.1/0.4/0.7 & 9.2/0.4/0.8 & 6.6/0.1/0.3 & 9.6/0.2/0.4 & 18.7/1.1/2.0 & 9.0/0.1/0.2 & 9.6/0.1/0.3 & 16.3/0.5/1.0 \\
    \rcsls & 19.2/0.2/0.5 & 18.1/0.6/1.2 & 29.7/4.0/7.1 & 20.9/0.2/0.4 & 22.3/0.6/1.1 & 35.7/3.5/6.3 & 24.7/0.2/0.4 & 27.4/0.6/1.2 & 46.5/4.0/7.3 \\
    \pr & 15.4/2.8/4.8 & 23.6/6.8/10.6 & 42.1/21.3/28.3 & 21.0/2.9/5.2 & 34.5/8.2/13.2 & 54.0/24.9/34.0 & 22.1/1.8/3.2 & 36.5/5.5/9.6 & 59.9/19.0/28.9 \\
    \revise{\ugw} & \revise{1.3/0.0/0.0} & \revise{1.3/0.0/0.0} & \revise{1.3/0.0/0.0} & \revise{2.1/0.0/0.1} & \revise{2.1/0.0/0.1} & \revise{2.1/0.0/0.1} & \revise{4.1/0.0/0.1} & \revise{4.1/0.0/0.1} & \revise{4.1/0.0/0.1} \\
    \ao & 0.2/0.2/0.2 & 0.3/0.3/0.3 & 0.4/0.4/0.4 & 0.1/0.1/0.1 & 0.2/0.2/0.2 & 0.3/0.3/0.3 & 0.1/0.1/0.1 & 44.7/44.7/44.7 & 0.4/0.4/0.4 \\
    \textbf{\mvh} & \textbf{60.2}/\textbf{92.8}/\textbf{73.1} & \textbf{60.7}/\textbf{92.8}/\textbf{73.4} & \textbf{60.6}/\textbf{93.2}/\textbf{73.5} & \textbf{70.1}/\textbf{93.8}/\textbf{80.3} & \textbf{70.4}/\textbf{93.7}/\textbf{80.4} & \textbf{70.0}/\textbf{93.7}/\textbf{80.2} & \textbf{82.9}/\textbf{94.2}/\textbf{88.2} & \textbf{82.9}/\textbf{94.1}/\textbf{88.1} & \textbf{82.7}/\textbf{94.0}/\textbf{88.0} \\
    \hline
    \end{tabular}
\end{table*}

\begin{table*}[!htbp]
    \centering
    \scriptsize
\setlength{\tabcolsep}{1.5ex} %
\renewcommand{\arraystretch}{1.1}%
    \caption{Vector linking on \nfcorpus{} (\qwen{}$\leftrightarrow$\kalm{}):
    {\normalfont each cell reports \textbf{precision/recall/F1} (\%). Best values per metric are \textbf{bolded}.}}
    \label{tab:single_nfcorpus_qwen_kalm}
    \begin{tabular}{l|ccc|ccc|ccc}
    \hline
    \multirow{2}{*}{Method} & \multicolumn{3}{c|}{Overlap .15} & \multicolumn{3}{c|}{Overlap .20} & \multicolumn{3}{c}{Overlap .30} \\
     & Seeds 15 & Seeds 20 & Seeds 30 & Seeds 15 & Seeds 20 & Seeds 30 & Seeds 15 & Seeds 20 & Seeds 30 \\
    \hline
    \linear & 2.1/4.2/2.8 & 3.4/7.9/4.7 & 7.8/13.3/9.8 & 2.4/5.1/3.3 & 3.3/8.8/4.8 & 8.3/17.7/11.3 & 3.0/7.4/4.3 & 4.1/12.2/6.1 & 9.7/22.6/13.6 \\
    \cca & 26.4/7.9/12.2 & 37.4/18.9/25.1 & 44.2/37.3/40.5 & 29.4/7.4/11.9 & 44.0/15.6/23.0 & 57.2/35.4/43.8 & 39.0/6.8/11.6 & 49.3/13.9/21.7 & 70.9/34.7/46.5 \\
    \mlp & 18.2/2.6/4.6 & 25.0/5.3/8.8 & 32.9/9.5/14.8 & 24.2/3.2/5.7 & 31.7/5.5/9.4 & 33.6/6.3/10.6 & 32.0/2.2/4.2 & 36.8/5.0/8.7 & 51.8/9.6/16.2 \\
    \rcsls & 30.8/4.2/7.4 & 34.1/8.4/13.5 & 39.6/22.9/29.0 & 36.2/4.0/7.2 & 43.1/8.7/14.4 & 48.6/21.9/30.1 & 36.4/2.8/5.2 & 44.3/6.1/10.7 & 57.0/19.4/29.0 \\
    \pr & 29.6/9.6/14.5 & 43.8/21.3/28.7 & 48.6/41.4/44.7 & 34.1/8.6/13.7 & 49.5/19.4/27.8 & 60.2/37.7/46.4 & 48.9/8.6/14.6 & 55.5/18.1/27.2 & 72.8/39.2/51.0 \\
    \revise{\ugw} & \revise{8.0/1.5/2.5} & \revise{8.0/1.5/2.6} & \revise{7.4/1.4/2.3} & \revise{6.9/1.0/1.7} & \revise{6.9/1.0/1.7} & \revise{7.0/1.0/1.8} & \revise{13.3/1.1/2.1} & \revise{13.0/1.1/2.1} & \revise{10.2/0.8/1.6} \\
    \ao & 6.2/6.2/6.2 & 24.6/24.6/24.6 & 27.2/27.2/27.2 & 6.1/6.1/6.1 & 21.0/21.0/21.0 & 32.9/32.9/32.9 & 42.3/42.3/42.3 & 42.1/42.1/42.1 & 43.0/43.0/43.0 \\
    \textbf{\mvh} & \textbf{59.1}/\textbf{89.4}/\textbf{71.2} & \textbf{59.1}/\textbf{90.3}/\textbf{71.4} & \textbf{57.6}/\textbf{89.7}/\textbf{70.2} & \textbf{67.7}/\textbf{87.9}/\textbf{76.5} & \textbf{65.7}/\textbf{87.6}/\textbf{75.1} & \textbf{65.1}/\textbf{86.4}/\textbf{74.2} & \textbf{79.9}/\textbf{89.4}/\textbf{84.4} & \textbf{79.4}/\textbf{88.4}/\textbf{83.7} & \textbf{79.7}/\textbf{88.6}/\textbf{83.9} \\
    \hline
    \end{tabular}
\end{table*}

\begin{table*}[!htbp]
    \centering
    \scriptsize
\setlength{\tabcolsep}{1.5ex} %
\renewcommand{\arraystretch}{1.1}%
    \caption{Vector linking on \scidocs{} (\qwen{}$\leftrightarrow$\kalm{}):
    {\normalfont each cell reports \textbf{precision/recall/F1} (\%). Best values per metric are \textbf{bolded}.}}
    \label{tab:single_scidocs_qwen_kalm}
    \begin{tabular}{l|ccc|ccc|ccc}
    \hline
    \multirow{2}{*}{Method} & \multicolumn{3}{c|}{Overlap .15} & \multicolumn{3}{c|}{Overlap .20} & \multicolumn{3}{c}{Overlap .30} \\
     & Seeds 15 & Seeds 20 & Seeds 30 & Seeds 15 & Seeds 20 & Seeds 30 & Seeds 15 & Seeds 20 & Seeds 30 \\
    \hline
    \linear & 1.2/0.6/0.8 & 1.5/1.4/1.4 & 2.9/3.8/3.3 & 1.0/0.6/0.7 & 1.3/1.2/1.2 & 3.3/4.3/3.7 & 0.9/0.7/0.8 & 1.6/2.0/1.8 & 3.8/6.7/4.9 \\
    \cca & 11.6/2.0/3.5 & 20.9/6.4/9.8 & 34.7/19.0/24.5 & 12.9/1.8/3.2 & 23.4/5.2/8.6 & 38.9/16.0/22.6 & 15.8/2.0/3.5 & 30.2/4.9/8.5 & 52.3/17.2/25.9 \\
    \mlp & 8.9/0.4/0.7 & 12.8/0.7/1.3 & 21.4/1.9/3.6 & 13.9/0.3/0.6 & 20.6/0.7/1.4 & 25.9/1.2/2.2 & 14.8/0.4/0.7 & 22.5/0.8/1.6 & 23.2/0.7/1.4 \\
    \rcsls & 31.0/0.5/0.9 & 15.4/1.6/2.8 & 25.9/6.0/9.8 & 30.5/0.4/0.8 & 19.7/1.0/1.8 & 29.9/5.2/8.8 & 20.6/0.4/0.8 & 22.8/1.3/2.5 & 39.9/5.8/10.1 \\
    \pr & 16.8/2.4/4.2 & 27.8/7.5/11.8 & 44.8/22.1/29.6 & 17.6/2.5/4.4 & 30.9/6.3/10.5 & 51.5/20.9/29.7 & 24.9/2.7/4.8 & 41.3/6.4/11.1 & 63.5/20.3/30.8 \\
    \revise{\ugw} & \revise{0.0/0.0/0.0} & \revise{0.0/0.0/0.0} & \revise{0.0/0.0/0.0} & \revise{0.0/0.0/0.0} & \revise{0.0/0.0/0.0} & \revise{0.0/0.0/0.0} & \revise{0.0/0.0/0.0} & \revise{0.0/0.0/0.0} & \revise{0.0/0.0/0.0} \\
    \ao & 0.6/0.6/0.6 & 0.6/0.6/0.6 & 23.6/23.6/23.6 & 0.3/0.3/0.3 & 0.5/0.5/0.5 & 30.7/30.7/30.7 & 0.2/0.2/0.2 & 0.3/0.3/0.3 & 0.4/0.4/0.4 \\
    \textbf{\mvh} & \textbf{61.4}/\textbf{87.0}/\textbf{72.0} & \textbf{62.0}/\textbf{87.6}/\textbf{72.6} & \textbf{63.2}/\textbf{87.0}/\textbf{73.2} & \textbf{70.1}/\textbf{89.1}/\textbf{78.4} & \textbf{71.6}/\textbf{88.9}/\textbf{79.3} & \textbf{71.8}/\textbf{87.7}/\textbf{78.9} & \textbf{82.6}/\textbf{89.4}/\textbf{85.9} & \textbf{82.7}/\textbf{88.8}/\textbf{85.7} & \textbf{83.2}/\textbf{89.0}/\textbf{86.0} \\
    \hline
    \end{tabular}
\end{table*}

\begin{table*}[!htbp]
    \centering
    \scriptsize
\setlength{\tabcolsep}{1.5ex} %
\renewcommand{\arraystretch}{1.1}%
    \caption{Vector linking on \arguana{} (\qwen{}$\leftrightarrow$\kalm{}):
    {\normalfont each cell reports \textbf{precision/recall/F1} (\%). Best values per metric are \textbf{bolded}.}}
    \label{tab:single_arguana_qwen_kalm}
    \begin{tabular}{l|ccc|ccc|ccc}
    \hline
    \multirow{2}{*}{Method} & \multicolumn{3}{c|}{Overlap .15} & \multicolumn{3}{c|}{Overlap .20} & \multicolumn{3}{c}{Overlap .30} \\
     & Seeds 15 & Seeds 20 & Seeds 30 & Seeds 15 & Seeds 20 & Seeds 30 & Seeds 15 & Seeds 20 & Seeds 30 \\
    \hline
    \linear & 1.4/1.4/1.4 & 2.7/2.6/2.7 & 5.3/7.1/6.1 & 1.4/1.6/1.5 & 2.6/2.9/2.7 & 5.6/7.6/6.4 & 1.7/2.2/1.9 & 2.8/4.9/3.5 & 6.2/11.3/8.0 \\
    \cca & 14.1/3.8/6.0 & 21.1/8.1/11.7 & 34.9/21.8/26.8 & 21.5/5.7/9.0 & 33.4/12.1/17.8 & 49.2/29.7/37.0 & 22.6/3.4/5.9 & 32.6/7.6/12.3 & 55.4/24.5/34.0 \\
    \mlp & 15.0/1.2/2.2 & 14.2/1.4/2.5 & 13.6/1.5/2.7 & 11.8/0.9/1.7 & 18.9/2.2/3.9 & 22.3/2.9/5.1 & 17.1/1.0/1.9 & 24.2/1.7/3.2 & 36.6/4.8/8.5 \\
    \rcsls & 21.6/2.1/3.7 & 20.6/4.3/7.1 & 32.1/16.5/21.8 & 24.0/1.6/2.9 & 22.5/4.1/6.9 & 36.7/15.3/21.6 & 23.6/1.1/2.1 & 29.0/3.1/5.7 & 45.4/11.4/18.2 \\
    \pr & 17.1/4.8/7.5 & 28.8/11.2/16.2 & 42.0/27.3/33.1 & 27.9/7.4/11.8 & 42.2/16.3/23.5 & 56.0/33.9/42.2 & 28.8/4.9/8.4 & 40.9/9.6/15.6 & 64.1/31.5/42.3 \\
    \revise{\ugw} & \revise{2.1/0.1/0.1} & \revise{2.1/0.1/0.2} & \revise{2.1/0.1/0.2} & \revise{1.9/0.1/0.1} & \revise{1.9/0.1/0.1} & \revise{1.9/0.1/0.1} & \revise{4.4/0.1/0.2} & \revise{4.4/0.1/0.2} & \revise{4.4/0.1/0.2} \\
    \ao & 8.9/8.9/8.9 & 27.4/27.4/27.4 & 26.6/26.6/26.6 & 32.5/32.5/32.5 & 32.9/32.9/32.9 & 33.4/33.4/33.4 & 10.1/10.1/10.1 & 44.8/44.8/44.8 & 44.7/44.7/44.7 \\
    \textbf{\mvh} & \textbf{53.8}/\textbf{81.1}/\textbf{64.7} & \textbf{54.0}/\textbf{79.9}/\textbf{64.4} & \textbf{53.4}/\textbf{79.6}/\textbf{63.9} & \textbf{64.3}/\textbf{81.0}/\textbf{71.7} & \textbf{64.7}/\textbf{79.8}/\textbf{71.4} & \textbf{63.4}/\textbf{79.9}/\textbf{70.7} & \textbf{75.1}/\textbf{80.8}/\textbf{77.9} & \textbf{76.0}/\textbf{79.6}/\textbf{77.7} & \textbf{75.9}/\textbf{80.0}/\textbf{77.9} \\
    \hline
    \end{tabular}
\end{table*}

\begin{table*}[!htbp]
    \centering
    \scriptsize
\setlength{\tabcolsep}{1.5ex} %
\renewcommand{\arraystretch}{1.1}%
    \caption{Vector linking on \scifact{} (\qwen{}$\leftrightarrow$\kalm{}):
    {\normalfont each cell reports \textbf{precision/recall/F1} (\%). Best values per metric are \textbf{bolded}.}}
    \label{tab:single_scifact_qwen_kalm}
    \begin{tabular}{l|ccc|ccc|ccc}
    \hline
    \multirow{2}{*}{Method} & \multicolumn{3}{c|}{Overlap .15} & \multicolumn{3}{c|}{Overlap .20} & \multicolumn{3}{c}{Overlap .30} \\
     & Seeds 15 & Seeds 20 & Seeds 30 & Seeds 15 & Seeds 20 & Seeds 30 & Seeds 15 & Seeds 20 & Seeds 30 \\
    \hline
    \linear & 3.4/2.4/2.8 & 5.0/5.4/5.2 & 7.8/9.9/8.7 & 2.4/3.1/2.7 & 3.6/5.7/4.4 & 6.1/13.2/8.3 & 2.1/4.1/2.8 & 4.0/9.1/5.5 & 7.9/18.5/11.1 \\
    \cca & 19.8/5.9/9.1 & 29.2/14.1/19.0 & 46.8/37.8/41.8 & 25.6/6.3/10.1 & 38.7/13.5/20.0 & 56.5/36.0/44.0 & 32.3/5.6/9.5 & 53.6/14.8/23.3 & 70.0/37.4/48.7 \\
    \mlp & 13.4/1.4/2.6 & 19.8/2.9/5.1 & 27.6/4.7/8.0 & 18.2/1.6/2.9 & 24.6/2.6/4.8 & 40.4/6.5/11.1 & 29.8/1.6/3.1 & 41.4/3.1/5.8 & 43.2/3.7/6.9 \\
    \rcsls & 27.4/3.1/5.7 & 34.3/7.2/11.9 & 38.0/17.4/23.8 & 33.0/2.3/4.3 & 36.4/4.9/8.6 & 48.0/15.3/23.2 & 32.7/1.9/3.6 & 40.8/4.3/7.9 & 54.3/14.5/22.9 \\
    \pr & 26.4/7.2/11.3 & 40.3/17.0/24.0 & 52.7/40.4/45.8 & 32.6/8.1/13.0 & 47.9/16.5/24.6 & 63.3/36.4/46.2 & 38.9/6.8/11.6 & 60.6/15.8/25.1 & 74.8/38.9/51.2 \\
    \revise{\ugw} & \revise{3.6/0.1/0.3} & \revise{3.6/0.1/0.3} & \revise{3.6/0.1/0.3} & \revise{6.2/0.2/0.4} & \revise{6.2/0.2/0.4} & \revise{6.5/0.2/0.4} & \revise{6.9/0.1/0.3} & \revise{6.9/0.1/0.3} & \revise{6.5/0.1/0.3} \\
    \ao & 22.0/22.0/22.0 & 24.3/24.3/24.3 & 27.4/27.4/27.4 & 7.5/7.5/7.5 & 9.6/9.6/9.6 & 32.6/32.6/32.6 & 12.9/12.9/12.9 & 5.3/5.3/5.3 & 45.3/45.3/45.3 \\
    \textbf{\mvh} & \textbf{62.0}/\textbf{96.5}/\textbf{75.5} & \textbf{61.7}/\textbf{96.0}/\textbf{75.1} & \textbf{59.7}/\textbf{96.1}/\textbf{73.7} & \textbf{71.4}/\textbf{95.7}/\textbf{81.8} & \textbf{71.4}/\textbf{95.7}/\textbf{81.8} & \textbf{70.2}/\textbf{95.8}/\textbf{81.1} & \textbf{83.2}/\textbf{96.0}/\textbf{89.2} & \textbf{83.2}/\textbf{95.8}/\textbf{89.1} & \textbf{82.1}/\textbf{95.9}/\textbf{88.5} \\
    \hline
    \end{tabular}
\end{table*}

\begin{table*}[!htbp]
    \centering
    \scriptsize
\setlength{\tabcolsep}{1.5ex} %
\renewcommand{\arraystretch}{1.1}%
    \caption{Vector linking on \fiqa{} (\qwen{}$\leftrightarrow$\kalm{}):
    {\normalfont each cell reports \textbf{precision/recall/F1} (\%). Best values per metric are \textbf{bolded}.}}
    \label{tab:single_fiqa_qwen_kalm}
    \begin{tabular}{l|ccc|ccc|ccc}
    \hline
    \multirow{2}{*}{Method} & \multicolumn{3}{c|}{Overlap .15} & \multicolumn{3}{c|}{Overlap .20} & \multicolumn{3}{c}{Overlap .30} \\
     & Seeds 15 & Seeds 20 & Seeds 30 & Seeds 15 & Seeds 20 & Seeds 30 & Seeds 15 & Seeds 20 & Seeds 30 \\
    \hline
    \linear & 0.5/0.4/0.4 & 0.9/1.1/1.0 & 2.2/3.8/2.8 & 0.4/0.6/0.5 & 0.8/1.4/1.0 & 2.3/4.4/3.0 & 0.7/0.8/0.8 & 1.2/2.2/1.6 & 3.2/6.8/4.3 \\
    \cca & 8.6/1.4/2.5 & 16.1/4.0/6.4 & 33.2/15.7/21.3 & 13.9/1.7/2.9 & 26.1/5.5/9.1 & 44.7/19.1/26.7 & 12.5/0.7/1.4 & 23.8/3.1/5.5 & 49.0/14.5/22.4 \\
    \mlp & 7.0/0.3/0.6 & 11.0/0.7/1.4 & 10.7/0.6/1.1 & 10.4/0.4/0.7 & 15.5/0.9/1.6 & 18.4/1.1/2.1 & 9.2/0.1/0.3 & 12.0/0.3/0.5 & 21.5/0.8/1.6 \\
    \rcsls & 11.9/0.3/0.6 & 13.5/1.0/1.9 & 25.4/5.8/9.4 & 14.1/0.2/0.4 & 14.5/0.8/1.5 & 28.2/5.0/8.6 & 15.8/0.2/0.4 & 18.9/0.8/1.6 & 40.4/6.3/10.9 \\
    \pr & 13.7/2.1/3.6 & 24.3/5.5/9.0 & 42.3/18.6/25.8 & 20.3/2.5/4.5 & 36.1/7.9/13.0 & 54.6/23.1/32.4 & 20.6/1.3/2.4 & 35.0/4.7/8.3 & 62.7/19.1/29.3 \\
    \revise{\ugw} & \revise{2.6/0.0/0.0} & \revise{2.6/0.0/0.0} & \revise{2.6/0.0/0.0} & \revise{1.3/0.0/0.0} & \revise{1.3/0.0/0.0} & \revise{1.3/0.0/0.0} & \revise{2.4/0.0/0.0} & \revise{2.4/0.0/0.0} & \revise{2.4/0.0/0.0} \\
    \ao & 0.2/0.2/0.2 & 0.2/0.2/0.2 & 23.4/23.4/23.4 & 0.3/0.3/0.3 & 0.2/0.2/0.2 & 30.1/30.1/30.1 & 0.1/0.1/0.1 & 0.1/0.1/0.1 & 0.3/0.3/0.3 \\
    \textbf{\mvh} & \textbf{55.7}/\textbf{75.9}/\textbf{64.2} & \textbf{57.8}/\textbf{74.6}/\textbf{65.1} & \textbf{57.8}/\textbf{73.6}/\textbf{64.8} & \textbf{66.5}/\textbf{76.9}/\textbf{71.3} & \textbf{67.0}/\textbf{76.2}/\textbf{71.3} & \textbf{66.7}/\textbf{72.1}/\textbf{69.3} & \textbf{79.8}/\textbf{79.9}/\textbf{79.8} & \textbf{79.3}/\textbf{76.6}/\textbf{77.9} & \textbf{79.2}/\textbf{75.7}/\textbf{77.4} \\
    \hline
    \end{tabular}
\end{table*}

\begin{table*}[!htbp]
    \centering
    \scriptsize
\setlength{\tabcolsep}{1.5ex} %
\renewcommand{\arraystretch}{1.1}%
    \caption{Vector linking on \nfcorpus{} (\qwen{}$\leftrightarrow$\openai{}):
    {\normalfont each cell reports \textbf{precision/recall/F1} (\%). Best values per metric are \textbf{bolded}.}}
    \label{tab:single_nfcorpus_qwen_openai}
    \begin{tabular}{l|ccc|ccc|ccc}
    \hline
    \multirow{2}{*}{Method} & \multicolumn{3}{c|}{Overlap .15} & \multicolumn{3}{c|}{Overlap .20} & \multicolumn{3}{c}{Overlap .30} \\
     & Seeds 15 & Seeds 20 & Seeds 30 & Seeds 15 & Seeds 20 & Seeds 30 & Seeds 15 & Seeds 20 & Seeds 30 \\
    \hline
    \linear & 2.3/4.4/3.0 & 3.7/8.8/5.2 & 9.4/15.7/11.8 & 2.1/4.6/2.9 & 3.5/9.8/5.2 & 9.4/18.8/12.6 & 2.8/7.2/4.0 & 5.4/13.6/7.7 & 10.8/25.4/15.2 \\
    \cca & 31.5/11.7/17.1 & 40.2/21.9/28.4 & 46.7/42.9/44.7 & 35.8/10.2/15.9 & 48.7/21.1/29.4 & 59.2/43.8/50.3 & 46.7/10.6/17.3 & 59.9/19.6/29.5 & 73.7/43.8/54.9 \\
    \mlp & 16.4/3.2/5.4 & 24.8/7.6/11.7 & 30.1/13.4/18.6 & 15.7/2.2/3.9 & 26.7/6.1/9.9 & 41.4/15.6/22.7 & 36.1/3.3/6.0 & 31.9/5.0/8.6 & 51.0/9.9/16.6 \\
    \rcsls & 32.0/5.5/9.3 & 39.9/10.5/16.7 & 41.8/27.6/33.2 & 35.8/4.3/7.7 & 40.3/8.8/14.5 & 49.9/24.7/33.1 & 38.9/3.0/5.6 & 43.6/7.2/12.4 & 59.7/23.5/33.7 \\
    \pr & 36.8/15.5/21.8 & 43.8/24.0/31.0 & 51.4/49.1/50.2 & 44.5/13.8/21.0 & 52.3/23.6/32.6 & 60.4/47.5/53.2 & 52.5/11.8/19.3 & 62.2/22.8/33.4 & 75.0/45.5/56.6 \\
    \revise{\ugw} & \revise{11.7/3.6/5.5} & \revise{11.7/3.6/5.5} & \revise{12.3/3.7/5.7} & \revise{14.5/3.5/5.6} & \revise{14.5/3.5/5.7} & \revise{14.5/3.6/5.8} & \revise{14.8/2.4/4.2} & \revise{13.7/2.2/3.9} & \revise{15.3/2.5/4.2} \\
    \ao & 12.8/12.8/12.8 & 23.1/23.1/23.1 & 27.9/27.9/27.9 & 13.4/13.4/13.4 & 31.5/31.5/31.5 & 33.3/33.3/33.3 & 22.5/22.5/22.5 & 43.9/43.9/43.9 & 44.5/44.5/44.5 \\
    \textbf{\mvh} & \textbf{61.9}/\textbf{93.2}/\textbf{74.4} & \textbf{59.8}/\textbf{93.5}/\textbf{73.0} & \textbf{58.3}/\textbf{95.5}/\textbf{72.4} & \textbf{68.7}/\textbf{94.1}/\textbf{79.4} & \textbf{68.9}/\textbf{94.3}/\textbf{79.6} & \textbf{68.4}/\textbf{95.1}/\textbf{79.5} & \textbf{82.1}/\textbf{95.6}/\textbf{88.3} & \textbf{82.4}/\textbf{95.2}/\textbf{88.3} & \textbf{81.0}/\textbf{94.8}/\textbf{87.4} \\
    \hline
    \end{tabular}
\end{table*}

\begin{table*}[!htbp]
    \centering
    \scriptsize
\setlength{\tabcolsep}{1.5ex} %
\renewcommand{\arraystretch}{1.1}%
    \caption{Vector linking on \scidocs{} (\qwen{}$\leftrightarrow$\openai{}):
    {\normalfont each cell reports \textbf{precision/recall/F1} (\%). Best values per metric are \textbf{bolded}.}}
    \label{tab:single_scidocs_qwen_openai}
    \begin{tabular}{l|ccc|ccc|ccc}
    \hline
    \multirow{2}{*}{Method} & \multicolumn{3}{c|}{Overlap .15} & \multicolumn{3}{c|}{Overlap .20} & \multicolumn{3}{c}{Overlap .30} \\
     & Seeds 15 & Seeds 20 & Seeds 30 & Seeds 15 & Seeds 20 & Seeds 30 & Seeds 15 & Seeds 20 & Seeds 30 \\
    \hline
    \linear & 1.2/0.9/1.0 & 2.0/2.0/2.0 & 4.3/5.1/4.7 & 1.2/0.7/0.9 & 1.7/1.5/1.6 & 4.3/6.0/5.0 & 1.0/1.0/1.0 & 2.0/2.8/2.3 & 4.8/8.8/6.2 \\
    \cca & 8.9/1.8/3.0 & 14.8/5.3/7.9 & 29.4/16.6/21.2 & 11.8/1.6/2.8 & 21.1/5.1/8.1 & 36.8/16.2/22.5 & 14.4/1.9/3.4 & 28.6/5.2/8.9 & 52.6/18.3/27.1 \\
    \mlp & 3.4/0.3/0.6 & 6.1/0.8/1.4 & 15.7/2.4/4.1 & 10.7/0.5/0.9 & 15.0/1.2/2.2 & 24.9/3.2/5.7 & 15.2/0.5/0.9 & 24.1/1.4/2.7 & 28.8/1.6/3.0 \\
    \rcsls & 14.6/0.8/1.5 & 16.0/2.0/3.5 & 27.2/7.9/12.2 & 15.1/0.5/1.0 & 16.7/1.2/2.2 & 32.7/6.9/11.4 & 18.2/0.6/1.1 & 24.7/1.7/3.1 & 43.4/7.8/13.2 \\
    \pr & 12.1/2.0/3.4 & 21.9/6.4/9.9 & 41.8/20.6/27.6 & 15.0/2.0/3.5 & 29.6/6.6/10.8 & 51.1/21.0/29.8 & 24.5/2.9/5.2 & 39.8/7.3/12.4 & 65.5/22.5/33.5 \\
    \revise{\ugw} & \revise{6.2/0.1/0.1} & \revise{6.5/0.1/0.1} & \revise{6.5/0.1/0.1} & \revise{0.0/0.0/0.0} & \revise{0.0/0.0/0.0} & \revise{0.0/0.0/0.0} & \revise{0.0/0.0/0.0} & \revise{0.0/0.0/0.0} & \revise{0.0/0.0/0.0} \\
    \ao & 0.5/0.5/0.5 & 8.1/8.1/8.1 & 24.4/24.4/24.4 & 10.6/10.6/10.6 & 0.7/0.7/0.7 & 0.8/0.8/0.8 & 0.3/0.3/0.3 & 0.3/0.3/0.3 & 0.6/0.6/0.6 \\
    \textbf{\mvh} & \textbf{63.2}/\textbf{90.9}/\textbf{74.6} & \textbf{64.2}/\textbf{91.3}/\textbf{75.4} & \textbf{63.7}/\textbf{90.4}/\textbf{74.7} & \textbf{71.9}/\textbf{90.9}/\textbf{80.3} & \textbf{72.0}/\textbf{90.6}/\textbf{80.3} & \textbf{72.6}/\textbf{90.6}/\textbf{80.6} & \textbf{83.0}/\textbf{91.9}/\textbf{87.2} & \textbf{83.3}/\textbf{91.0}/\textbf{87.0} & \textbf{83.7}/\textbf{91.1}/\textbf{87.2} \\
    \hline
    \end{tabular}
\end{table*}

\begin{table*}[!htbp]
    \centering
    \scriptsize
\setlength{\tabcolsep}{1.5ex} %
\renewcommand{\arraystretch}{1.1}%
    \caption{Vector linking on \arguana{} (\qwen{}$\leftrightarrow$\openai{}):
    {\normalfont each cell reports \textbf{precision/recall/F1} (\%). Best values per metric are \textbf{bolded}.}}
    \label{tab:single_arguana_qwen_openai}
    \begin{tabular}{l|ccc|ccc|ccc}
    \hline
    \multirow{2}{*}{Method} & \multicolumn{3}{c|}{Overlap .15} & \multicolumn{3}{c|}{Overlap .20} & \multicolumn{3}{c}{Overlap .30} \\
     & Seeds 15 & Seeds 20 & Seeds 30 & Seeds 15 & Seeds 20 & Seeds 30 & Seeds 15 & Seeds 20 & Seeds 30 \\
    \hline
    \linear & 1.4/1.8/1.6 & 3.2/4.4/3.7 & 7.1/9.3/8.0 & 1.9/2.7/2.2 & 3.0/5.3/3.8 & 7.7/12.8/9.6 & 1.9/3.1/2.4 & 3.1/7.1/4.3 & 7.1/15.3/9.7 \\
    \cca & 18.2/5.6/8.6 & 30.8/12.7/18.0 & 41.8/30.2/35.1 & 25.5/6.8/10.7 & 36.6/14.1/20.4 & 54.1/34.1/41.8 & 31.7/5.9/10.0 & 40.8/12.3/18.9 & 63.9/34.3/44.6 \\
    \mlp & 11.8/1.5/2.6 & 19.7/2.9/5.0 & 22.7/4.4/7.4 & 13.7/1.5/2.6 & 21.7/3.2/5.6 & 36.9/9.2/14.7 & 14.0/1.0/1.8 & 27.3/2.2/4.1 & 43.5/6.3/11.1 \\
    \rcsls & 21.7/2.4/4.3 & 25.6/5.2/8.6 & 36.1/18.2/24.2 & 22.2/2.2/3.9 & 27.6/5.1/8.7 & 46.4/18.0/26.0 & 27.1/1.4/2.7 & 36.7/4.9/8.6 & 55.1/15.6/24.3 \\
    \pr & 23.1/7.5/11.4 & 38.7/17.7/24.3 & 47.4/37.6/41.9 & 34.6/9.8/15.2 & 47.3/19.4/27.5 & 59.8/40.3/48.2 & 39.4/7.2/12.2 & 54.6/17.3/26.3 & 70.4/41.7/52.4 \\
    \revise{\ugw} & \revise{2.3/0.1/0.2} & \revise{2.3/0.1/0.2} & \revise{2.3/0.1/0.2} & \revise{0.0/0.0/0.0} & \revise{0.0/0.0/0.0} & \revise{0.0/0.0/0.0} & \revise{2.2/0.0/0.1} & \revise{2.2/0.0/0.1} & \revise{2.2/0.0/0.1} \\
    \ao & 3.3/3.3/3.3 & 14.9/14.9/14.9 & 28.6/28.6/28.6 & 1.7/1.7/1.7 & 22.8/22.8/22.8 & 34.3/34.3/34.3 & 20.5/20.5/20.5 & 22.5/22.5/22.5 & 46.1/46.1/46.1 \\
    \textbf{\mvh} & \textbf{55.2}/\textbf{87.3}/\textbf{67.6} & \textbf{56.2}/\textbf{86.0}/\textbf{68.0} & \textbf{55.3}/\textbf{85.8}/\textbf{67.3} & \textbf{65.6}/\textbf{87.3}/\textbf{74.9} & \textbf{66.2}/\textbf{87.3}/\textbf{75.3} & \textbf{65.0}/\textbf{87.6}/\textbf{74.6} & \textbf{77.1}/\textbf{87.7}/\textbf{82.1} & \textbf{77.7}/\textbf{86.9}/\textbf{82.0} & \textbf{78.0}/\textbf{86.4}/\textbf{82.0} \\
    \hline
    \end{tabular}
\end{table*}

\begin{table*}[!htbp]
    \centering
    \scriptsize
\setlength{\tabcolsep}{1.5ex} %
\renewcommand{\arraystretch}{1.1}%
    \caption{Vector linking on \scifact{} (\qwen{}$\leftrightarrow$\openai{}):
    {\normalfont each cell reports \textbf{precision/recall/F1} (\%). Best values per metric are \textbf{bolded}.}}
    \label{tab:single_scifact_qwen_openai}
    \begin{tabular}{l|ccc|ccc|ccc}
    \hline
    \multirow{2}{*}{Method} & \multicolumn{3}{c|}{Overlap .15} & \multicolumn{3}{c|}{Overlap .20} & \multicolumn{3}{c}{Overlap .30} \\
     & Seeds 15 & Seeds 20 & Seeds 30 & Seeds 15 & Seeds 20 & Seeds 30 & Seeds 15 & Seeds 20 & Seeds 30 \\
    \hline
    \linear & 2.6/3.0/2.7 & 4.5/6.0/5.2 & 8.2/11.7/9.6 & 2.3/3.8/2.9 & 4.1/7.4/5.2 & 8.4/15.8/11.0 & 2.0/4.5/2.8 & 4.2/9.2/5.8 & 10.8/21.1/14.3 \\
    \cca & 27.4/8.4/12.9 & 30.8/14.3/19.5 & 46.7/38.0/41.9 & 28.6/8.0/12.5 & 42.8/16.3/23.6 & 56.1/40.0/46.7 & 32.5/6.2/10.5 & 52.1/16.4/25.0 & 67.6/38.0/48.7 \\
    \mlp & 18.7/2.2/4.0 & 23.7/6.2/9.8 & 28.7/11.4/16.3 & 23.3/4.9/8.1 & 28.3/7.7/12.1 & 41.3/16.0/23.1 & 24.1/2.1/3.9 & 37.5/6.1/10.5 & 48.7/13.4/21.1 \\
    \rcsls & 30.9/3.6/6.4 & 32.7/7.9/12.7 & 41.2/22.5/29.1 & 26.9/2.7/4.9 & 35.5/6.4/10.9 & 48.5/21.3/29.6 & 30.7/2.4/4.5 & 40.5/5.9/10.3 & 57.7/20.3/30.0 \\
    \pr & 26.5/7.6/11.8 & 36.1/15.8/22.0 & 50.8/41.2/45.5 & 34.5/9.8/15.2 & 53.8/20.3/29.4 & 65.2/43.3/52.0 & 42.3/7.9/13.3 & 61.4/18.1/27.9 & 74.5/42.4/54.1 \\
    \revise{\ugw} & \revise{3.1/0.3/0.5} & \revise{1.6/0.1/0.2} & \revise{1.6/0.1/0.2} & \revise{5.1/0.3/0.6} & \revise{5.1/0.3/0.6} & \revise{5.1/0.3/0.6} & \revise{8.5/0.3/0.6} & \revise{8.6/0.3/0.6} & \revise{8.6/0.3/0.6} \\
    \ao & 4.5/4.5/4.5 & 15.1/15.1/15.1 & 27.7/27.7/27.7 & 5.2/5.2/5.2 & 13.3/13.3/13.3 & 34.3/34.3/34.3 & 29.1/29.1/29.1 & 46.4/46.4/46.4 & 47.5/47.5/47.5 \\
    \textbf{\mvh} & \textbf{63.5}/\textbf{96.2}/\textbf{76.5} & \textbf{64.9}/\textbf{95.9}/\textbf{77.4} & \textbf{62.3}/\textbf{95.6}/\textbf{75.4} & \textbf{72.8}/\textbf{96.7}/\textbf{83.1} & \textbf{73.2}/\textbf{96.3}/\textbf{83.1} & \textbf{70.9}/\textbf{96.6}/\textbf{81.8} & \textbf{83.8}/\textbf{96.9}/\textbf{89.9} & \textbf{84.0}/\textbf{96.9}/\textbf{90.0} & \textbf{83.9}/\textbf{96.4}/\textbf{89.7} \\
    \hline
    \end{tabular}
\end{table*}

\begin{table*}[!htbp]
    \centering
    \scriptsize
\setlength{\tabcolsep}{1.5ex} %
\renewcommand{\arraystretch}{1.1}%
    \caption{Vector linking on \fiqa{} (\qwen{}$\leftrightarrow$\openai{}):
    {\normalfont each cell reports \textbf{precision/recall/F1} (\%). Best values per metric are \textbf{bolded}.}}
    \label{tab:single_fiqa_qwen_openai}
    \begin{tabular}{l|ccc|ccc|ccc}
    \hline
    \multirow{2}{*}{Method} & \multicolumn{3}{c|}{Overlap .15} & \multicolumn{3}{c|}{Overlap .20} & \multicolumn{3}{c}{Overlap .30} \\
     & Seeds 15 & Seeds 20 & Seeds 30 & Seeds 15 & Seeds 20 & Seeds 30 & Seeds 15 & Seeds 20 & Seeds 30 \\
    \hline
    \linear & 1.1/0.6/0.8 & 1.7/1.5/1.6 & 3.8/4.9/4.3 & 1.3/0.6/0.8 & 1.9/1.7/1.8 & 3.9/5.5/4.5 & 1.2/1.3/1.3 & 2.1/3.1/2.5 & 5.1/8.9/6.5 \\
    \cca & 5.9/1.1/1.9 & 10.3/2.9/4.5 & 24.5/12.7/16.8 & 9.9/1.4/2.5 & 18.1/4.3/6.9 & 34.9/15.8/21.8 & 9.3/0.7/1.3 & 20.7/3.1/5.4 & 38.1/11.0/17.1 \\
    \mlp & 2.3/0.1/0.3 & 5.1/0.4/0.8 & 10.0/1.0/1.8 & 4.1/0.2/0.5 & 8.8/0.7/1.3 & 19.3/2.6/4.6 & 3.8/0.1/0.2 & 9.7/0.4/0.7 & 17.2/1.0/1.9 \\
    \rcsls & 6.5/0.3/0.6 & 10.2/1.1/2.0 & 21.5/5.7/9.1 & 7.9/0.3/0.6 & 10.9/0.8/1.5 & 24.4/4.7/7.8 & 10.7/0.3/0.6 & 16.1/1.0/1.9 & 32.6/5.2/9.0 \\
    \pr & 10.4/1.8/3.1 & 17.5/4.7/7.4 & 37.0/17.2/23.5 & 15.8/2.1/3.7 & 26.3/5.9/9.6 & 48.6/20.4/28.7 & 17.0/1.3/2.4 & 32.0/4.5/7.8 & 53.0/14.8/23.2 \\
    \revise{\ugw} & \revise{0.0/0.0/0.0} & \revise{2.2/0.0/0.0} & \revise{2.1/0.0/0.0} & \revise{0.0/0.0/0.0} & \revise{0.0/0.0/0.0} & \revise{0.0/0.0/0.0} & \revise{0.0/0.0/0.0} & \revise{0.0/0.0/0.0} & \revise{0.0/0.0/0.0} \\
    \ao & 0.2/0.2/0.2 & 0.3/0.3/0.3 & 0.4/0.4/0.4 & 0.2/0.2/0.2 & 0.2/0.2/0.2 & 0.4/0.4/0.4 & 0.1/0.1/0.1 & 43.0/43.0/43.0 & 43.6/43.6/43.6 \\
    \textbf{\mvh} & \textbf{54.5}/\textbf{70.3}/\textbf{61.4} & \textbf{55.1}/\textbf{68.9}/\textbf{61.3} & \textbf{56.8}/\textbf{68.8}/\textbf{62.2} & \textbf{65.3}/\textbf{69.9}/\textbf{67.5} & \textbf{66.1}/\textbf{70.2}/\textbf{68.1} & \textbf{66.0}/\textbf{69.7}/\textbf{67.8} & \textbf{77.9}/\textbf{70.8}/\textbf{74.2} & \textbf{79.0}/\textbf{71.5}/\textbf{75.1} & \textbf{78.0}/\textbf{70.3}/\textbf{74.0} \\
    \hline
    \end{tabular}
\end{table*}

\subsubsection{Additional Results on Out-of-Domain Anchors}\label{app:ood-heatmaps}
Figure~\ref{fig:ood-heatmap-appendix} extends the main-text OOD analysis (Fig.~\ref{fig:ood-heatmap}) by sweeping the seed budget $n\in\{15,20,30\}$ and the target overlap ratio $\alpha\in\{0.15,0.2,0.3\}$. Overall, most reference–target pairs preserve strong precision and recall under OOD seeding; the few degraded cases align with our Theorem~\ref{thm-local-iso}, which predicts that links supported primarily by long-range anchors are less reliable.

\begin{figure*}
  \centering
  \begin{subfigure}[t]{0.49\textwidth}
    \centering
    \includegraphics[width=\linewidth]{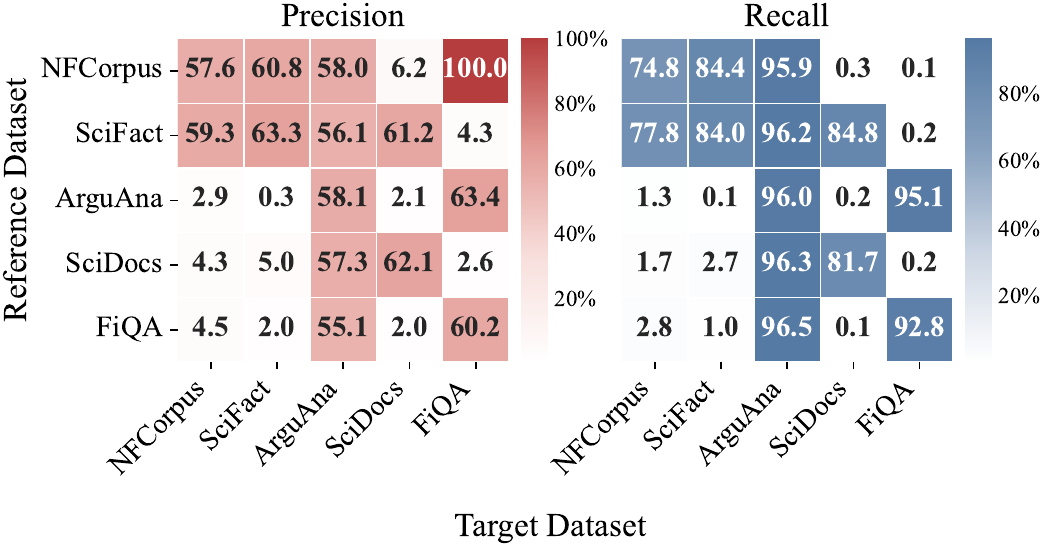}
    \caption{$n{=}15$, $o{=}0.15$}
    \label{fig:ood-n15-o015}
  \end{subfigure}\hfill
  \begin{subfigure}[t]{0.49\textwidth}
    \centering
    \includegraphics[width=\linewidth]{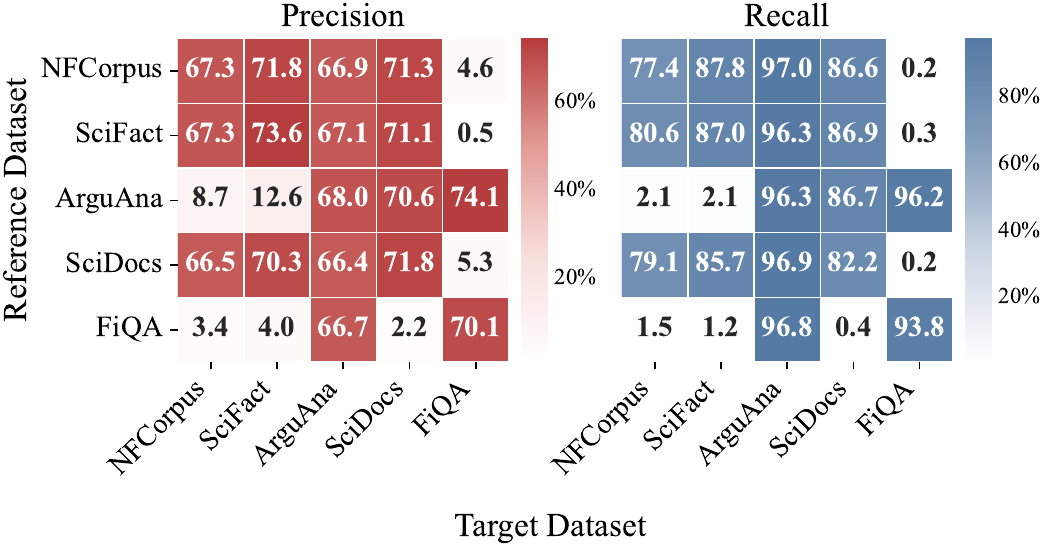}
    \caption{$n{=}15$, $o{=}0.2$}
    \label{fig:ood-n15-o02}
  \end{subfigure}

  \begin{subfigure}[t]{0.49\textwidth}
    \centering
    \includegraphics[width=\linewidth]{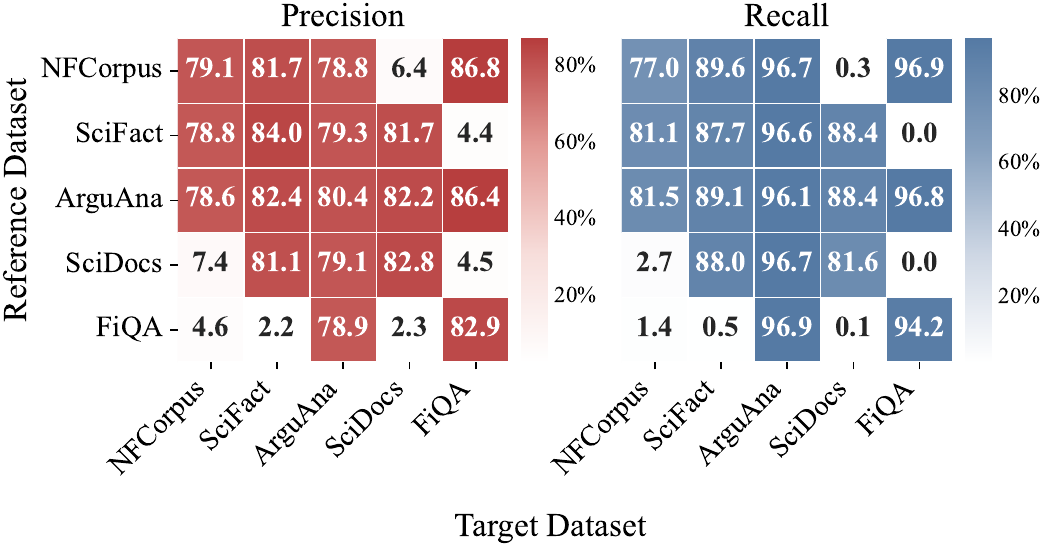}
    \caption{$n{=}15$, $o{=}0.3$}
    \label{fig:ood-n15-o03}
  \end{subfigure}\hfill
  \begin{subfigure}[t]{0.49\textwidth}
    \centering
    \includegraphics[width=\linewidth]{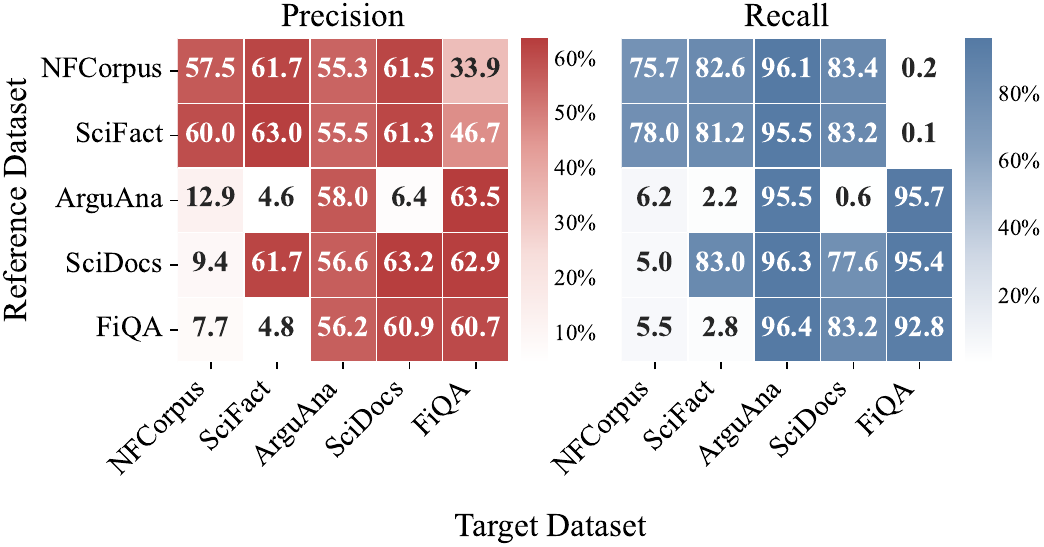}
    \caption{$n{=}20$, $o{=}0.15$}
    \label{fig:ood-n20-o015}
  \end{subfigure}

  \begin{subfigure}[t]{0.49\textwidth}
    \centering
    \includegraphics[width=\linewidth]{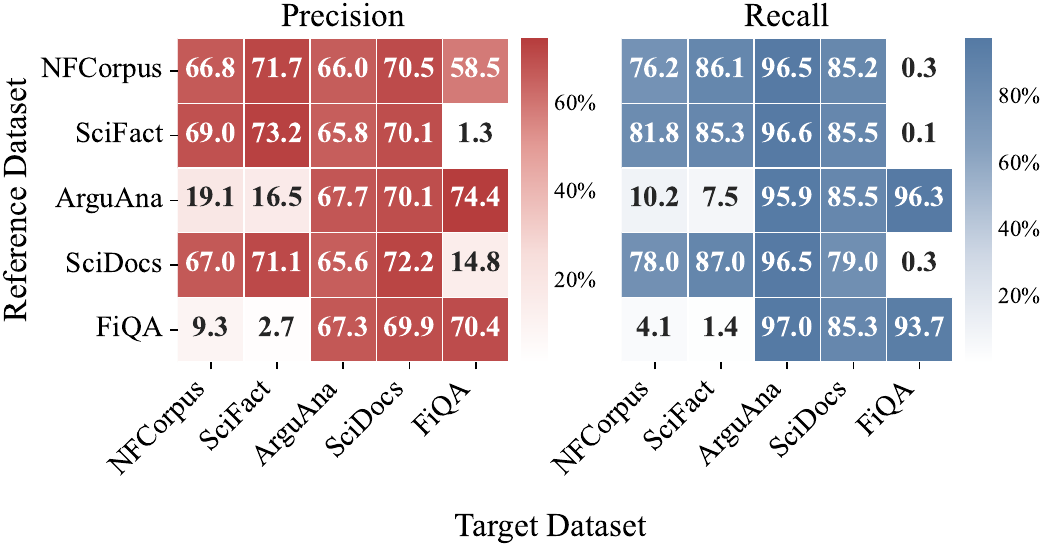}
    \caption{$n{=}20$, $o{=}0.2$}
    \label{fig:ood-n20-o02}
  \end{subfigure}\hfill
  \begin{subfigure}[t]{0.49\textwidth}
    \centering
    \includegraphics[width=\linewidth]{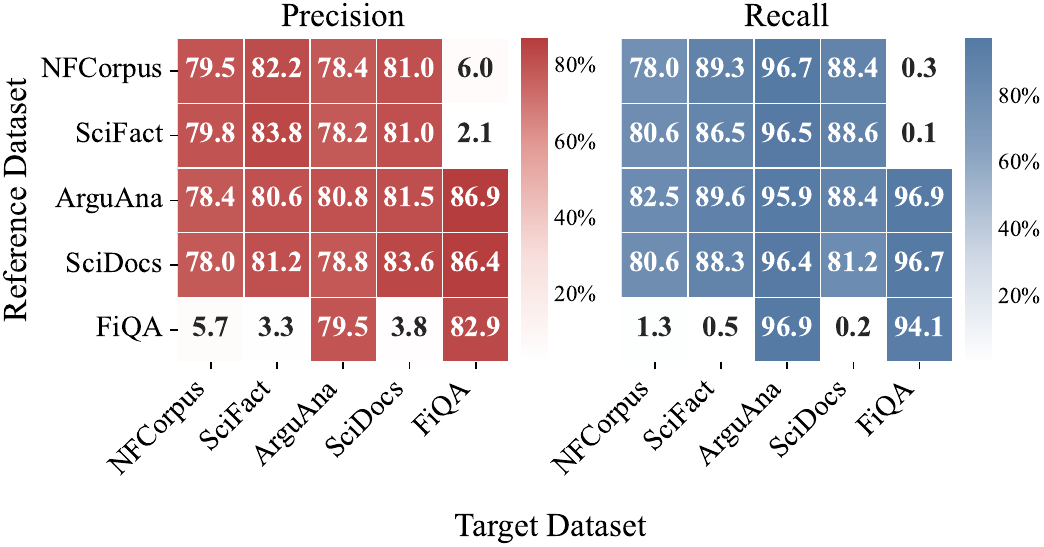}
    \caption{$n{=}20$, $o{=}0.3$}
    \label{fig:ood-n20-o03}
  \end{subfigure}

  \begin{subfigure}[t]{0.49\textwidth}
    \centering
    \includegraphics[width=\linewidth]{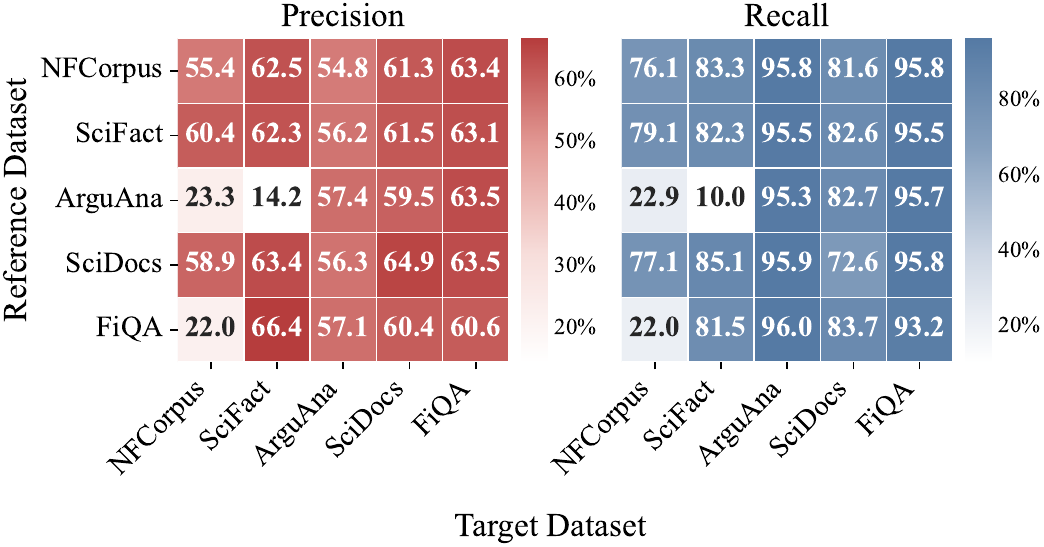}
    \caption{$n{=}30$, $o{=}0.15$}
    \label{fig:ood-n30-o015}
  \end{subfigure}\hfill
  \begin{subfigure}[t]{0.49\textwidth}
    \centering
    \includegraphics[width=\linewidth]{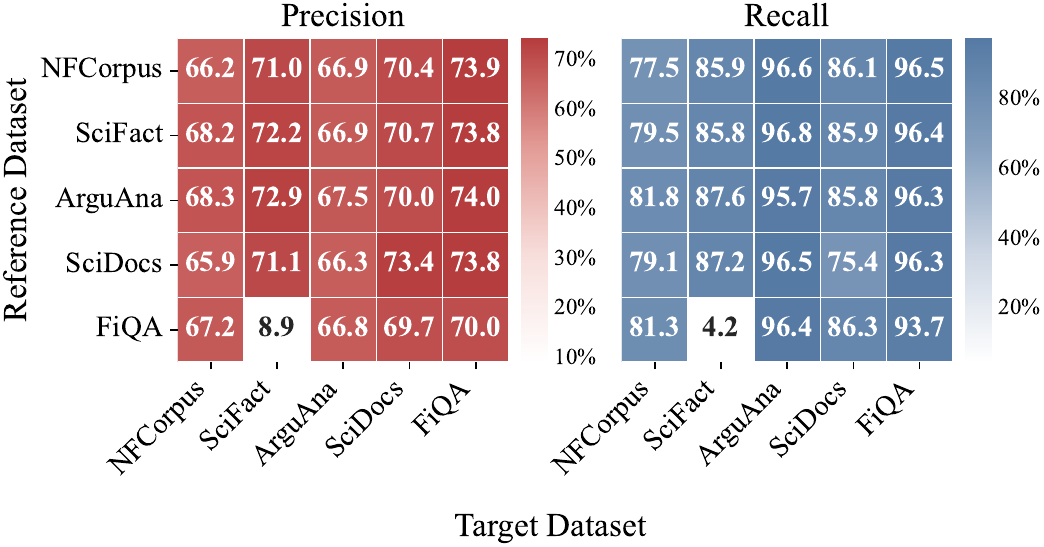}
    \caption{$n{=}30$, $o{=}0.2$}
    \label{fig:ood-n30-o02}
  \end{subfigure}

  \caption{Out-of-domain reference transfer (additional settings):{\normalfont\footnotesize
  Accuracy (left) and recall (right) on five target datasets (columns) when seeds are drawn from an out-of-domain reference dataset (rows).
  Each panel varies the number of seeds $n$ and target overlap $o$. The main text reports the case $n{=}30$, $o{=}0.3$ (Fig.~\ref{fig:ood-heatmap}).}}
  \label{fig:ood-heatmap-appendix}
\end{figure*}

\eat{
\subsection{Scalability on FEVER}
\label{app:scalability}

\stitle{Setup.} We test scalability on FEVER~\cite{thakur2021beir}
using \mistral $\leftrightarrow$ \openai embeddings, fixing the overlap ratio
to $\alpha=0.3$, seed budget $|\Sc|=30$). All methods are run on a single
NVIDIA A100 (80\,GB). For \mvh, the per-view local-set restriction described
in Appendix~\ref{app:per-view-impl} is automatically activated because
$\max(|\Ec_1|,|\Ec_2|)>\tau$.

\stitle{Results.} Table~\ref{tab:fever_scalability} reports precision, recall,
and end-to-end wall-clock runtime on a single A100. \mvh attains $93.8\%$
precision and $68.9\%$ recall in $3328$\,s. \mvh
remains within the same order of magnitude as the fastest baseline (\cca,
$\sim$$2.1\times$ slower) and is faster than \pr/TPS/\mlp/\linear/\rcsls{}
despite its iterative design. The performance gap is large: relative to \cca,
\mvh{} improves precision from $5.22\%$ to $93.8\%$  and
recall from $0.61\%$ to $68.9\%$, at only $\sim$$2.1\times$
the wall-clock cost.

\begin{table}[t]
\centering
\small
\caption{Scalability on \fever{} 
{\normalfont\footnotesize on \mistral$\leftrightarrow$\openai, single A100 80\,GB GPU,
overlap $\alpha=0.3$, $|\Sc|=30$. \textbf{Bold} marks the best per column;
runtime is end-to-end wall-clock seconds.}}
\label{tab:fever_scalability}
\begin{tabular}{lccc}
\toprule
Method & Precision (\%) & Recall (\%) & Runtime (s) \\
\midrule
\linear & 4.13 & 0.00 & 4420 \\
\cca    & 5.22 & 0.61 & \textbf{1613} \\
\mlp    & 1.93 & 0.01 & 4414 \\
\rcsls  & 6.40 & 0.11 & 3348 \\
\pr     & 9.03 & 0.73 & 4494 \\
TPS     & 9.06 & 0.73 & 4517 \\
\textbf{\mvh} & \textbf{93.8} & \textbf{68.9} & 3328 \\
\bottomrule
\end{tabular}
\end{table}

CSLS+MNN link extraction is shared by all alignment baselines, so the
incremental cost of \mvh{} comes only from evaluating multiple views.
Crucially, each view operates in the low-dimensional distance-to-anchor hash
space induced by a small anchor set, rather than re-matching in the original
embedding space. The iterative procedure thus performs many moderate-cost
hash-space retrievals rather than repeated dense searches over full embedding
clouds. 
}

\section{Details of Section~\ref{sec-applications}}
\label{app:application}
\subsection{Implementation Details of Applications}
\subsubsection{Vector Database Integration}
\label{sec:appendix_la2m}

We follow the evaluation protocol of \citet{LA2M}: for each benchmark corpus $O$, we place all benchmark answer documents
only in $O_1 \cup O_2$ (i.e., $O_\cap$ contains no answer documents). We then build two vector databases
$D_1 = \mathrm{emb}_1(O_1 \cup O_\cap)$ and $D_2 = \mathrm{emb}_2(O_2 \cup O_\cap)$, and evaluate retrieval using queries
encoded by $\mathrm{emb}_2$, i.e., we query the integrated database with $\mathrm{emb}_2(q)$. 

Given vector links $(X,Y)$ induced by $O_\cap$ (vectors in $D_1$ and $D_2$ that encode the same items),
we compute an integration mapping $T$ from the vector space of $D_1$ to that of $D_2$ using the local-isometry-based
framework of \citet{LA2M}, and return the integrated database $T(D_1)\cup D_2$. Let $Q$ be the benchmark query set, and for each query $q\in Q$ let $\mathrm{ans}_q$ denote its ground-truth relevant set.
Let $\mathrm{top}\text{-}k(q)$ be the top-$k$ results returned by searching the integrated database with $\mathrm{emb}_2(q)$.
We report:
\[
\mathrm{Recall@}k \;=\; \frac{1}{|Q|}\sum_{q\in Q}\frac{|\mathrm{ans}_q \cap \mathrm{top}\text{-}k(q)|}{|\mathrm{ans}_q|}.
\]
For rank-sensitive evaluation, we also report NDCG:
\[
\mathrm{NDCG@}k \;=\; \frac{\mathrm{DCG@}k}{\mathrm{IDCG@}k},
\qquad
\mathrm{DCG@}k \;=\; \sum_{i=1}^{k}\frac{\mathrm{rel}_i}{\log_2(i+1)},
\]
where $\mathrm{rel}_i$ is the graded relevance of the item at rank $i$ and $\mathrm{IDCG@}k$ is the DCG of the ideal ranking.

We use a FAISS GPU index with inner-product search over $\ell_2$-normalized embeddings (equivalently cosine similarity).

\subsubsection{Global Cross-Model Clustering}\label{app:global_xmodel_cluster}
We evaluate cross-model clustering using two clustering benchmarks from MTEB \cite{enevoldsen2025mmtebmassivemultilingualtext}.
Both datasets consist of short \emph{titles} and provide gold cluster labels (e.g., subreddit or StackExchange community). 
Dataset statistics are summarized in Table~\ref{tab:clustering_datasets_stats}.
\begin{table*}
\centering
\caption{Cross-model clustering datasets from MTEB (test split metadata).\normalfont\footnotesize{Lengths are measured in characters per title.}}
\label{tab:clustering_datasets_stats}
\begin{tabular}{lrrrr}
\toprule
Dataset & \#Titles & \#Clusters & Avg. len. & Min/Max len. \\
\midrule
RedditClustering.v2~\cite{geigle2021tweactransformerextendableqa} & 2048 & 50  & 65.49 & 18 / 299 \\
StackExchangeClustering.v2~\cite{geigle2021tweactransformerextendableqa} & 2048 & 121 & 57.51 & 19 / 148 \\
\bottomrule
\end{tabular}
\end{table*}

For a dataset with texts $\{t_\ell\}_{\ell=1}^N$ and gold labels $\{y_\ell\}$, we generate two embedding sets
$\Ec_1 =\{e^{(1)}_\ell\}$ and $\Ec_2 =\{e^{(2)}_\ell\}$ using two embedding models (e.g., \qwen{} and \kalm{}). We then create
a partial-overlap partition by selecting index sets $\mathcal{I}_1,\mathcal{I}_2\subseteq[N]$ such that
$\mathcal{I}_\cap=\mathcal{I}_1\cap\mathcal{I}_2$ contains the shared items and $\mathcal{I}_1\setminus\mathcal{I}_2$,
$\mathcal{I}_2\setminus\mathcal{I}_1$ are model-specific items. The ground-truth correspondence set is
$\mathcal{P}^\star=\{(\ell,\ell): \ell\in\mathcal{I}_\cap\}$; predicted correspondences $\hat{\mathcal{P}}$ are produced by a linking method.

\stitle{Graph construction.}
For each embedding space, we build a $k$-NN graph $G_1$ on $\{e^{(1)}_\ell:\ell\in\mathcal{I}_1\}$ and $G_2$ on
$\{e^{(2)}_\ell:\ell\in\mathcal{I}_2\}$ (cosine similarity). We choose $k$ adaptively to ensure connectivity, by increasing $k$ until the largest connected component covers at least $95\%$ of nodes. We then form a unified graph $G$ by merging each correspondence pair $(\ell_1,\ell_2)\in\hat{\mathcal{P}}$ into a single super-node that inherits the incident edges from both $G_1$ and $G_2$. When no correspondences are provided, $G$ is simply the disjoint union of $G_1$ and $G_2$.

\stitle{Clustering.}
We apply Leiden community detection on $G$. To make comparisons fair across methods, we use the same graph-based clustering pipeline throughout
and tune the Leiden resolution by binary search to match the known number of gold clusters in the evaluated set. 

We report : (i) \textbf{Full-space} single-model clustering on each complete embedding space independently (\qwen and \kalm), serving as optimal references; (ii) \textbf{Concat}, which zero-pads embeddings to a common dimension and concatenates them without using any cross-space correspondences; (iii) \textbf{Seed}, which stitches the two $k$-NN graphs by node-merging using only ground-truth seed correspondences; and (iv) \textbf{Ours}, which performs the same node-merging procedure using predicted correspondences $\hat{\mathcal{P}}$ from \mvh.
                                                                                                    
\subsection{Experimental Results on Cross Model Clustering}
We report cross-model clustering results for \qwen$\leftrightarrow$\kalm at overlap ratios $\alpha\in\{0.2,0.3\}$ and seed budgets $n\in\{20,30\}$. We evaluate clustering quality using V-measure, NMI, and ARI, and additionally report
\emph{Overlap Agreement Rate} (OAR), defined as the fraction of overlapped items whose two embeddings (one from each space) are assigned to the
same community in the unified clustering. OAR is not reported for naive concatenation since it produces a disjoint union of the two graphs and
does not induce cross-space communities. As shown in Table~\ref{tab:clustering_unified}, using only seed correspondences yields limited cross-space connectivity and suboptimal global coherence,
whereas using \mvh{} to stitch the graphs achieves high cross-space coupling (OAR $=75$--$98\%$) and recovers cluster quality within $\approx1\%$
of single-space performance.
\begin{table}[!htbp]
    \centering
    \footnotesize
    \setlength{\tabcolsep}{3pt}
    \caption{Cross-model clustering performance for Qwen$\leftrightarrow$KaLM embeddings.\normalfont\footnotesize{
    Each cell reports V-measure / NMI / ARI (\%).
    OAR = Overlap Agreement Rate (\%).
    \textbf{Bold} indicates best per metric among Concat/Seed/OURS for each configuration.}}
    \label{tab:clustering_unified}
\begin{tabular}{@{}lccccccccc@{}}
\toprule
\multirow{2}{*}{Dataset} &
  \multicolumn{1}{c}{\multirow{2}{*}{Overlap}} &
  \multirow{2}{*}{Seed} &
  \multicolumn{3}{c}{V / NMI / ARI} &
  \multicolumn{2}{c}{OAR} &
  \multicolumn{2}{c}{V / NMI / ARI} \\
 & \multicolumn{1}{c}{} &    & Concat         & Seed           & OURS                                      & Seed & OURS          & Qwen & KaLM \\ \midrule
\multirow{4}{*}[-0.6ex]{Reddit} &
  \multirow{2}{*}{0.2} &
  20 &
  52.8/52.8/20.7 &
  53.4/53.4/20.8 &
  \textbf{62.8}/\textbf{62.8}/\textbf{37.6} &
  6.6 &
  \textbf{83.2} &
  \multirow{4}{*}[-0.6ex]{63.5/63.5/39.8} &
  \multirow{4}{*}[-0.6ex]{65.6/65.6/42.7} \\
 &                      & 30 & 52.8/52.8/20.7 & 54.0/54.0/22.3 & \textbf{63.0}/\textbf{63.0}/\textbf{38.0} & 10.5 & \textbf{84.6} &      &      \\ \cmidrule(lr){2-8}
 & \multirow{2}{*}{0.3} & 20 & 51.7/51.7/19.2 & 53.6/53.6/21.6 & \textbf{66.5}/\textbf{66.5}/\textbf{44.6} & 4.6  & \textbf{96.9} &      &      \\
 &                      & 30 & 51.7/51.7/19.2 & 53.9/53.9/22.1 & \textbf{67.1}/\textbf{67.1}/\textbf{45.6} & 7.7  & \textbf{97.6} &      &      \\ \midrule
\multirow{4}{*}[-0.6ex]{StackEx} &
  \multirow{2}{*}{0.2} &
  20 &
  62.2/62.2/19.2 &
  62.5/62.5/20.2 &
  \textbf{67.0}/\textbf{67.0}/\textbf{28.9} &
  7.1 &
  \textbf{75.9} &
  \multirow{4}{*}[-0.6ex]{68.7/68.7/30.7} &
  \multirow{4}{*}[-0.6ex]{68.8/68.8/31.1} \\
 &                      & 30 & 62.2/62.2/19.2 & 62.4/62.4/20.2 & \textbf{67.4}/\textbf{67.4}/\textbf{30.2} & 9.8  & \textbf{88.8} &      &      \\ \cmidrule(lr){2-8}
 & \multirow{2}{*}{0.3} & 20 & 61.8/61.8/18.9 & 61.8/61.8/19.6 & \textbf{67.5}/\textbf{67.5}/\textbf{31.2} & 5.9  & \textbf{78.8} &      &      \\
 &                      & 30 & 61.8/61.8/18.9 & 62.1/62.1/20.1 & \textbf{68.3}/\textbf{68.3}/\textbf{32.7} & 8.3  & \textbf{91.9} &      &      \\ \bottomrule
\end{tabular}
\end{table}

\end{document}